\documentclass[journal]{IEEEtran}

\usepackage{float}
\usepackage{caption}

\ifCLASSINFOpdf
 
\else

\fi

\usepackage{amsmath}

\usepackage{amsmath,amssymb,amsthm}
\usepackage{subfigure}
\usepackage{algorithm}
\usepackage{algorithmic}

\newtheorem{theorem}{Theorem}
\newtheorem{lemma}{Lemma}
\newtheorem{property}{Property}
\newtheorem{definition}{Definition}
\newtheorem{remark}{Remark}
\usepackage{adjustbox}
\usepackage{xcolor}

\begin{document}

\title{Effective Streaming Low-tubal-rank Tensor Approximation via Frequent Directions}

\author{Qianxin Yi, Chenhao Wang, Kaidong Wang, and Yao Wang
	\thanks{Qianxin Yi, Chenhao Wang, Kaidong Wang and Yao Wang are with the Center for Intelligent Decision-making and Machine Learning, School of Management, Xi'an Jiaotong University, Xi'an, 710049, China.}
	\thanks{Qianxin Yi and Kaidong Wang are also with the School of Mathematics and Statistics, Xi'an Jiaotong University, Xi'an, 710049, China.}
	\thanks{Yao Wang is the corresponding author. Email: yao.s.wang@gmail.com.}}

\makeatletter 
\renewcommand{\@thesubfigure}{\hskip\subfiglabelskip}
\makeatother


\maketitle

\begin{abstract}
Low-tubal-rank tensor approximation has been proposed to analyze large-scale and multi-dimensional data. However, finding such an accurate approximation is challenging in the streaming setting, due to the limited computational resources. To alleviate this issue, this paper extends a popular matrix sketching technique, namely Frequent Directions, for constructing an efficient and accurate low-tubal-rank tensor approximation from streaming data based on the tensor Singular Value Decomposition (t-SVD). Specifically, the new algorithm allows the tensor data to be observed slice by slice, but only needs to maintain and incrementally update a much smaller sketch which could capture the principal information of the original tensor. The rigorous theoretical analysis shows that the approximation error of the new algorithm can be arbitrarily small when the sketch size grows linearly. Extensive experimental results on both synthetic and real multi-dimensional data further reveal the superiority of the proposed algorithm compared with other sketching algorithms for getting low-tubal-rank approximation, in terms of both efficiency and accuracy.
\end{abstract}


\IEEEpeerreviewmaketitle

\section{Introduction}

\IEEEPARstart{T}{ensors}, or multi-dimensional arrays, are generalizations of vectors and matrices, which have been commonly used  in representing real-world data, such as videos \cite{jiang2017novel, kim2007tensor}, hyperspectral images \cite{du2016pltd, renard2008denoising}, multilinear signals \cite{de1997signal, comon2002tensor} and communication networks \cite{papalexakis2014spotting, nakatsuji2017semantic} among numerous others.  Common to these and many other data, the so-called low-rank structure can oftentimes be used to identify these tensor data, and thus low-rank tensor approximation is becoming a fundamental tool in today's data analytics \cite{renard2008denoising, liu2012denoising, grasedyck2013literature}, \cite{liu2015generalized, liu2014trace}.  However, it is oftentimes infeasible to find an accurate approximation because of the large size of these tensors. For example, as stated in~\cite{gilman2020grassmannian}, the hyperspectral video with hundreds of spectral bands and megapixel spatial resolution needs to be stored at the order of 10 gigabit per second. This clearly means that such a large tensor  may not fit in main memory on a computer, which brings difficulties to the subsequent low-rank approximation by a tensor decomposition. Also, calculating such an approximation needs to perform SVD (Singular Value Decomposition)  or tensor SVD,  which is usually time-consuming.

To address this concern, several tensor sketching methods \cite{wang2015fast, battaglino2018practical,xia2017effective,malik2018low, zhang2018randomized} were designed to perform fast low-rank approximation based on specific decompositions with a slight loss of precision, while significantly reducing the memory requirement. More precisely, analogous to the matrix sketching technique \cite{woodruff2014sketching} (such as random sampling \cite{boutsidis2014near} and random projection \cite{bingham2001random}), tensor sketching aims to compute a sketch tensor that is significantly smaller than the original one but still preserves its vital properties for further computations.  In many real applications, however, the tensor data just as the aforementioned hyperspectral video often arrives in a streaming manner which naturally requires the sketching algorithms to be one pass. This poses a challenge that how to develop a sketching algorithm to perform streaming low-rank tensor approximation efficiently. 

In this paper,  we develop an effective sketching algorithm to compute the low-tubal-rank tensor approximation from streaming data using the tensor SVD (t-SVD) framework. Similar to the matrix SVD, a key property of the t-SVD is the optimality of the truncated t-SVD for tensor approximation in terms of the Frobenius norm. Another key property of the t-SVD framework is that the derived tubal rank can well characterize the inherent low-rank structure of a tensor. With these good properties, the t-SVD has been extensively studied in dealing with several low-rank tensor approximation problems, both theoretically and practically. See, e.g., \cite{hao2013facial}, \cite{lu2016tensor}, among others. Considering that most existing studies are focused on the batch setting that requires tensor data to be fitted in the main memory, very little is known about the performance of t-SVD based low-rank approximation in the streaming setting. Our work presented here tries to fill in this void. 

By extending the simple deterministic sketching technique, Frequent Directions \cite{liberty2013simple}, we propose a new tensor sketching algorithm called tensor Frequent Directions (t-FD) to pursue an efficient streaming low-tubal-rank tensor approximation. The key idea of the proposed algorithm is to maintain a small sketch tensor dynamically updated in the streaming data. We shall summarize the main contributions of this paper as follows:

\begin{itemize}
	\item This is the first attempt to apply FD technique to deal with higher order tensors. Specifically, the proposed 
	t-FD algorithm only requires a single pass over the tensor, and thus is applicable to the streaming setting. 
	\item Considering that operations under the t-SVD are mainly processed in the Fourier domain, we analyze the  relationship of the tensor norms between the original and Fourier domains, and further derive the tensor covariance and projection error bounds. The theoretical analysis shows that the proposed t-FD is  within $1+\varepsilon$ of best tubal-rank-$k$ approximation.  
	\item Extensive experiments are carried out on both synthetic and real tensor data to illustrate the superiority of our t-FD algorithm over matrix FD algorithm and two randomized tensor sketching algorithms (srt-SVD and NormSamp) in most cases. These experimental results also partially verify our theoretical findings. 
\end{itemize}

\section{Related work}
\subsection{Streaming low-rank tensor approximation}
Finding the low-rank structure from the streaming tensor data has been developed well in recent years.  Most of the previous works focus on the 
{Tucker and CP} decompositions. The Tucker decomposition factorizes a tensor into the multiplication of a core tensor with  orthogonal factor matrices along each mode, and the corresponding Tucker rank is defined as the tuple of the ranks of all  unfolding matrices. Therefore, its computation is  heavily relied on the computation of SVD, thus the work \cite{hu2011incremental} considered in incorporating incremental SVD to update the data dynamically. Subsequently, \cite{malik2018low, sun2019low} integrated the randomized sketching techniques into the traditional HOSVD/HOOI algorithms for single pass  low Tucker rank decomposition. The CP decomposition factorizes a tensor into the sum of several rank-one tensors, and the corresponding CP rank is defined as the minimum number of such rank-one tensors, which is intuitive and similar to that of matrix rank. There are also a series of works \cite{zhou2016accelerating,ma2018randomized} focusing on the online CP problem, however, tracking the CP decomposition of online tensor often utilizes the alternating least squares method, namely CP-ALS, to update factor matrices in a nonconvex optimization manner, thus the performance would be highly dependent on a good initialization, which may not be satisfied in some real situations.

	More recently, the works ~\cite{kilmer2011factorization, martin2013order} proposed an alternative tensor decomposition named t-SVD, which is an elegant and natural extension of matrix SVD. More specifically, t-SVD factorizes a tensor into three factor tensors based on a new defined tensor-tensor product (t-product) operation, and could capture spatial-shifting correlation without losing intrinsic structures caused by matricization.  As further stated in~\cite{kilmer2013third}, t-SVD possesses both efficient computations and solid mathematical foundations, and thus has been widely used in a great number of low-rank tensor related problems, e.g.,~\cite{hao2013facial,zhang2016exact, lu2019tensor}. However, very little is known about the performance of t-SVD in dealing with streaming low-rank tensor data.




\subsection{Frequent Directions}

The main idea of the so-called matrix sketching technique~\cite{woodruff2014sketching} is to first construct a sketch matrix whose size is much smaller than the original matrix but can retain most of the information, and then use such sketch instead of the original matrix to do the subsequent operations, such as matrix multiplication and SVD. To get the sketch matrix, several randomized algorithms, such as random sampling~\cite{boutsidis2014near} and  random projection~\cite{bingham2001random}, have drawn great attention. The random sampling technique obtains a precise representation of the original matrix by sampling a small number of rows or columns and reweighting them. The most well-known random sampling technique is the leverage score sampling, in which the sampling probability
is proportional to the leverage score of each column. This obviously poses the difficulty that the leverage score involves the calculation of the singular vectors of the original matrix, and thus is hard to process streaming data. As for the random projection technique, its key is to find a random matrix used to project the original matrix to a much smaller one. This needs to load the original matrix completely in memory, which is obviously unsuitable for streaming setting. As mentioned previously, these randomized sketching techniques have been extended to get fast low-rank tensor approximation based on specific decompositions, namely  CP~\cite{battaglino2018practical}, Tucker~\cite{malik2018low} and t-SVD~\cite{zhang2018randomized}. Similar to the matrix case, such randomized tensor sketching algorithms cannot process streaming data directly.

Recently, a deterministic matrix sketching technique named Frequent Directions (FD), which was introduced by \cite{liberty2013simple} and further analyzed by \cite{ghashami2016frequent}, is well suited for the streaming data.  Precisely, the sketch is first initialized to an all zero-valued matrix. Then FD needs to insert the rows of the original matrix  into the sketch matrix  until it is fulfilled. A subsequent shrinking procedure is conducted by computing the SVD of the sketch and subtracting the squared $\ell$-th singular value. Considering that the last row of the sketch is always all zero-valued after the shrinking procedure, FD inserts the rows continually until all rows are processed. It has been proved in~\cite{woodruff2014low} that FD can achieve the optimal tradeoff between space and accuracy. Since FD could deal with streaming data without the sacrifice of accuracy, many online learning tasks have adopted it. Leng \cite{leng2015online}  utilized FD to learn hashing function efficiently in the online setting. Ilja \cite{kuzborskij2019efficient} showed FD could accelerate two popular linear contextual bandit algorithms without losing much precision. 
In  recent years, many subsequent  attempts have been made to improve the precision and speed of FD. Luo \cite{luo2019robust} proposed Robust Frequent Directions (RFD) by introducing an additional variable to make the FD more robust. Huang \cite{huang2019near}  considered to sample the removed part in the shrinking procedure then concatenate it with the sketch as the final result. And he theoretically proved such procedure is a space-optimal algorithm with improved running time compared with traditional FD. Besides, some papers considered the random projection technique to accelerate the original FD. That is, the subsampled randomized Hadamard transform and Count Sketch matrix were considered in \cite{chen2017frosh} and \cite{teng2018fast}, respectively.

\section{Notations and preliminaries}
We use the symbols $a$, $\boldsymbol{a}$, $\boldsymbol{A}$, $\mathbf{\mathcal{A}}$ for scalars, vectors, matrices,
and tensors respectively. 
	For the order-$p$ tensor $\mathbf{\mathcal{A}} \in \mathbb{R}^{n_{1} \times n_{2} \times\cdots \times n_{p}}\ (p\ge3)$, the $(i_1,i_2,\ldots,i_p)$-th entry is denoted by $\mathbf{\mathcal{A}}_{i_1i_2\ldots i_p}$,  and matrix frontal slices of order-$p$ tensors can be referenced using linear indexing by reshaping the tensor into an $n_{1} \times n_{2} \times \rho$ third-order tensor and referring to the $k$-th frontal slice as $\boldsymbol{A}^{(k)}$, where $\rho=n_3n_4\ldots n_p$, and the corresponding relationship is as follows: 
	$$
	(i_1,i_2,i_3,\ldots,i_p)\rightarrow(i_1,i_2,\sum_{a=4}^{p}(i_a-1)\Pi_{b=3}^{a-1}n_b+i_3).
	$$
	 The Frobenius norm of $\mathcal{A}$ is denoted by $\|\mathbf{\mathcal{A}}\|_{F}=\sqrt{\sum_{i_1 i_2, \cdots,i_p }\left|\mathbf{\mathcal{A}}_{i_1 i_2, \ldots,i_p}\right|^{2}}$. We represent $\mathbf{\mathcal{A}} $ as $[\mathbf{\mathcal{A}}_1, ..., \mathbf{\mathcal{A}}_{n_1} ]$, where $\mathcal{A}_i \in \mathbb{R}^{1 \times n_2  \times\cdots \times n_{p}}$ denotes the $i$-th horizontal tensor. Furthermore, $\mathcal{A}^{(i)}\in \mathbb{R}^{n_{1} \times  \cdots \times n_{p-1}}$ for $i=1,\ldots,n_p$ denotes the $(p - 1)$-order tensor created by holding the
	 $p$-th index of $\mathcal{A}$ fixed at $i$.  It is easy to see that when $p=3$, $\mathcal{A}^{(i)} $ is equivalent to the previously defined frontal slice  $\boldsymbol{A}^{(i)}$. And the mode-1 unfolding matrix $\boldsymbol{A}_{(1)}$ of $\mathbf{\mathcal{A}}$ is denoted as
	$$
	\boldsymbol{A}_{(1)}=\left[\boldsymbol{A}^{(1)}\ \boldsymbol{A}^{(2)}\ \cdots\ \boldsymbol{A}^{(\rho)}\right].
	$$ 
	 Moreover, $\mathbf{\bar{\mathcal{A}}}\in \mathbb{C}^{n_{1} \times n_{2} \times\cdots \times n_{p}}$ is obtained by repeating FFTs along each mode of $\mathbf{\mathcal{A}}$, and $\boldsymbol{\bar{A}}$ is the block diagonal matrix composed of each frontal slice of $\bar{\mathbf{\mathcal{A}}}$, i.e.,
	$$
	\boldsymbol{\bar{A}}=\mathtt{bdiag}(\mathbf{\bar{\mathcal{A}}})=\left[\begin{array}{cccc}
		\boldsymbol{\bar{A}}^{(1)} & & & \\
		& \boldsymbol{\bar{A}}^{(2)} & & \\
		& & \ddots & \\
		& & & \boldsymbol{\bar{A}}^{\left(\rho\right)}
	\end{array}\right].
	$$
	 Note that 
	\begin{align}
		\boldsymbol{\bar{A}}=\left(\boldsymbol{\tilde{F}} \otimes \boldsymbol{I}_{n_{1}}\right) \cdot \boldsymbol{\tilde{A}} \cdot \left(\boldsymbol{\tilde{F}}^{-1} \otimes \boldsymbol{I}_{n_{2}}\right), \label{orthop}
	\end{align} 
	where $\boldsymbol{\tilde{F}}=\boldsymbol{F}_{n_{p}} \otimes\boldsymbol{F}_{n_{p-1}} \otimes \cdots \otimes \boldsymbol{F}_{n_{3}}$, $\boldsymbol{F}_{n_{i}}$ is the discrete Fourier transformation matrix, $\otimes$ denotes the Kronecker product, and $\boldsymbol{\tilde{A}}$ is the $n_1\rho \times n_2\rho$ block matrix formed from $\mathbf{\mathcal{A}}$ in the base
	level of recursion (see Fig. 3.2 in \cite{martin2013order} for details). Specifically, for the third-order tensor, we get that $\boldsymbol{\tilde{F}}= \boldsymbol{F}_{n_{3}}, \  \boldsymbol{\tilde{A}} =\mathtt{bcirc}(\mathbf{\mathcal{A}})$, and the block circulant matrix $\mathtt{bcirc}(\mathbf{\mathcal{A}})$ is denoted as 
	$$
	\mathtt{bcirc}(\mathbf{\mathcal{A}})=\left[\begin{array}{cccc}
		\boldsymbol{A}^{(1)} & \boldsymbol{A}^{\left(n_{3}\right)} & \cdots & \boldsymbol{A}^{(2)} \\
		\boldsymbol{A}^{(2)} & \boldsymbol{A}^{(1)} & \cdots & \boldsymbol{A}^{(3)} \\
		\vdots & \vdots & \ddots & \vdots \\
		\boldsymbol{A}^{\left(n_{3}\right)} & \boldsymbol{A}^{\left(n_{3}-1\right)} & \cdots & \boldsymbol{A}^{(1)}
	\end{array}\right].
	$$

Now we shall give a brief review of some related definitions of  tensors used in the paper. 

	\begin{definition}[t-product for order-$p$ tensors ($p\ge3$), \cite{martin2013order}] \label{deftproduct-p}
	Let $\mathbf{\mathcal{A}} \in \mathbb{R}^{n_{1} \times n_{2} \times \cdots \times n_{p}}$ and $\mathbf{\mathcal{B}} \in \mathbb{R}^{n_{2} \times \ell \times \cdots \times n_{p}} .$ Then the $t$-product $\mathbf{\mathcal{A}} * \mathbf{\mathcal{B}}$ is the order-$p$ tensor defined recursively as
	$$
	\mathbf{\mathcal{A}} * \mathbf{\mathcal{B}}=\mathtt {fold}(\mathtt{bcirc}(\mathbf{\mathcal{A}}) * \mathtt{unfold}(\mathbf{\mathcal{B}})).
	$$
	The $(p-1)$-order tensor $\mathtt{bcirc}(\mathbf{\mathcal{A}})$ is defined as
	$$
	\mathtt{bcirc}(\mathbf{\mathcal{A}})=\left[\begin{array}{ccccc}
		\mathcal{A}^{(1)} & \mathcal{A}^{(n_{p})} & \mathcal{A}^{(n_{p}-1)} & \cdots & \mathcal{A}^{(2)} \\
		\mathcal{A}^{(2)} & \mathcal{A}^{(1)} & \mathcal{A}^{(n_{p})} & \cdots & \mathcal{A}^{(3)} \\
		\vdots & \ddots & \ddots & \ddots & \vdots \\
		\mathcal{A}^{(n_{p})} & \mathcal{A}^{(n_{p}-1)} & \cdots & \mathcal{A}^{(2)} & \mathcal{A}^{(1)}
	\end{array}\right].
	$$
	Define $\mathtt { unfold }(\cdot)$ by taking an $n_1 \times \cdots \times n_p$ tensor and returning an $n_1n_p \times
	n_2 \times \cdots \times n_{p-1}$ block tensor in the following way:
	$$
	\mathtt{unfold}(\mathcal{A})=\left[\begin{array}{c}
		\mathcal{A}^{(1)} \\
		\mathcal{A}^{(2)} \\
		\vdots \\
		\mathcal{A}^{(n_{p})}
	\end{array}\right]. 
	$$
	Thus, the operation $\mathtt { fold }(\cdot)$ takes an $n_{1} n_{p} \times n_{2} \times \cdots \times n_{p-1}$ block tensor and returns an $n_{1} \times \cdots \times n_{p}$ tensor. That is,
	$$
	\mathtt{fold}\left(\mathtt{unfold}(\mathcal{A})\right)=\mathcal{A}.
	$$
\end{definition}

\begin{definition}[Tensor transpose for order-$p$ tensors ($p\ge3$), \cite{martin2013order}]
	If $\mathcal{A}$ is $n_{1} \times \cdots \times n_{p}$, then $\mathcal{A}^{T}$ is the $n_{2} \times n_{1} \times n_{3} \times \cdots \times n_{p}$
	tensor obtained by tensor transposing each $\mathcal{A}^{(i)}$ for $i=1, \ldots, n_{p}$, and then reversing the order of the $\mathcal{A}^{(i)}$ for $2$ through $n_{p} .$ In other words, 
	$$
	\mathcal{A}^{T}= \mathtt{ fold }\left(\left[\begin{array}{c}
		(\mathcal{A}^{(1)})^{T} \\
		(\mathcal{A}^{(n_{p})})^{T} \\
		(\mathcal{A}^{(n_{p-1})})^{T} \\
		\vdots \\
		(\mathcal{A}^{(n_{2})})^{T}
	\end{array}\right]\right).
	$$
	For complex tensor, the tensor transpose is conjugate.
\end{definition}

\begin{definition}[Identity tensor for third-order tensors, \cite{kilmer2013third}]	
	The identity tensor $\mathbf{\mathcal{I}} \in \mathbb{R}^{n \times n \times n_{3}}$ is the tensor whose first frontal
	slice is the $n\times n$ identity matrix, and other frontal slices are all zeros.
\end{definition}

\begin{definition}[Identity tensor for order-$p$ tensors ($p>3$), \cite{martin2013order}]	
	The $n \times n \times \ell_{1}\times \cdots \times \ell_{p-2}$ order-$p$ identity tensor $ \mathcal{I}$ is the tensor such that $\mathcal{I}^{(1)}$ is the $n \times n \times \ell_{1}\times  \cdots \times \ell_{p-3}$ order-$(p-1)$ identity tensor, and $\mathcal{I}^{(j)}$ is the order-$(p-1)$ zero tensor for $j=2, \ldots, \ell_{p-2}$.
\end{definition}

\begin{definition}[Orthogonal tensor for order-$p$ tensors ($p\ge3$), \cite{martin2013order}] 
	A real-valued tensor $\mathbf{\mathcal{Q}} \in \mathbb{R}^{n \times n \times \ell_{1} \times \cdots \times \ell_{p-2}}$ is orthogonal if it satisfies
	$
	\mathbf{\mathcal{Q}}^{T}*\mathbf{\mathcal{Q}}=\mathbf{\mathcal{Q}}*\mathbf{\mathcal{Q}}^{T}=\mathbf{\mathcal{I}}
	$. A real-valued tensor $\mathbf{\mathcal{Q}} \in \mathbb{R}^{p \times q \times  \ell_{1} \times \cdots \times \ell_{p-2}}$ is partially orthogonal if it satisfies
	$
	\mathbf{\mathcal{Q}}^{T}*\mathbf{\mathcal{Q}}=\mathbf{\mathcal{I}}.
	$
\end{definition}

\begin{definition}[f-diagonal tensor for order-$p$ tensors ($p\ge3$), \cite{martin2013order}]	
	The f-diagonal tensor $\mathcal{A}$ has the property that $\mathcal{A}_{i_{1} i_{2} \ldots i_{p}}=0$ unless $i_{1}=i_{2} .$ 
\end{definition}

\begin{definition}[$\ell_{2^{*}}$ norm of tensor column for order-$p$ tensors ($p\ge3$), \cite{zhang2016exact}]	
	Let $\vec{\boldsymbol{x}}$ be an $n_{1} \times 1 \times n_{3}\times \cdots\times n_{p}$ tensor column, the $\ell_{2^{*}}$ norm denotes
	$$\|\vec{\boldsymbol{x}}\|_{2^{*}}=\sqrt{\sum_{i_1=1}^{n_{1}} \sum_{i_3=1}^{n_{3}} \cdots \sum_{i_p=1}^{n_{p}}\vec{\boldsymbol{x}}_{i_1 1 i_3\ldots i_p}^{2}}.$$
\end{definition}

\begin{definition}[Tensor spectral norm for order-$p$ tensors ($p\ge3$)]\label{def_tsn}
	Given $\mathcal{A} \in \mathbb{R}^{n_{1} \times n_{2} \times \cdots \times n_{p}}$ and $\mathcal{V} \in \mathbb{R}^{n_{2} \times 1 \times \cdots \times n_{p}} $, the tensor spectral norm is defined as 	
	\begin{align}
		\|\mathcal{A}\|: &=\sup _{\|\mathcal{V}\|_{F} \leq 1}\|\mathcal{A} * \mathcal{V}\|_{F} \notag\\
		&=\sup _{\|\mathcal{V}\|_{F} \leq 1}\|\operatorname{bcirc}(\mathcal{A}) * \operatorname{unfold}(\mathcal{V})\|_{F} \notag \\
		&=\sup _{\|\mathcal{V}\|_{F} \leq 1}\|\boldsymbol{\tilde{A}} \cdot \boldsymbol{\hat{V}} \|_{F}\notag\\
		&=\|\boldsymbol{\tilde{A}}\|  \label{msnp}\\
		&=\|\boldsymbol{\bar{A}}\| ,\label{simprop}
	\end{align}	
	where $\boldsymbol{\hat{V}}$ is the $  n_2\rho\times 1$ unfold matrix formed from $\mathbf{\mathcal{V}}$ in the base
	level of recursion. \end{definition}
	It is not hard to check that the equation (\ref{msnp}) holds by the definition of matrix spectral norm, and the equation (\ref{simprop}) holds by combining (\ref{orthop}) and the property that $\left(\boldsymbol{\tilde{F}} \otimes \boldsymbol{I}_{n_{1}}\right) / \sqrt{\rho}$ is orthogonal.

\begin{definition}[t-SVD for order-$p$ tensors ($p\ge3$), \cite{martin2013order}]
	Let $\mathcal{A}$ be an $n_{1} \times \cdots \times n_{p}$ real-valued tensor. Then $\mathcal{A}$ can be factored as
	$$
	\mathcal{A}=\mathcal{U} * \mathcal{S} * \mathcal{V}^{T}
	$$
	where $\mathcal{U}, \mathcal{V}$ are orthogonal $n_{1} \times n_{1} \times n_{3} \times n_{4} \times \cdots \times n_{p}$ and $n_{2} \times n_{2} \times n_{3} \times n_{4} \times \cdots \times n_{p}$ tensors
	respectively, and $\mathcal{S}$ is an $n_{1} \times n_{2} \times \cdots \times n_{p}$ f-diagonal tensor. The factorization is called the t-SVD.
\end{definition}
Basically,  t-SVD can be computed efficiently by the following steps:
\begin{itemize}
	\item[1.] Compute $\mathbf{\mathcal{A}}=\mathtt{fft}(\mathbf{\mathcal{A}},[ ], i)$, for $i=3,\ldots,p$;
	\item[2.] Set $\mathbf{\bar{\mathcal{A}}}:=\mathbf{\mathcal{A}}$; 
	\item[3.] Compute matrix SVD $\boldsymbol{\bar{A}}^{(i)}=\boldsymbol{\bar{U}}^{(i)} \boldsymbol{\bar{S}}^{(i)} \boldsymbol{\bar{V}}^{(i)*}$ for each frontal slice, $ i=1,\ldots,\rho$;
	\item[4.] Compute $\mathbf{\bar{\mathcal{U}}}=\mathtt{ifft}(\mathbf{\bar{\mathcal{U}}}, [], i)$, $\mathbf{\bar{\mathcal{S}}}=\mathtt{ifft}(\mathbf{\bar{\mathcal{S}}}, [], i),$ and $\mathbf{\bar{\mathcal{V}}}=\mathtt{ifft}(\mathbf{\bar{\mathcal{V}}}, [], i)$, for $i=3,\ldots,p$; 
	\item[5.] Set $\mathbf{\mathcal{U}}:=\mathbf{\bar{\mathcal{U}}}$, $\mathbf{\mathcal{S}}:=\mathbf{\bar{\mathcal{S}}}$ and $\mathbf{\mathcal{V}}:=\mathbf{\bar{\mathcal{V}}}$. 
\end{itemize}

\begin{definition}[Tensor tubal rank for third-order tensors, \cite{lu2019tensor}] \label{tubalrank3}
	For $\mathbf{\mathcal{A}} \in$ $\mathbb{R}^{n_{1} \times n_{2} \times n_{3}},$ the tensor tubal rank, denoted as $\operatorname{rank}_{t}(\mathbf{\mathcal{A}}),$ is defined as the number of nonzero singular tubes of $\mathbf{\mathcal{S}},$ where $\mathbf{\mathcal{S}}$ is from the t-SVD of $\mathbf{\mathcal{A}}=\mathbf{\mathcal{U}} * \mathbf{\mathcal{S} }* \mathbf{\mathcal{V}}^{T} .$ We can write
	$$
	\operatorname{rank}_{t}(\mathbf{\mathcal{A}})=\#\{i, \mathbf{\mathcal{S}}(i, i, 1) \neq 0\}=\#\{i, \mathbf{\mathcal{S}}(i, i,:) \neq 0\}.
	$$
\end{definition}

	In the following, we shall extend the above definition to more general order-$p$ case with $p>3$. 
	\begin{definition}[Tensor tubal rank for order-$p$ tensors ($p>3$)] 
		For $\mathbf{\mathcal{A}} \in$ $\mathbb{R}^{n_{1} \times \cdots \times n_{p}},$ the tensor tubal rank, denoted as $\operatorname{rank}_{t}(\mathbf{\mathcal{A}}),$ is defined as the number of nonzero singular scalars of $\mathbf{\mathcal{S}},$ where $\mathbf{\mathcal{S}}$ is from the t-SVD of $\mathbf{\mathcal{A}}=\mathbf{\mathcal{U}} * \mathbf{\mathcal{S} }* \mathbf{\mathcal{V}}^{T} .$ We can write
		\begin{align}
			\operatorname{rank}_{t}(\mathbf{\mathcal{A}})&=\#\{i, \mathbf{\mathcal{S}}(i, i, 1, \ldots ,1) \neq 0\}\notag\\
			&=\#\{i, \mathbf{\mathcal{S}}(i, i,:,\ldots,:) \neq 0\}.\notag
		\end{align}
	\end{definition}

\begin{lemma}[Best tubal-rank-$k$ approximation for third-order tensors, \cite{kilmer2011factorization}]
	Let the t-SVD of $\mathbf{\mathcal{A}} \in \mathbb{R}^{n_{1} \times n_{2} \times n_{3}}$ be $
	\mathbf{\mathcal{A}}=\mathbf{\mathcal{U}} * \mathbf{\mathcal{S}} * \mathbf{\mathcal{V}}^{T}
	$. For a given positive
	integer $k$, define $\mathbf{\mathcal{A}}_{k}=\sum_{s=1}^{k} \mathbf{\mathcal{U}}(:, s,:) * \mathbf{\mathcal{S}}(s, s,:) * \mathbf{\mathcal{V}}^{T}(:, s,:)$. Then $\mathbf{\mathcal{A}}_{k}=\underset{\mathbf{\hat{\mathcal{A}}} \in \mathbb{A}}{\arg \min }\|\mathbf{\mathcal{A}}-\mathbf{\hat{\mathcal{A}}}\|_{F},$ where $\mathbb{A}=\left\{\mathcal{X} * \mathcal{Y}^{T} | \mathcal{X} \in\mathbb{R}^{n_{1} \times k \times n_{3}}, \mathcal{Y} \in \mathbb{R}^{n_{2} \times k \times n_{3}}\right\}$. This means that $\mathbf{\mathcal{A}}_{k}$ is the approximation of $\mathbf{\mathcal{A}}$ with the tubal rank at most $k$.
\end{lemma}

	The extension of above result to general order-$p$ tensors is presented as follows, and the proof will be given in Section VI.
\begin{lemma}[Best tubal-rank-$k$ approximation for order-$p$ tensors ($p>3$)]\label{lemma2}
	Let the t-SVD of $\mathbf{\mathcal{A}} \in \mathbb{R}^{n_{1} \times \cdots \times n_{p}}$ be $
	\mathbf{\mathcal{A}}=\mathbf{\mathcal{U}} * \mathbf{\mathcal{S}} * \mathbf{\mathcal{V}}^{T}
	$. For a given positive
	integer $k$, define $\mathbf{\mathcal{A}}_{k}=\sum_{s=1}^{k} \mathbf{\mathcal{U}}(:, s,:,\ldots,:) * \mathbf{\mathcal{S}}(s, s,:,\ldots,:) * \mathbf{\mathcal{V}}^{T}(:, s,:,\ldots,:)$. Then $\mathbf{\mathcal{A}}_{k}=\underset{\mathbf{\hat{\mathcal{A}}} \in \mathbb{A}}{\arg \min }\|\mathbf{\mathcal{A}}-\mathbf{\hat{\mathcal{A}}}\|_{F},$ where $\mathbb{A}=\left\{\mathcal{X} * \mathcal{Y}^{T} | \mathcal{X} \in\mathbb{R}^{n_{1} \times k \times n_{3} \times \cdots \times n_{p}}, \mathcal{Y} \in \mathbb{R}^{n_{2} \times k \times n_{3} \times \cdots \times n_{p}}\right\}$. This means that $\mathbf{\mathcal{A}}_{k}$ is the approximation of $\mathbf{\mathcal{A}}$ with the tubal rank at most $k$.
\end{lemma}

\section{The proposed t-FD algorithm}

In this section, we shall first focus on deriving the algorithmic procedure and conducting the corresponding theoretical analysis for the third-order tensors, and then extend these to more general order-$p$ tensors with $p>3$. 
\subsection{Algorithmic procedure}

Here we first briefly describe the core idea of matrix FD. It receives the input matrix $\boldsymbol{A} \in \mathbb{R}^{n \times d}$ in a streaming fashion, and produces the sketch matrix $\boldsymbol{B} \in \mathbb{R}^{\ell \times d} $ which contains only $\ell \ll n$ rows but still approximates well for the original matrix $\boldsymbol{A}$. Specifically, $\boldsymbol{B}$ is first initialized to an all-zero valued matrix and then receives each row of matrix $\boldsymbol{A}$ one after the other. Once $\boldsymbol{B}$ is full, we orthogonalize $\boldsymbol{B}$ by taking SVD and then implement a shrinking procedure to construct the new sketch matrix, which are repeated throughout the entire streaming data. As shown in~\cite{ghashami2016frequent}, the algorithm guarantees that 
\begin{equation}
\label{eq:fd}
\left\|\boldsymbol{A}^{ T} \boldsymbol{A}-\boldsymbol{B}^{T} \boldsymbol{B}\right\|_{2}
\leq  \frac{   \left\|\boldsymbol{A}-\boldsymbol{A}_{k}\right\|_{F}^{2}}{\ell-k}.
\end{equation}
Note that setting\begin{small} $\ell=\lceil k+1/\varepsilon\rceil$\end{small} yields the error of \begin{small}$\varepsilon \left\|\boldsymbol{A}-\boldsymbol{A}_{k}\right\|_{F}^{2}$\end{small}, that is to say, the sketch matrix $\boldsymbol{B}$ is within $(1+\varepsilon)$ best low-rank approximation. 

For the higher order tensor case, although there has been many explorations focused on the Tucker/CP decomposition dealing with streaming data, the random techniques or complicated optimization strategies are required for getting a good low-rank approximation. Motivated by the matrix FD, we propose a simple and deterministic tensor sketching algorithm (t-FD) as stated in Algorithm $\ref{tensor-FD}$ to get a low-tubal-rank approximation from streaming data.  Our goal is to find a small sketch $\mathcal{B}$ that could derive a similar error bound as (\ref{eq:fd}). Assume we have $n_1$ data samples $\mathcal{A}_1, ..., \mathcal{A}_{n_1}$ of size $n_2 \times n_3$, and they are received sequentially, we arrange them to a third-order tensor $\mathcal{A} \in \mathbb{R}^{n_1 \times n_2 \times n_3}$. By utilizing the well defined algebraic structure of tensor-tensor product, the idea of matrix FD could be extended naturally and used to update the components of sketch $\mathcal{B}$.   
\begin{figure*}
	\centering
	\includegraphics[width=0.68\textwidth]{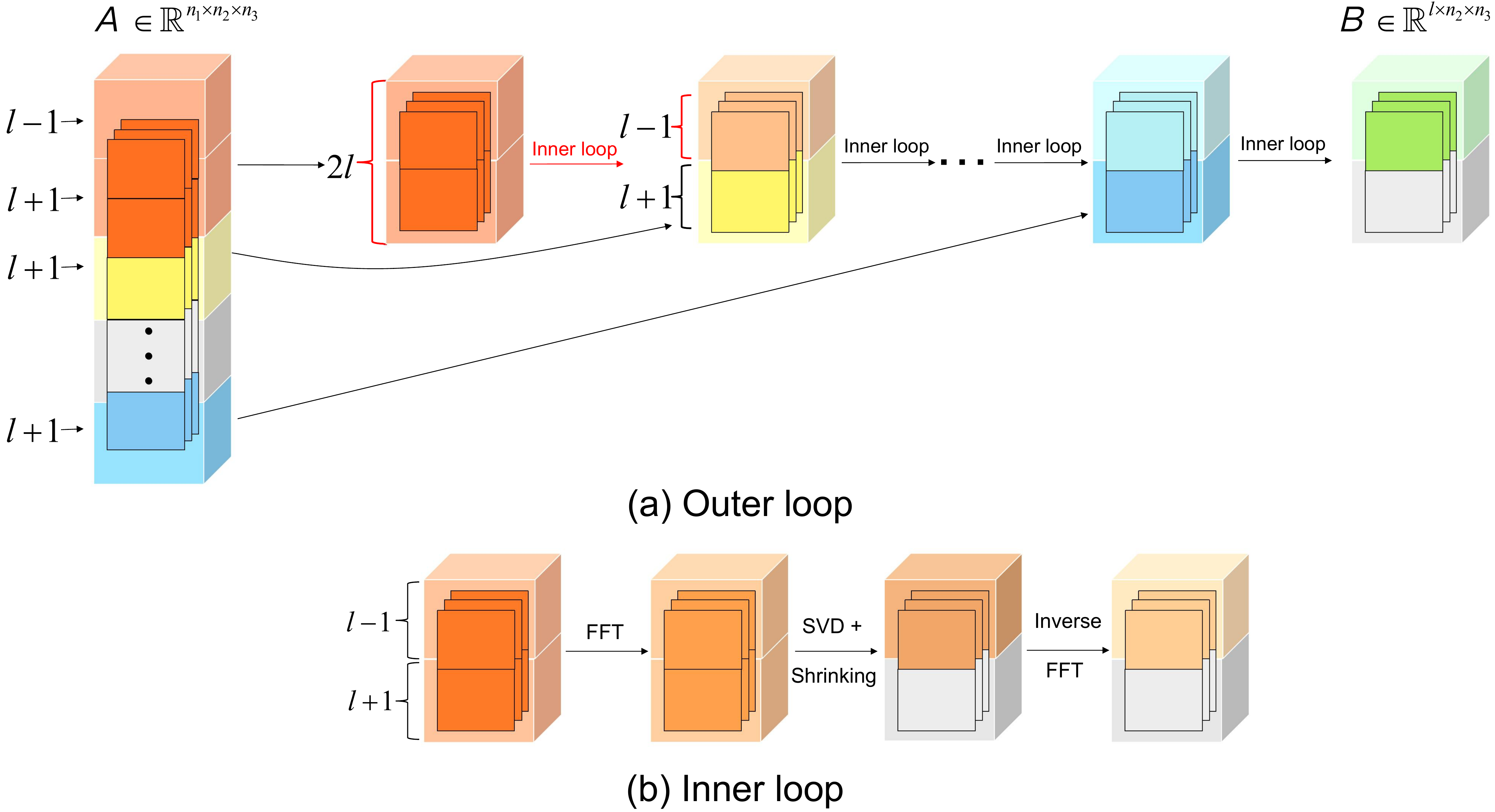} 
	\caption{Illustration of t-FD.}
	\label{figtFD}
\end{figure*}

During the implementation  of Algorithm 1, $\mathbf{\mathcal{B}}$ is first initialized to an all-zero valued tensor, and then the all zero valued horizontal slices in $\mathbf{\mathcal{B}}$ are simply replaced by the horizontal slices from $\mathbf{\mathcal{A}}$. Intuitively speaking, each frontal slice $\boldsymbol{B}^{(i)}\ ( i=1,\ldots,n_{3})$ receives one row of $\boldsymbol{A}^{(i)}$ each time one after the other as shown in step 4. Afterwards, the last $\ell+1$ horizontal slices are nullified by a four-stage process (steps 8-12 in Algorithm 1). First, we use the Fast Fourier Transform (FFT) to get $\mathbf{\bar{\mathcal{B}}}$. Then, each frontal slice $\boldsymbol{\bar{B}}^{(i)}$ is rotated (from the left) using its SVD such that its rows are orthogonal and in descending magnitude order. Further then, the rotated frontal slice is implemented by a shrinking procedure to make the last $\ell+1$ rows be zero. Finally, compute the sketch $\mathbf{\mathcal{B}}$ from $\mathbf{\bar{\mathcal{B}}}$ via the inverse FFT. More detailed procedure is illustrated in Fig. \ref{figtFD}.

\begin{algorithm}[!htb]
	\caption{tensor FD (t-FD) for third-order tensors}
	\label{tensor-FD}
	\begin{algorithmic}[1]
		\REQUIRE $\mathbf{\mathcal{A}} \in \mathbb{R}^{n_{1} \times n_{2} \times n_{3}}, \text { sketch size } \ell$\\
		\ENSURE $\mathbf{\mathcal{B}} \in \mathbb{R}^{\ell \times n_{2} \times n_{3}}$\\
		\STATE Initialize $\mathcal{B} \in \mathbb{R}^{2\ell \times n_2 \times n_3}$ as a tensor with all elements being zero
		\FOR{$j = 1, \ldots, n_{1}$} 
		\STATE Initialize $\delta_{j}^{(i)}=0 , i=1,\ldots,n_3$
		\STATE Insert $\mathcal{A}_j$ into a zero valued horizontal slice of $\mathcal{B}$\\
		\IF{$\mathcal{B}$ is fulfilled}
		\STATE Compute $\mathbf{\bar{\mathcal{B}}}=\mathtt{fft}(\mathbf{\mathcal{B}},[], 3)$\\
		\FOR{$i = 1, \ldots, n_{3}$} 
		\STATE $\left[\boldsymbol{\bar{U}}^{(i)}, \boldsymbol{\bar{S}}^{(i)}, \boldsymbol{\bar{V}}^{(i)}\right] \leftarrow \operatorname{svd}(\boldsymbol{\bar{B}}^{(i)})$\\
		\STATE $\boldsymbol{\bar{C}}^{(i)} \leftarrow \boldsymbol{\bar{S}}^{(i)}\boldsymbol{\bar{V}}^{(i)T}$\\
		\STATE $\delta_{j}^{(i)} \leftarrow s_{\ell}^{(i)2}$\\
		\STATE $\boldsymbol{\bar{B}}^{(i)} \leftarrow \sqrt{\max(\boldsymbol{\bar{S}}^{(i)2}-\delta_{j}^{(i)} \boldsymbol{I}_{2\ell},0)} \cdot \boldsymbol{\bar{V}}^{(i)T}$
		\STATE $\mathbf{\bar{\mathcal{B}}}^{(i)} \leftarrow \boldsymbol{\bar{B}}^{(i)}$\\
		\ENDFOR
		\STATE Compute $\mathbf{\mathcal{B}}=\mathtt{ifft}(\mathbf{\bar{\mathcal{B}}},[], 3)$\\
		\ENDIF
		\ENDFOR	
		\STATE Set $\mathbf{\mathcal{B}}\leftarrow \mathbf{\mathcal{B}}(1 :\ell, :, :)$
	\end{algorithmic}
\end{algorithm}

As for the theoretical analysis of the new algorithm, it becomes more challenging. First, there are several tensor norms which are more complicated, e.g., the tensor spectral norm. Second, since the truncated procedure is implemented in the Fourier domain, the relationship between the original tensor and the sketch tensor is hard to derive explicitly. So the proof techniques cannot directly move from  the matrix FD algorithm to the new t-FD algorithm.

\subsection{Error bounds}
This subsection presents our main theoretical results for Algorithm \ref{tensor-FD}. In the subsequent analysis, we used
two different error metrics to evaluate the distance between original tensor $\mathbf{\mathcal{A}}$ and sketch tensor $\mathbf{\mathcal{B}}$.

The first error metric is the tensor covariance error, $\left\|\mathbf{\mathcal{A}}^{T}*\mathbf{\mathcal{A}}-\mathbf{\mathcal{B}}^{T}*\mathbf{\mathcal{B}}\right\|$, which is used to measure the maximal singular value gap between the block diagonal matrices $\boldsymbol{\bar{A}}$ and $\boldsymbol{\bar{B}}$. This can be easily verified by the definition of tensor spectral norm (Def. \ref{def_tsn}). The tensor covariance error is given by the following theorem.

\begin{theorem}[Tensor covariance error]\label{main1}
	Given $\mathbf{\mathcal{A}} \in \mathbb{R}^{n_{1} \times n_{2} \times n_{3}}$ and the sketch size $\ell$, the sketch $\mathbf{\mathcal{B}} \in \mathbb{R}^{\ell \times n_{2} \times n_{3}}$ is constructed by Algorithm \ref{tensor-FD}, then for any $k<\frac{\ell}{c}$,
	$$
	\left\|\mathbf{\mathcal{A}}^{T}*\mathbf{\mathcal{A}}-\mathbf{\mathcal{B}}^{T}*\mathbf{\mathcal{B}}\right\|\le\frac{\left\|\mathbf{\mathcal{A}}-\mathbf{\mathcal{A}}_{k}\right\|_{F}^{2}}{\frac{\ell}{c}-k},$$
	where $c=\frac{n_{3}\sum_{j=1}^{n_1}\max\limits_{i} \delta_{j}^{(i)} }{\sum_{j=1}^{n_1}\sum_{i=1}^{n_{3}} \delta_{j}^{(i)}}$.
\end{theorem}

The second error metric is the tensor projection error, $\left\|\mathbf{\mathcal{A}}-\mathbf{\mathcal{A}} * \mathbf{\mathcal{V}}_{k} * \mathbf{\mathcal{V}}_{k}^{T}\right\|_{F}^{2}$, where $\mathbf{\mathcal{A}} * \mathbf{\mathcal{V}}_{k} * \mathbf{\mathcal{V}}_{k}^{T}$ denotes the projection of $\mathbf{\mathcal{A}}$ on the rank $k$ right orthogonal tensor of $\mathbf{\mathcal{B}}$. Intuitively, the tensor projection error measures the deviation generated during the projection process, and further indicates how accurate the choice of subspace is. The detailed analysis is shown in the following theorem.
\begin{theorem}[Tensor projection error]\label{main}
	Given $\mathbf{\mathcal{A}} \in \mathbb{R}^{n_{1} \times n_{2} \times n_{3}}$ and the sketch size $\ell$, the sketch tensor $\mathbf{\mathcal{B}} \in \mathbb{R}^{\ell \times n_{2} \times n_{3}}$ is constructed by Algorithm \ref{tensor-FD}. Let $\mathbf{\mathcal{B}}=\mathbf{\mathcal{U}}*\mathbf{\mathcal{S}}*\mathbf{\mathcal{V}}^{T}$, and the tubal-rank-$k$ approximation is $\mathbf{\mathcal{B}}_{k}=\mathbf{\mathcal{U}}_{k}*\mathbf{\mathcal{S}}_{k}*\mathbf{\mathcal{V}}_{k}^{T}$. For any $k<\frac{\ell}{c}$, we have that
	\begin{align}\label{tpecompare}
	\left\|\mathbf{\mathcal{A}}-\mathbf{\mathcal{A}} * \mathbf{\mathcal{V}}_{k} * \mathbf{\mathcal{V}}_{k}^{T}\right\|_{F}^{2} \le \frac{\ell}{\ell-ck}\left\|\mathbf{\mathcal{A}}-\mathbf{\mathcal{A}}_{k}\right\|_{F}^{2}, 
	\end{align}
	where $c=\frac{n_{3}\sum_{j=1}^{n_1}\max\limits_{i} \delta_{j}^{(i)} }{\sum_{j=1}^{n_1}\sum_{i=1}^{n_{3}} \delta_{j}^{(i)}}$. 
\end{theorem}

\begin{remark} If we set $\ell=c\lceil k+k / \varepsilon\rceil$, we can get the standard $(1+\epsilon)$ bound form as $$     \left\|\mathbf{\mathcal{A}}-\mathbf{\mathcal{A}} * \mathbf{\mathcal{V}}_{k} * \mathbf{\mathcal{V}}_{k}^{T}\right\|_{F}^{2} \le (1+\varepsilon)\left\|\mathbf{\mathcal{A}}-\mathbf{\mathcal{A}}_{k}\right\|_{F}^{2}.$$ Note that when the rank $k$ and parameter $c$ are fixed, as the sketch size increases, $ \varepsilon$ would decrease linearly, thus we could achieve nearly optimal deterministic error bound. Moreover, when the third dimension $n_3$ is $1$, our algorithm reduce to the matrix FD, that is to say, matrix FD is a special case of t-FD. Remember that we could identify the parameter $c = 1$ when $n_3 = 1$, then the theoretical guarantee of matrix FD is a special case of our Theorems 1 and 2. \end{remark}

\begin{algorithm}[!htb]
	\caption{tensor FD (t-FD) for order-$p$ tensors ($p > 3$)}
	\label{tensor-FDp}
	\begin{algorithmic}[1]
		\REQUIRE $\mathbf{\mathcal{A}} \in \mathbb{R}^{n_{1} \times n_{2} \times \cdots\times n_{p}}, \text { sketch size } \ell$\\
		\ENSURE $\mathbf{\mathcal{B}} \in \mathbb{R}^{\ell \times n_{2} \times \cdots\times n_{p}}$\\
		\STATE Initialize $\mathcal{B} \in \mathbb{R}^{2\ell \times n_{2} \times \cdots\times n_{p}}$ as a tensor with all elements being zero
		\FOR{$j=1, \ldots, n_{1}$} 
		\STATE Initialize $\delta_{j}^{(i)}=0 , i=1,\ldots,\rho$
		\STATE Insert $\mathcal{A}_j$ into a zero valued horizontal tensor of $\mathcal{B}$, where $\mathbf{\mathcal{A}} =[\mathbf{\mathcal{A}}_1, ..., \mathbf{\mathcal{A}}_{n_1} ]$ and $\mathcal{A}_j \in \mathbb{R}^{1 \times n_2  \times\cdots \times n_{p}}$\\
		\IF{$\mathcal{B}$ is fulfilled}
		\FOR{$i=3,\ldots,p$}
		\STATE Compute $\mathbf{\mathcal{B}}=\operatorname{fft}(\mathbf{\mathcal{B}},[], i)$\\
		\ENDFOR
		\STATE Set $\mathbf{\bar{\mathcal{B}}}=\mathcal{B}$
		\FOR{$i=1,\ldots,\rho$}
		\STATE $\left[\boldsymbol{\bar{U}}^{(i)}, \boldsymbol{\bar{S}}^{(i)}, \boldsymbol{\bar{V}}^{(i)}\right] \leftarrow \operatorname{svd}(\boldsymbol{\bar{B}}^{(i)})$\\
		\STATE $\boldsymbol{\bar{C}}^{(i)} \leftarrow \boldsymbol{\bar{S}}^{(i)}\boldsymbol{\bar{V}}^{(i)T}$\\
		\STATE $\delta_{j}^{(i)} \leftarrow s_{\ell}^{(i)2}$\\
		\STATE $\boldsymbol{\bar{B}}^{(i)} \leftarrow \sqrt{\max(\boldsymbol{\bar{S}}^{(i)2}-\delta_{j}^{(i)} \boldsymbol{I}_{2\ell},0)} \cdot \boldsymbol{\bar{V}}^{(i)T}$
		\STATE $\mathbf{\bar{\mathcal{B}}}^{(i)} \leftarrow \boldsymbol{\bar{B}}^{(i)}$\\
		\ENDFOR
		\FOR{$i=3,\ldots,p$}
		\STATE Compute $\mathbf{\bar{\mathcal{B}}}=\operatorname{ifft}(\mathbf{\bar{\mathcal{B}}},[], i)$\\
		\ENDFOR
		\STATE Set $\mathcal{B}=\mathbf{\bar{\mathcal{B}}}$
		\ENDIF
		\ENDFOR	
		\STATE Return $\mathbf{\mathcal{B}}$
	\end{algorithmic}
\end{algorithm}

\begin{remark}It should be pointed out, as detailedly shown in Section VI, our proof techniques are different from the matrix FD case in two aspects. Firstly, t-FD algorithm based on t-SVD computes on the Fourier domain, which leads us to use tensor  operators such as $\mathtt{bcirc}$ and $\mathtt{unfold}$ as a bridge to find the relationship between original and Fourier domains. Secondly, to bound the information loss in each iteration requires us to utilize the relationship among frontal slices.
	
\end{remark}

\subsection{Extension to order-$p$ tensors ($p > 3$)}

	We now state the algorithmic procedure for the general case of order-$p$ tensors ($p >3$) in Algorithm \ref{tensor-FDp}. Even though the above results are focused on the third-order tensors,  these can be fairly easy to generalize for higher order tensors, by combing the well-defined algebraic framework of order-$p$ t-SVD \cite{martin2013order} ($p > 3$) along with the properties of the discrete Fourier transform (DFT) matrix. It is worth mentioning that for the original tensor $\mathbf{\mathcal{A}} \in \mathbb{R}^{n_{1} \times n_{2} \times \cdots\times n_{p}}$ and the sketch tensor $\mathbf{\mathcal{B}} \in \mathbb{R}^{\ell \times n_{2} \times \cdots\times n_{p}}$, the tensor covariance and projection errors are similar to the third-order case, with the only difference being that $c$ changes from $\frac{n_{3}\sum_{j=1}^{n_1}\max\limits_{i} \delta_{j}^{(i)} }{\sum_{j=1}^{n_1}\sum_{i=1}^{n_{3}} \delta_{j}^{(i)}}$ to $\frac{\rho\sum_{j=1}^{n_1}\max\limits_{i} \delta_{j}^{(i)} }{\sum_{j=1}^{n_1}\sum_{i=1}^{\rho} \delta_{j}^{(i)}}$. The details of such error bounds are shown below. 
\begin{theorem}[Tensor covariance error for order-$p$ tensors ($p > 3$)]\label{tce-p-order}
	Given $\mathbf{\mathcal{A}} \in \mathbb{R}^{n_{1} \times n_{2} \times \cdots\times n_{p}}$ and the sketch size $\ell$, the sketch $\mathbf{\mathcal{B}} \in \mathbb{R}^{\ell \times n_{2} \times \cdots\times n_{p}}$ is constructed by Algorithm \ref{tensor-FDp}, then for any $k<\frac{\ell}{c}$,
	$$
	\left\|\mathbf{\mathcal{A}}^{T}*\mathbf{\mathcal{A}}-\mathbf{\mathcal{B}}^{T}*\mathbf{\mathcal{B}}\right\|\le\frac{\left\|\mathbf{\mathcal{A}}-\mathbf{\mathcal{A}}_{k}\right\|_{F}^{2}}{\frac{\ell}{c}-k},$$
	where $c=\frac{\rho\sum_{j=1}^{n_1}\max\limits_{i} \delta_{j}^{(i)} }{\sum_{j=1}^{n_1}\sum_{i=1}^{\rho} \delta_{j}^{(i)}}$ and $\rho=n_{3}n_{4}\ldots n_p$.
\end{theorem}
\begin{theorem}[Tensor projection error for order-$p$ tensors ($p > 3$)]\label{tpe-p-order}
	Given $\mathbf{\mathcal{A}} \in \mathbb{R}^{n_{1} \times n_{2} \times \cdots\times n_{p}}$ and the sketch size $\ell$, the sketch tensor $\mathbf{\mathcal{B}} \in \mathbb{R}^{\ell \times n_{2} \times \cdots\times n_{p}}$ is constructed by Algorithm \ref{tensor-FDp}. Let $\mathbf{\mathcal{B}}=\mathbf{\mathcal{U}}*\mathbf{\mathcal{S}}*\mathbf{\mathcal{V}}^{T}$, and the tubal-rank-$k$ approximation is $\mathbf{\mathcal{B}}_{k}=\mathbf{\mathcal{U}}_{k}*\mathbf{\mathcal{S}}_{k}*\mathbf{\mathcal{V}}_{k}^{T}$. For any $k<\frac{\ell}{c}$, we have that
	\begin{align}
	\left\|\mathbf{\mathcal{A}}-\mathbf{\mathcal{A}} * \mathbf{\mathcal{V}}_{k} * \mathbf{\mathcal{V}}_{k}^{T}\right\|_{F}^{2} \le \frac{\ell}{\ell-ck}\left\|\mathbf{\mathcal{A}}-\mathbf{\mathcal{A}}_{k}\right\|_{F}^{2}, \label{thm4111}
	\end{align}
	where $c=\frac{\rho\sum_{j=1}^{n_1}\max\limits_{i} \delta_{j}^{(i)} }{\sum_{j=1}^{n_1}\sum_{i=1}^{\rho} \delta_{j}^{(i)}}$ and $\rho=n_{3}n_{4}\ldots n_p$. 
\end{theorem}

\subsection{Comparison to matricization FD}
A naive approach for tackling  tensor data is the so-called matricization technique, that is, vectorizing the horizontal slices separately and then regarding it as a matrix. We thus set up the following comparison algorithm. For the upcoming data sample $\mathcal{A}_{i}$, we convert it to
the mode-1 unfolding matrix and then utilize it to update the sketch matrix $\boldsymbol{B}$. Lastly, a folding procedure is used to obtain the tensor $\mathbf{\mathcal{B}}$. One can see  Algorithm $\ref{Matricization-tensorFD}$ for more details.

\begin{algorithm}[!htb]
	\caption{Matricization-tensor-FD (MtFD)}
	\label{Matricization-tensorFD}
	\begin{algorithmic}[1]
		\REQUIRE $\mathbf{\mathcal{A}} \in \mathbb{R}^{n_{1} \times n_{2} \times\cdots\times n_{p}}, \text { sketch size } \ell$\\
		\ENSURE $\mathbf{\mathcal{B}} \in \mathbb{R}^{\ell \times n_{2}\times\cdots \times n_{p}}$\\
		\STATE Get the mode-1 unfolding matrix $\boldsymbol{A}_{(1)}\in \mathbb{R}^{n_{1} \times n_{2}  \rho}$
		\STATE Compute $\boldsymbol{B}_{(1)}\in \mathbb{R}^{\ell \times n_{2}\rho}$ = FD($\boldsymbol{A}_{(1)}$)
		\STATE Get the mode-1 folding tensor $\mathbf{\mathcal{B}} \in \mathbb{R}^{\ell \times n_{2} \times\cdots\times n_{p}}$
	\end{algorithmic}
\end{algorithm}

For this matricization algorithm, we can also derive a simple covariance error bound.
\begin{theorem}\label{thm-norm2-matrization}
		Given $\mathbf{\mathcal{A}} \in \mathbb{R}^{n_{1} \times n_{2} \times\cdots\times n_{p}}$ and sketch size $\ell$, the sketch $\mathbf{\mathcal{B}} \in \mathbb{R}^{\ell \times n_{2} \times \cdots \times n_{p}}$ is constructed by Algorithm \ref{Matricization-tensorFD}, then 
		$$
		\left\|\mathbf{\mathcal{A}}^{T}*\mathbf{\mathcal{A}}-\mathbf{\mathcal{B}}^{T}*\mathbf{\mathcal{B}}\right\| \le \frac{\rho}{\ell-k}\left\|\mathbf{\mathcal{A}}\right\|_{F}^{2}.
		$$
\end{theorem}

Since the unfolding operation destroys the correlations among frontal slices, we couldn't derive a tighter bound in the form of $\left\|\mathbf{\mathcal{A}}-\mathbf{\mathcal{A}}_{k}\right\|_{F}^{2}$. However, through the inequality $\left\|\mathbf{\mathcal{A}}-\mathbf{\mathcal{A}}_{k}\right\|_{F}^{2} \leq \left\|\mathbf{\mathcal{A}}\right\|_{F}^{2}$, we are able to obtain a weaker error bound for Theorem \ref{tce-p-order} as $\left\|\mathbf{\mathcal{A}}\right\|_{F}^{2} /( {\frac{\ell}{c}-k})$.  In this case, when $\ell \ge \left(1+\frac{1-1/c}{1/c-1/\rho}\right)k$, we achieve a smaller covariance  error bound. Even though the parameter $c$ could not be explicitly calculated, the extensive numerical experiments would later demonstrate that $c$ is much smaller than $\rho$ generally. We also notice that when the tensor $\mathcal{A}$ satisfies the low-tubal-rank assumption,  $\left\|\mathbf{\mathcal{A}}-\mathbf{\mathcal{A}}_{k}\right\|_{F}^{2} $ would be much smaller than $\left\|\mathbf{\mathcal{A}}\right\|_{F}^{2}$, which reveals the superiority of t-FD compared with MtFD.

Here, we only give the tensor covariance error bound of MtFD. The following theorem gives the potential relationship between tensor covariance error and projection error. After that, the projection error can be obtained immediately.
\begin{theorem}
	\label{rem: rel}
	For any tensor $\mathcal{A} \in \mathbb{R}^{n_1 \times n_2\times\cdots \times n_p}$, we have the following relationship between the projection error and the covariance error bounds:
	$$\small
	\left\|\mathbf{\mathcal{A}}-\mathbf{\mathcal{A}} * \mathbf{\mathcal{V}}_{k} * \mathbf{\mathcal{V}}_{k}^{T}\right\|_{F}^{2}\le \left\|\mathbf{\mathcal{A}}-\mathbf{\mathcal{A}}_{k}\right\|_{F}^{2}+2k\left\|\mathbf{\mathcal{A}}^{T}*\mathbf{\mathcal{A}}-\mathbf{\mathcal{B}}^{T}*\mathbf{\mathcal{B}}\right\|.$$
\end{theorem}
\begin{remark}
	By combining Theorems \ref{thm-norm2-matrization} and \ref{rem: rel}, we can derive the projection error bound for MtFD, i.e.,
		\begin{align}\label{tpe}
			\left\|\mathbf{\mathcal{A}}-\mathbf{\mathcal{A}} * \mathbf{\mathcal{V}}_{k} * \mathbf{\mathcal{V}}_{k}^{T}\right\|_{F}^{2}\le \left\|\mathbf{\mathcal{A}}-\mathbf{\mathcal{A}}_{k}\right\|_{F}^{2}+\frac{2k\rho}{\ell-k}\left\|\mathbf{\mathcal{A}}\right\|_{F}^{2}.
	\end{align}Due to $\rho$ is larger than $\ell$ and $c$ is usually small (the effect of $c$ is explained in detail in Section V), the right-hand side in (\ref{thm4111}) is much smaller than the second term of the right-hand side in (\ref{tpe}). This indicates the consistently better theoretical guarantee for t-FD compared with MtFD.
\end{remark}

\subsection{Complexity analysis}
{
	For the proposed t-FD algorithm, we only need to store the sketch  $\mathcal{B} \in \mathbb{R}^{2 \ell \times n_2 \times\cdots\times n_p}$, and the memory requirement is $O(\ell n_2 \rho)$ with  $\rho=n_3n_4\ldots n_p$, which is substantially smaller than loading an entire tensor into memory. In each update, one fast Fourier Transform and one inverse Fourier Transform are required, which costs  $O\left(\ell n_{2}\rho\log \rho\right)$ time. Computing the Singular Value Decomposition for all frontal slices takes $O\left(\ell^{2}n_{2}\rho\right)$. Since the total iterations are at most $\lceil \frac{n_{1}-\ell+1}{\ell+1} \rceil $, the computational costs are bounded by $O\left(n_{1}n_{2}\rho\left(\log \rho+\ell\right)\right)$. In experiments, we assume $\frac{n_{1}-\ell+1}{\ell+1}$ is an integer, otherwise we can append zero horizontal slices  to ensure it is an integer. Specifically, for the third-order case, the computational cost is {$O\left(n_{1}n_{2}n_3\left(\log n_3+\ell\right)\right)$ due to $\rho=n_3$.
}
\section{Experiments}
In this section, we compare the efficiency and effectiveness of our proposed t-FD with other three streaming algorithms on both synthetic and real-world tensor data. All these algorithms are implemented in MATLAB R2020a and conducted on a Dual Intel(R) Xeon(R) Gold 5120 CPU @ 2.20GHz with 256 GB memory.  In each setting, we run each algorithm 10 times and report the average results. The detailed information of such compared algorithms are listed in the following.
\begin{itemize}
	\item[1.] MtFD: As we briefly illustrated in Algorithm \ref{Matricization-tensorFD}, the direct way of implementing FD for the order-$p$ tensor data is unfolding the tensor firstly, then updating a sketch matrix $\boldsymbol{B} \in \mathbb{R}^{\ell \times n_2 \ldots n_p}$ by FD algorithm, and thus the sketch tensor $\mathcal{B}$ is obtained by folding the matrix $\boldsymbol{B}$.
	
	\item[2.] srt-SVD: We adopt the rt-SVD algorithm proposed in \cite{zhang2018randomized} to a single pass randomized algorithm. The sketch $\mathcal{B}$ is calculated from $\mathcal{Q} * \mathcal{A}$, where $\mathcal{Q} \in \mathbb{R}^{\ell \times n_1 \times n_3  \cdots n_p}$ is a random Gaussian tensor such that the first frontal slice $\mathbf{Q}^{(1)} \sim \frac{\mathcal{N}(0,1)}{\sqrt{\ell}}$, and other slices are all zeros. The t-product between $\mathcal{A}$ and $\mathcal{Q}$ can be easily derived in a streaming way.
	\item[3.] Norm Sampling (NormSamp for short) \cite{drineas2006fast}, \cite{holodnak2015randomized}: We adopt a well-known random sampling method to deal with the tensor data. Precisely, the sketch is formulated by sampling $\ell$ horizontal slices independently from the $n_1$ horizontal slices of $\mathcal{A}$ and rescaling. The $i$-th slice is chosen with probability $p_i = \|\mathcal{A}_i\|_F^2/ \|\mathcal{A}\|_F^2$ and rescaled to $\mathcal{A}_i / \sqrt{\ell p_i}$. Since the value of $\|\mathcal{A}\|_F$ is unknown, we implement this method through two passes over the data by  calculating  the norm and sampling separately.
\end{itemize}

We consider three measures to compare the performance: 
\begin{itemize}
	\item The projection error: $\frac{\left\|\mathbf{\mathcal{A}}-\mathbf{\mathcal{A}} * \mathbf{\mathcal{V}}_{k} * \mathbf{\mathcal{V}}_{k}^{T}\right\|_{F}^{2}}{ \left\|\mathbf{\mathcal{A}}-\mathbf{\mathcal{A}}_{k}\right\|_{F}^{2}}$
	\item The covariance error: $\frac{\left\|\mathbf{\mathcal{A}}^{T}*\mathbf{\mathcal{A}}-\mathbf{\mathcal{B}}^{T}*\mathbf{\mathcal{B}}\right\|}{\left\|\mathbf{\mathcal{A}}-\mathbf{\mathcal{A}}_{k}\right\|_{F}^{2}}$
	\item Running time in seconds
\end{itemize}
\subsection{Synthetic data examples}
In this set of experiments, we consider to use third and fourth order synthetic tensor data for illustration. Inspired by the data generation process for testing FD in \cite{liberty2013simple}, we construct the third-order tensor data using $\mathcal{A} = \mathcal{S} *  \hat{\mathcal{D}} * \mathcal{U} + \mathcal{N}/\eta$, where $\mathcal{S} \in \mathbb{R}^{n_1 \times k \times n_3}$  is generated by $\mathcal{ S }_{ijk} \sim \mathcal{N}(0,1)$ i.i.d. and $\mathcal{U} \in \mathbb{R}^{n_2 \times k \times n_3}$ is a partially orthogonal tensor. $ \hat{\mathcal{D}} $ is an f-diagonal tensor, in which the elements represent the tensor singular values. Here we consider the singular values in the Fourier domain having different decaying spectrums  including linearly, polynomially and exponentially decaying spectrums. For each slice, we randomly choose one of the three types. $\mathcal{N} \in \mathbb{R}^{n_1 \times n_2 \times n_3}$ represents the Gaussian noise and $\eta$ decides the noise level.  We also adopt this procedure to generate fourth-order tensors. 

In this experiment, we mainly measure the approximation power with different rank settings. We consider the cases of $\mathcal{A} \in \mathbb{R}^{10000 \times 1000 \times 10}$ and $\mathcal{A} \in \mathbb{R}^{10000 \times 1000 \times 10 \times 3}$, the noise level $\eta = 10$ and the true rank $k \in \{10, 20, 50\}$.  For each setting, we generate three tensors and report the average performance.

\begin{table*}[ht!]
	\centering
	
	\begin{adjustbox}{center}
		\begin{tabular}{|c|c|c|c|c|c|}
			\hline
			\rotatebox{90}{\begin{tiny}$\mathsf{Projection\ error}$\end{tiny}} & 
			\raisebox{-1mm}{\includegraphics[scale = 0.22]{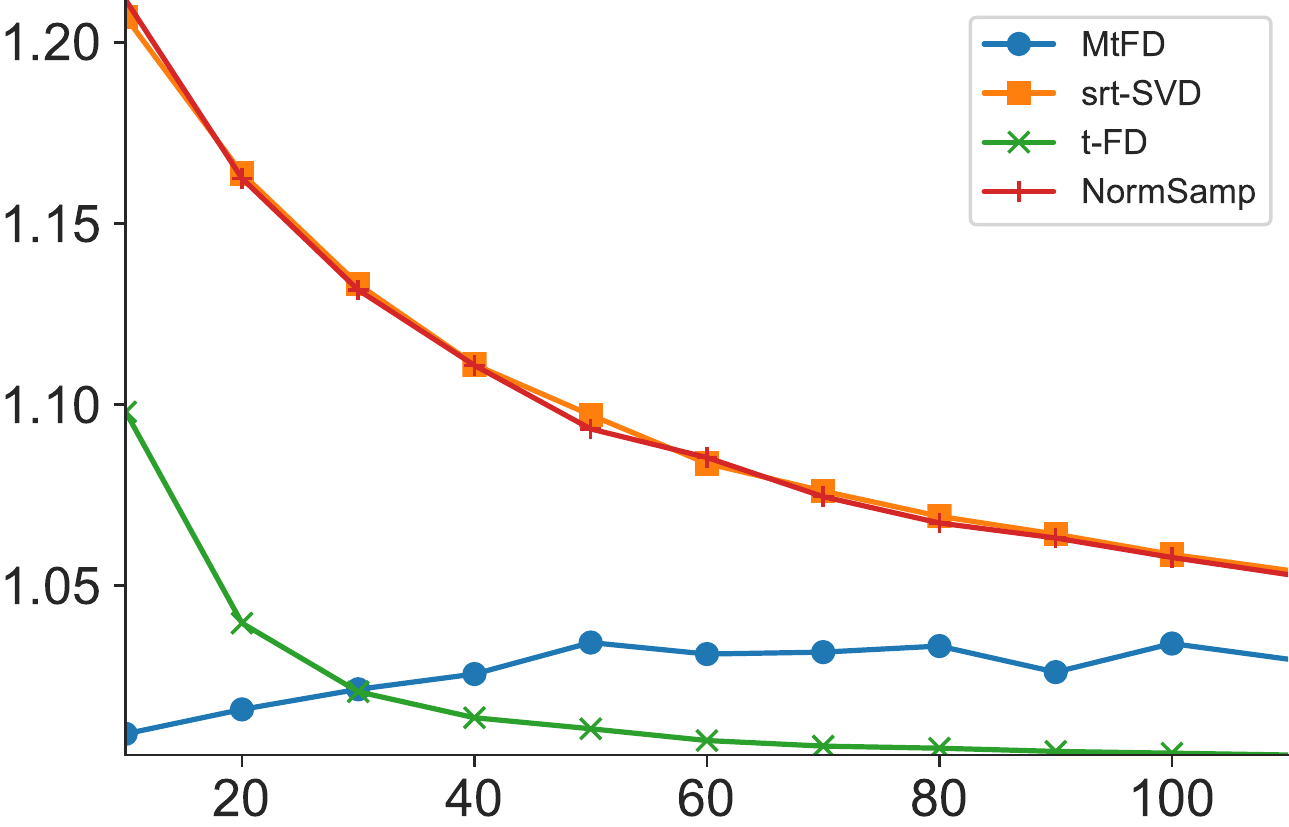}}
			&  \raisebox{-1mm}{\includegraphics[scale = 0.22]{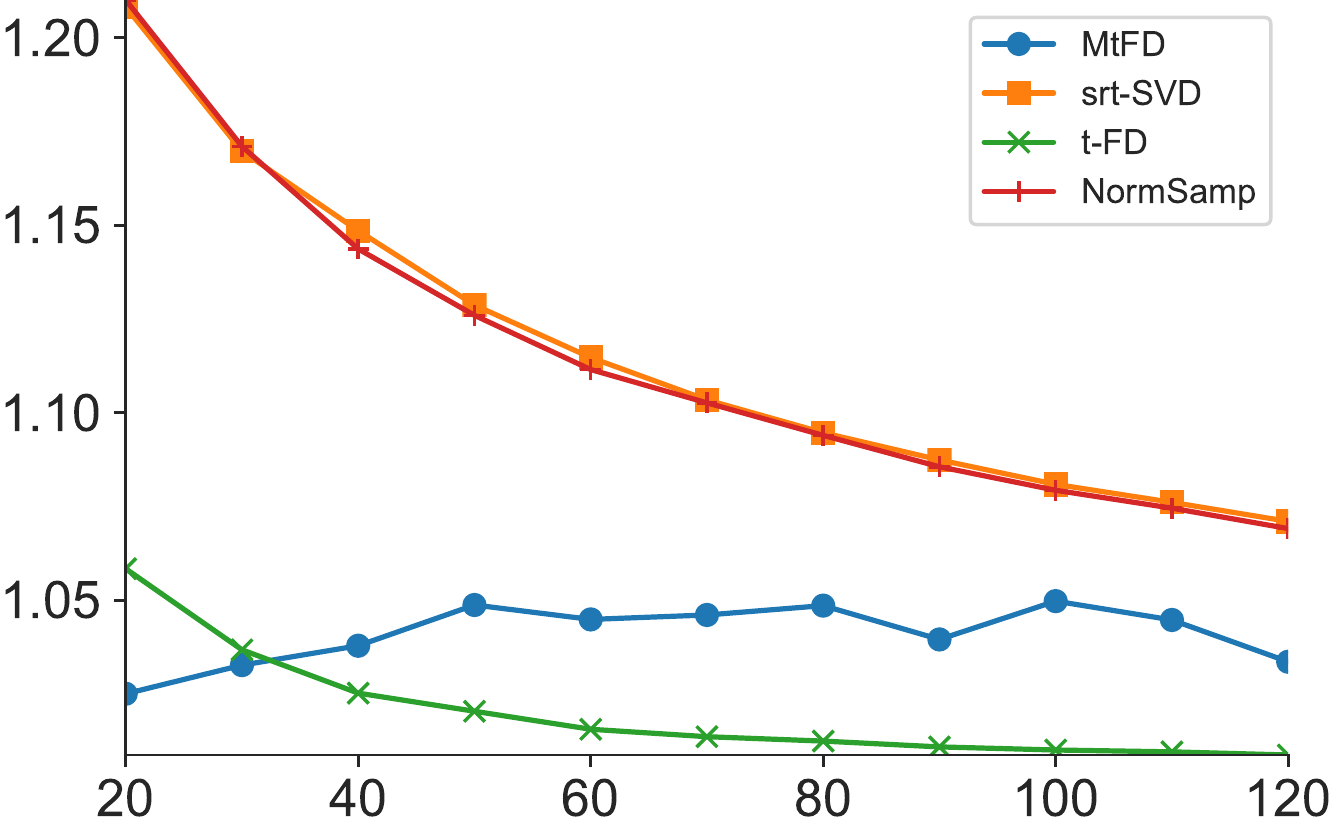}} &  \raisebox{-1mm}{\includegraphics[scale = 0.23]{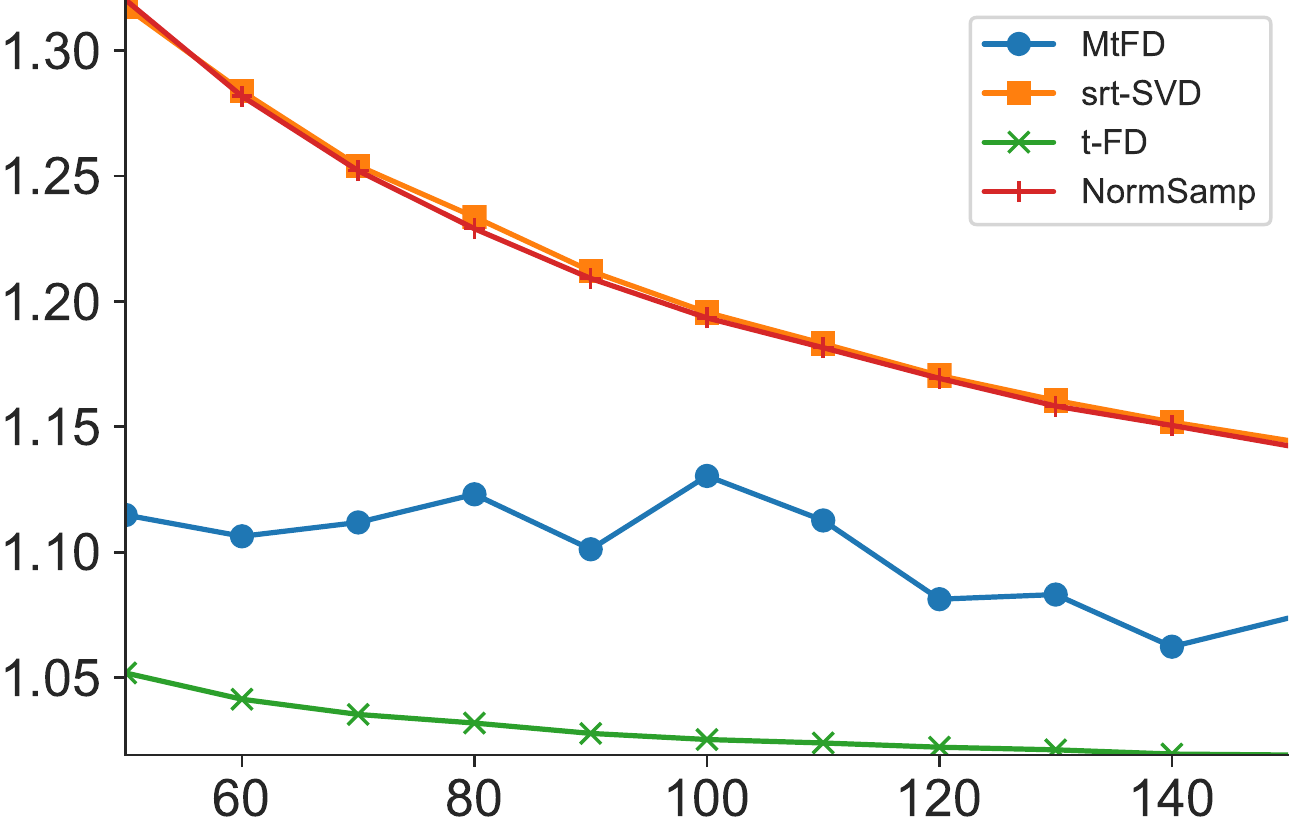}} &  \raisebox{-1mm}{\includegraphics[scale = 0.22]{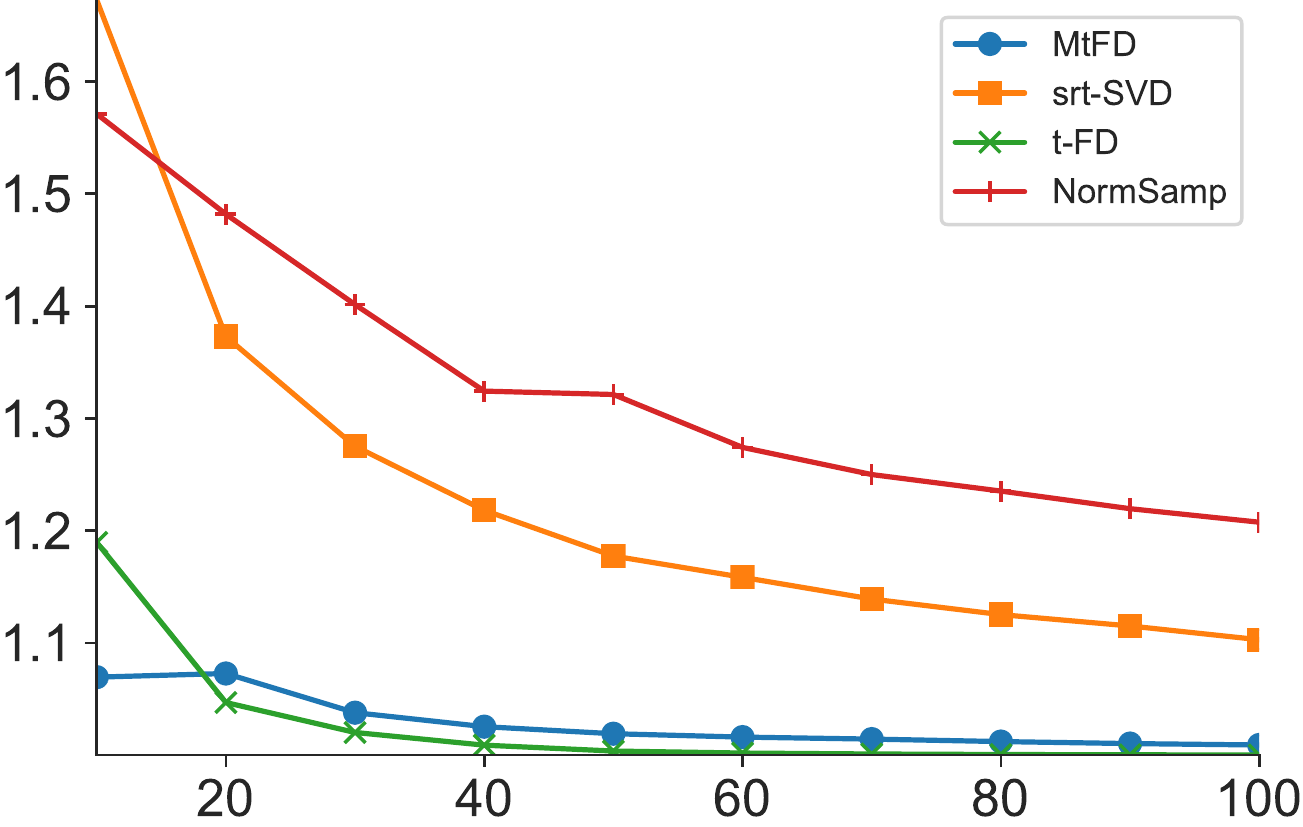}} & \raisebox{-1mm}{\includegraphics[scale = 0.23]{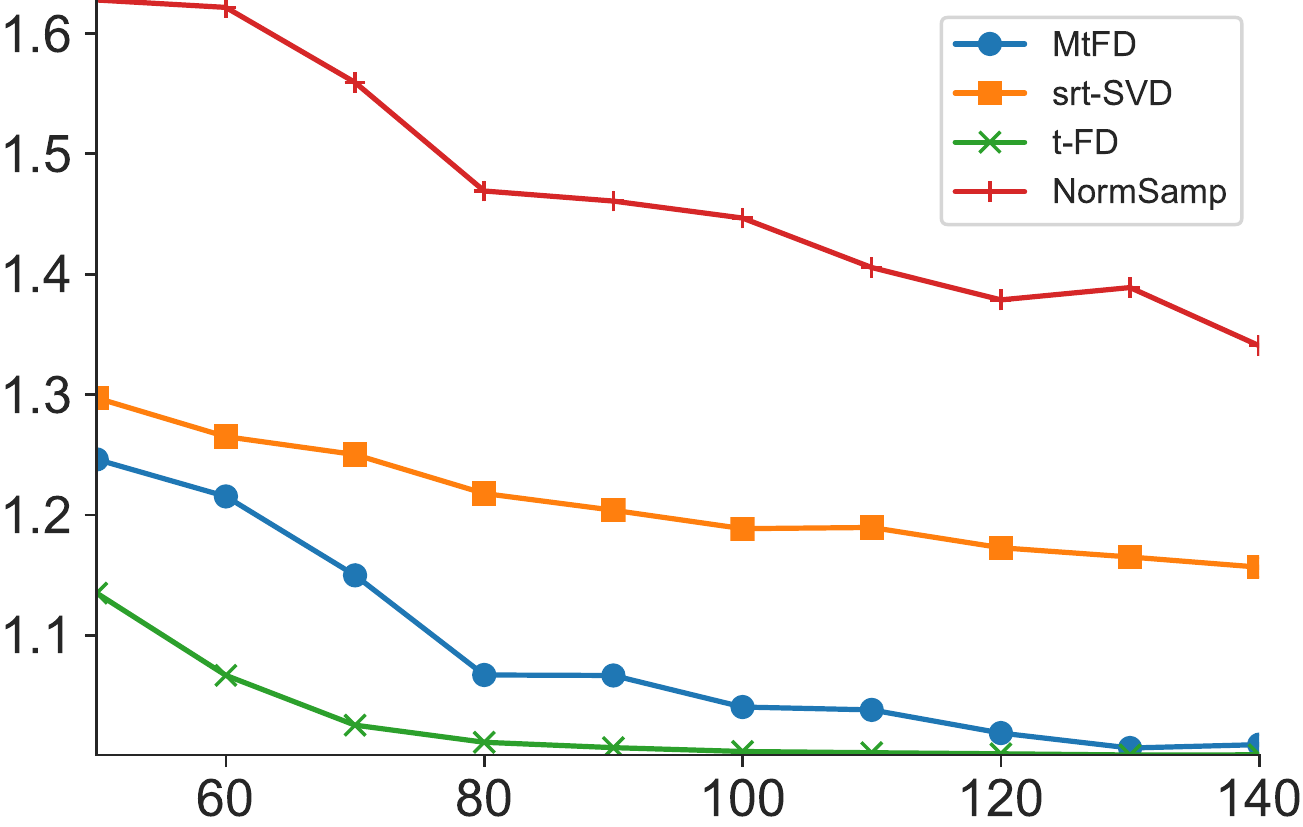}}\\
			\hline
			\rotatebox{90}{\begin{tiny}$\mathsf{Covariance\ error}$\end{tiny}} & 
			\raisebox{-1mm}{\includegraphics[scale = 0.22]{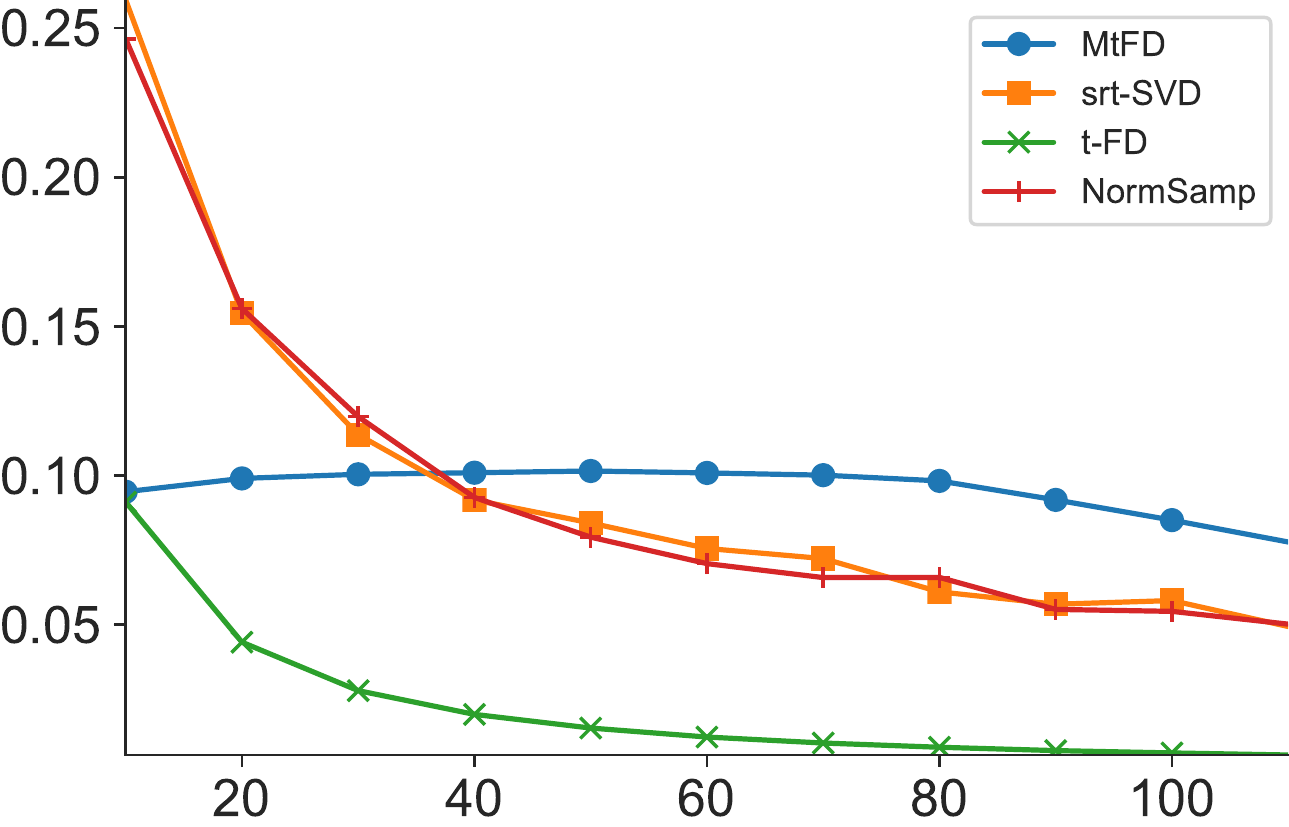}} &  \raisebox{-1mm}{\includegraphics[scale = 0.23]{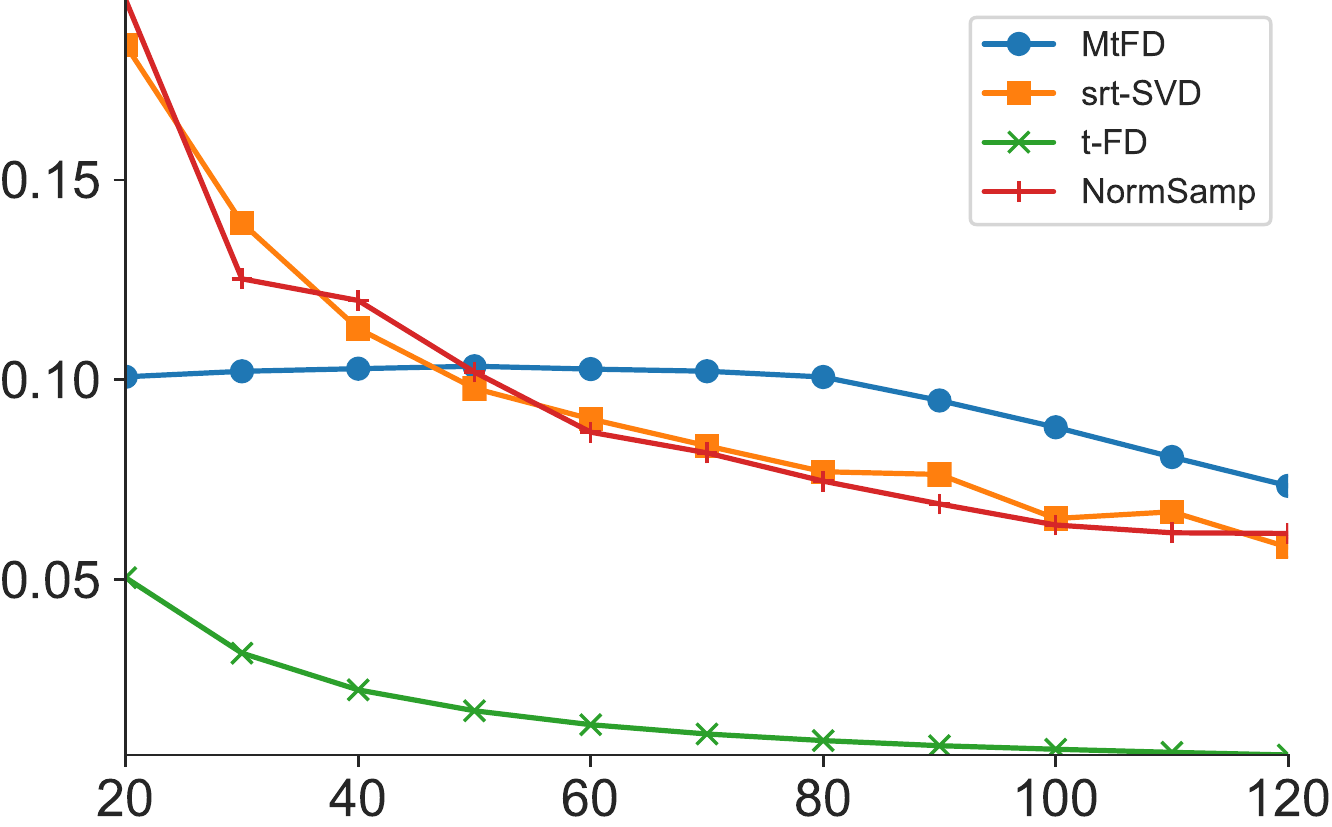}} &  \raisebox{-1mm}{\includegraphics[scale = 0.22]{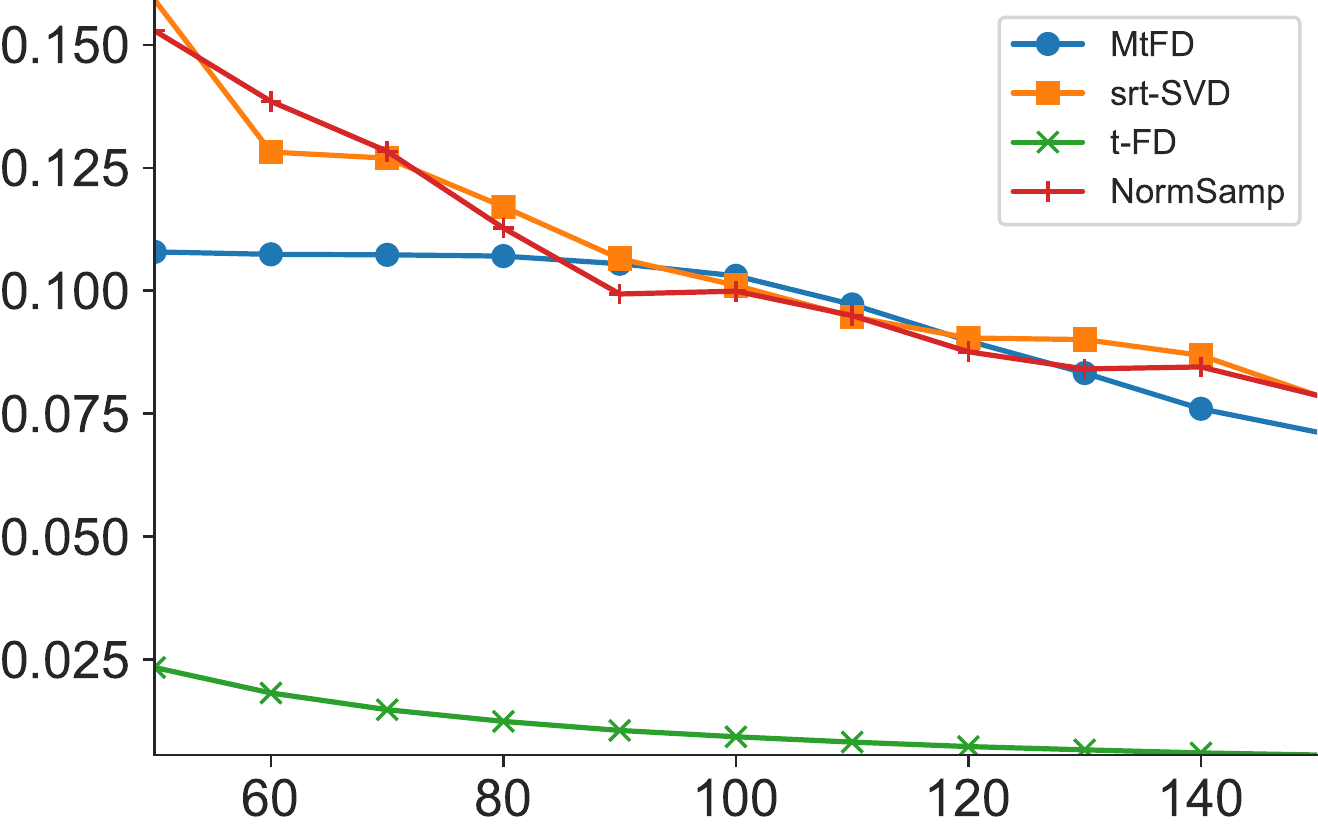}} &  \raisebox{-1mm}{\includegraphics[scale = 0.23]{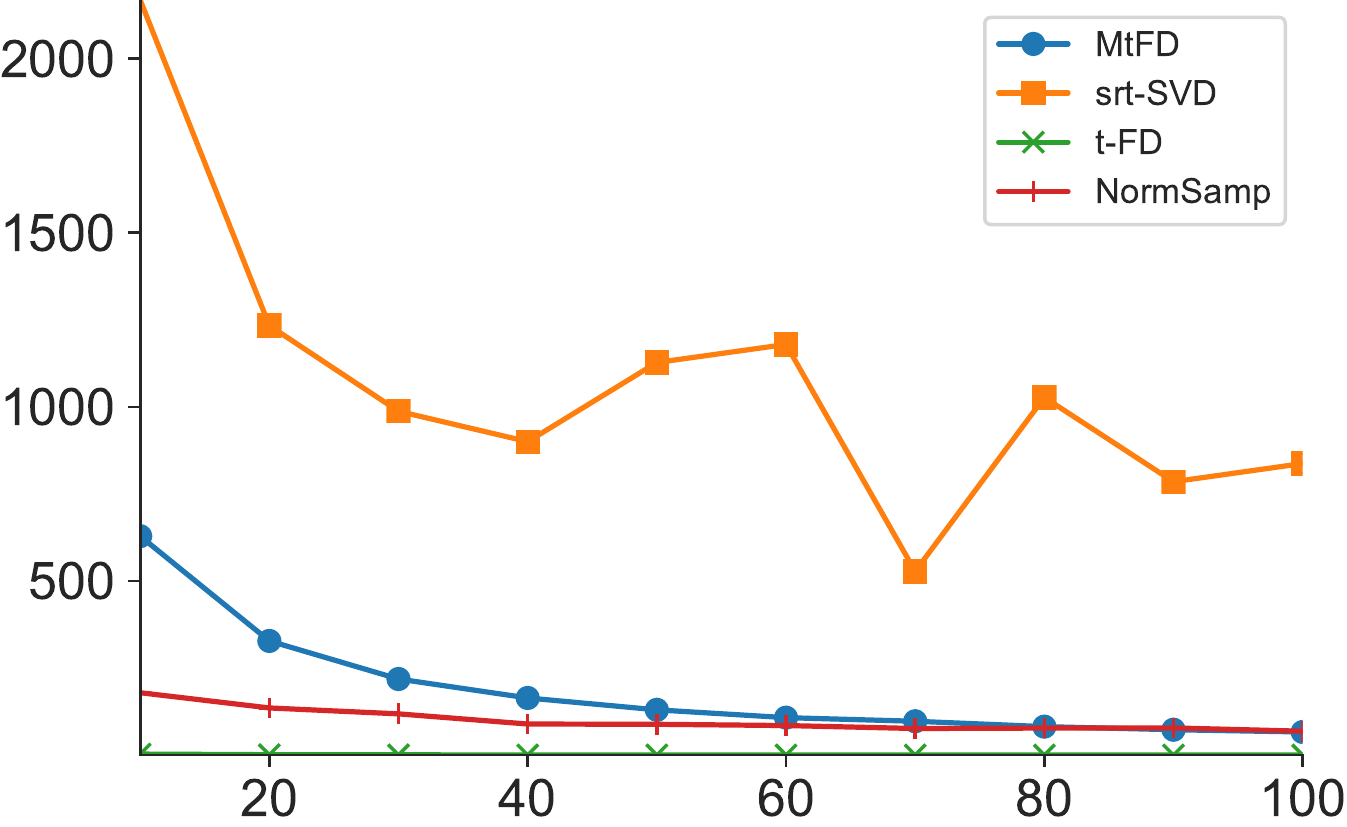}} &  \raisebox{-1mm}{\includegraphics[scale = 0.23]{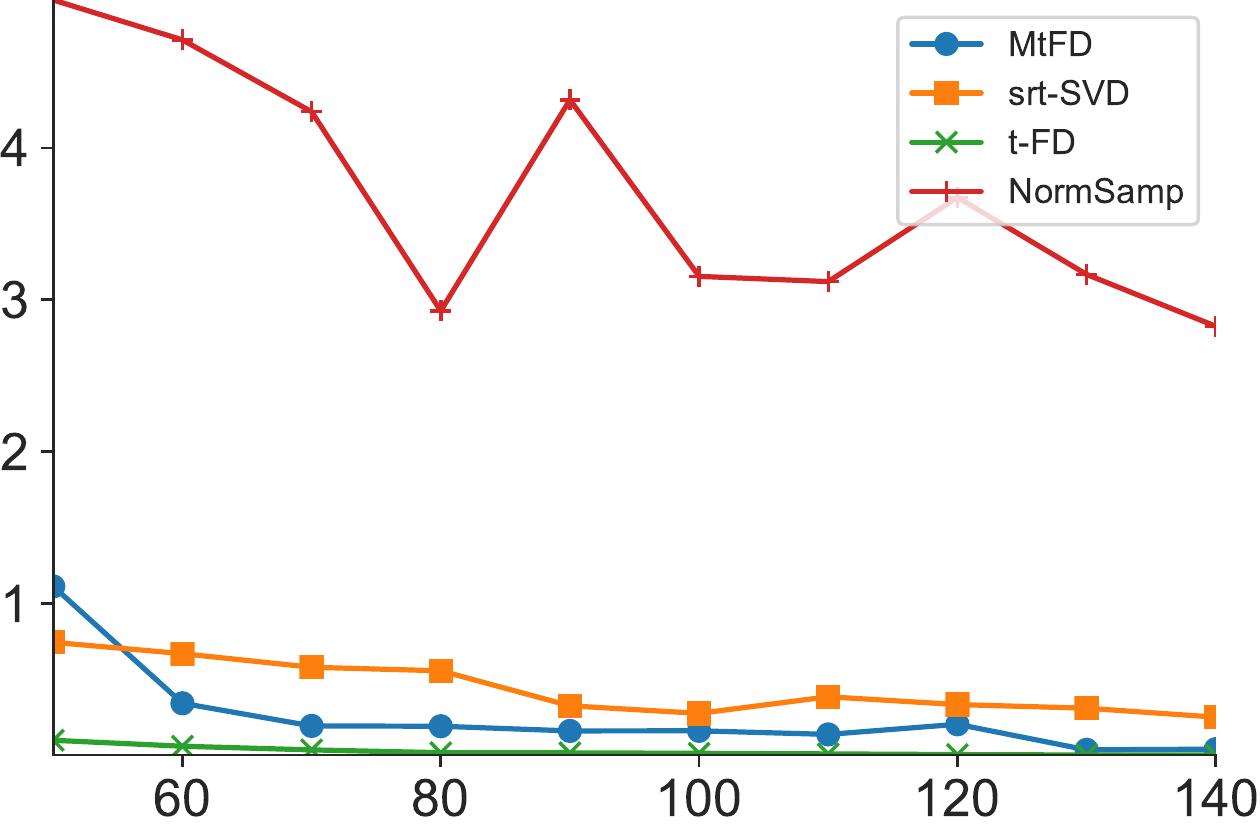}}\\
			\hline
			\rotatebox{90}{\begin{tiny}$\mathsf{Running\ time}$\end{tiny}} & 
			\raisebox{-1mm}{\includegraphics[scale = 0.22]{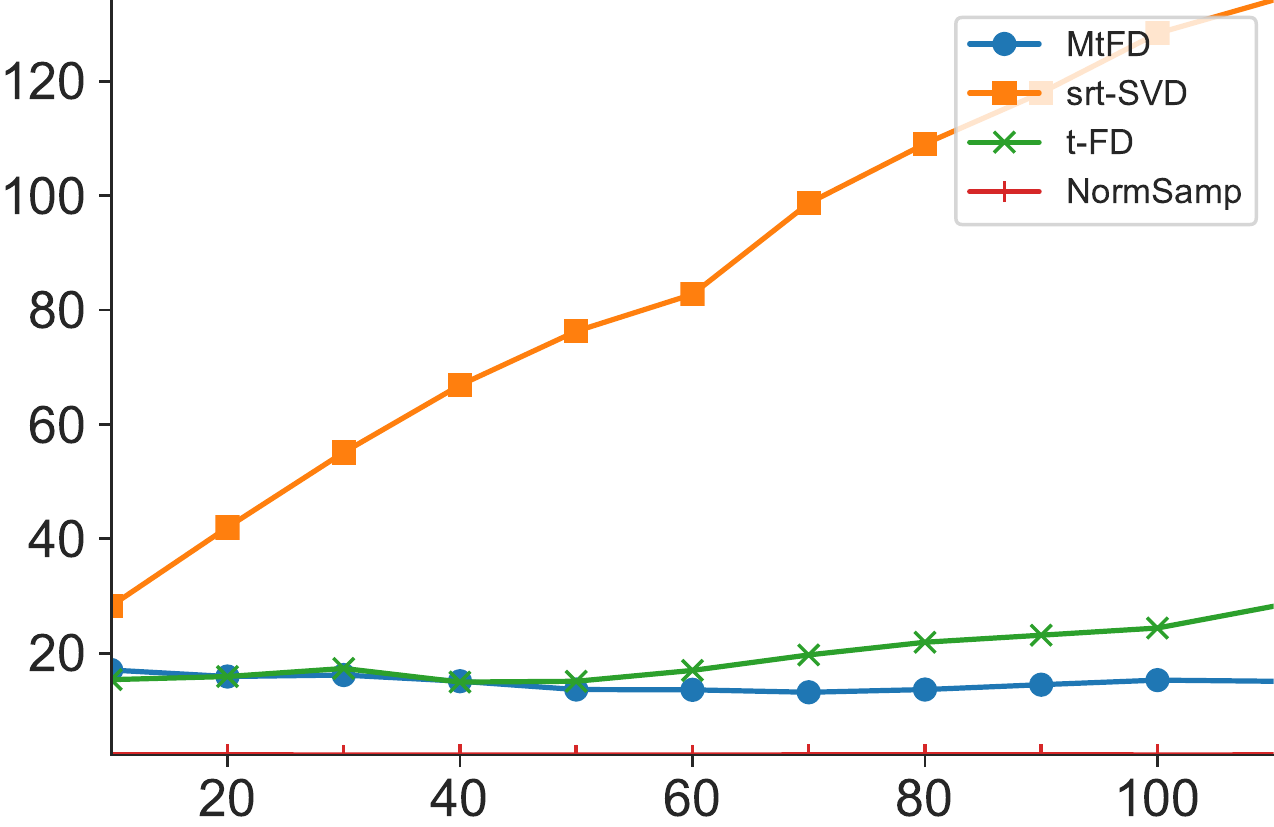}} &  \raisebox{-1mm}{\includegraphics[scale = 0.23]{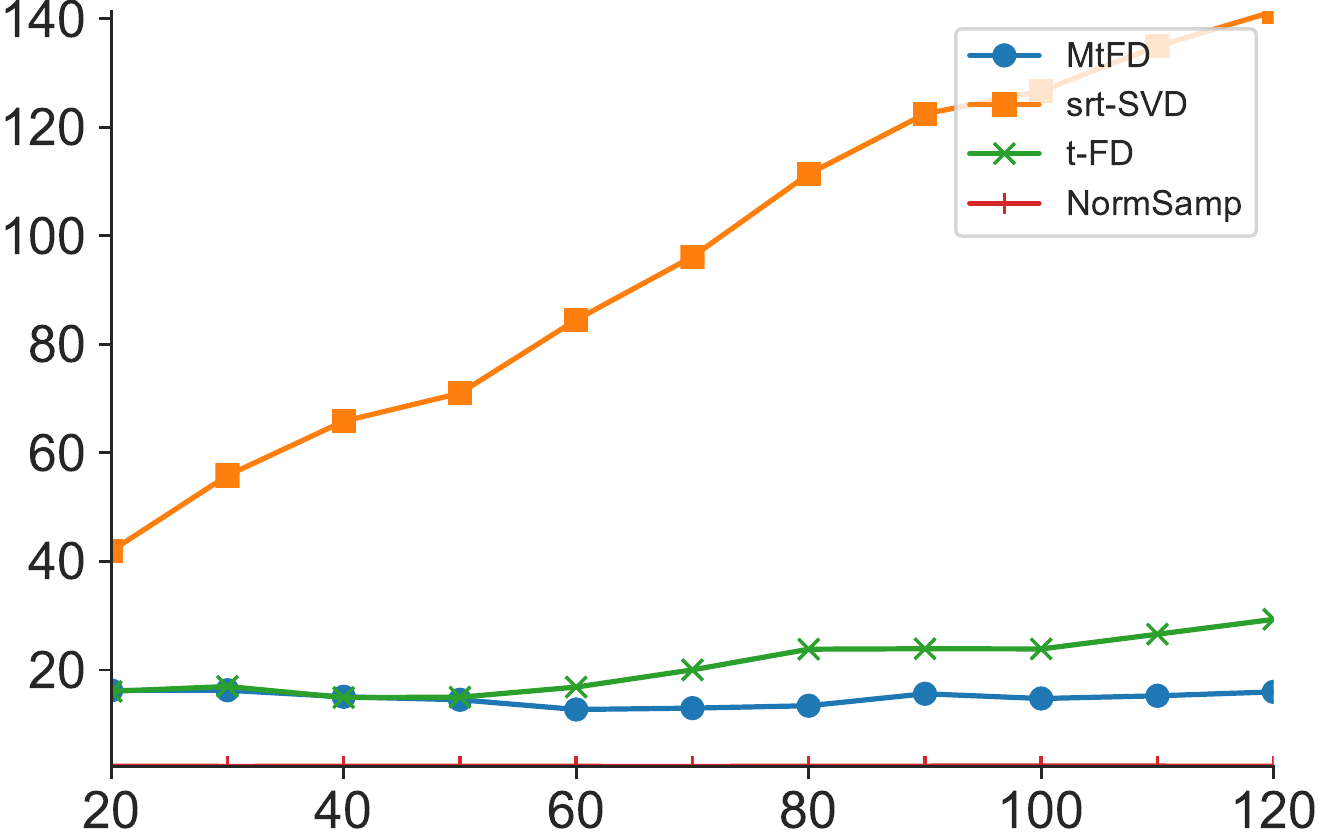}} &  \raisebox{-1mm}{\includegraphics[scale = 0.22]{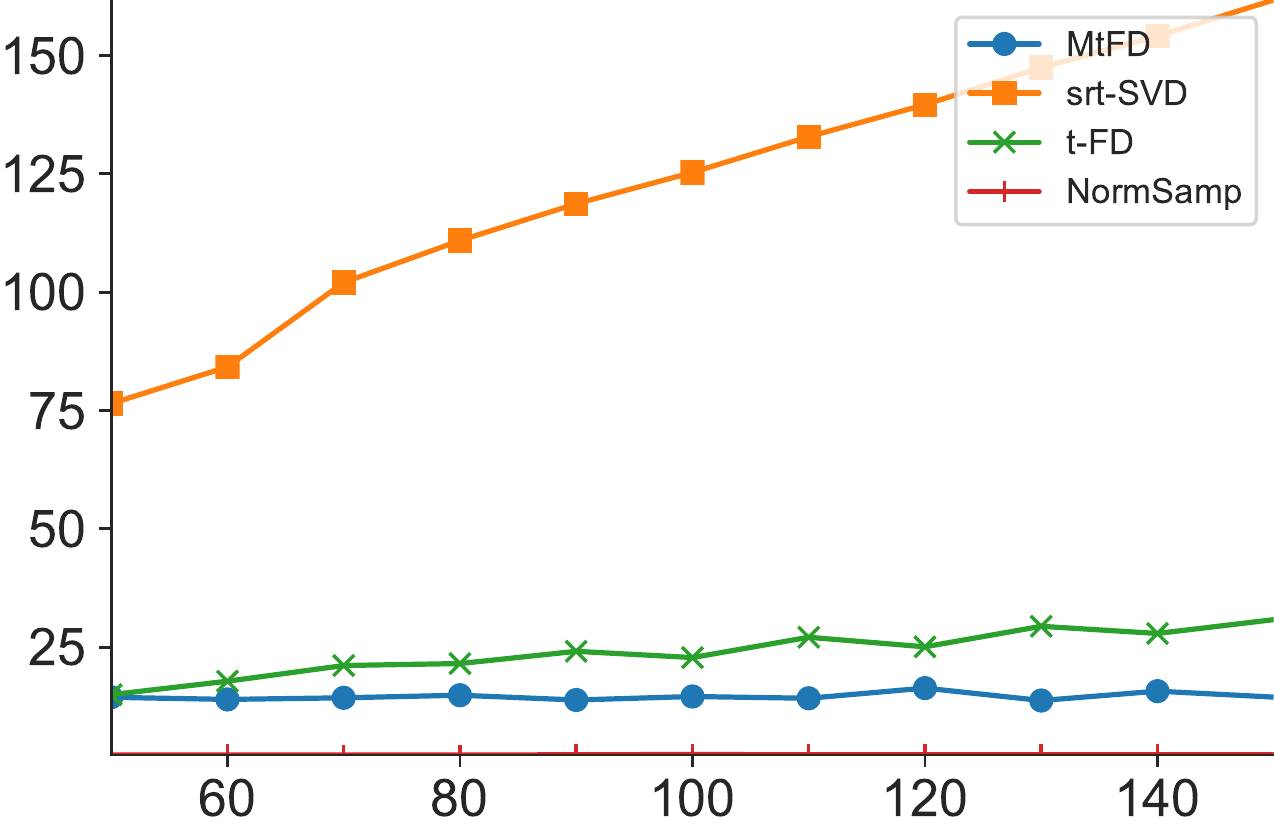}} &  \raisebox{-1mm}{\includegraphics[scale = 0.23]{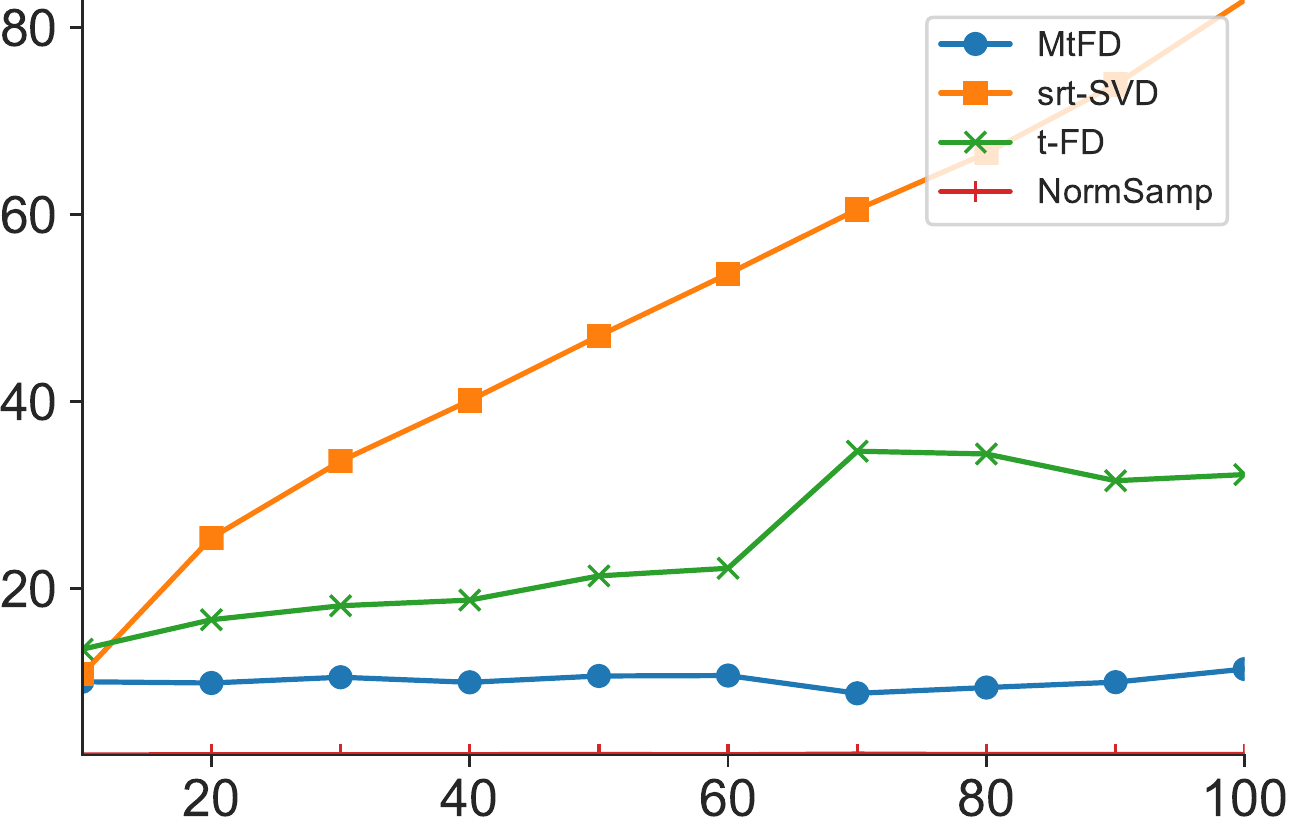}} &  \raisebox{-1mm}{\includegraphics[scale = 0.22]{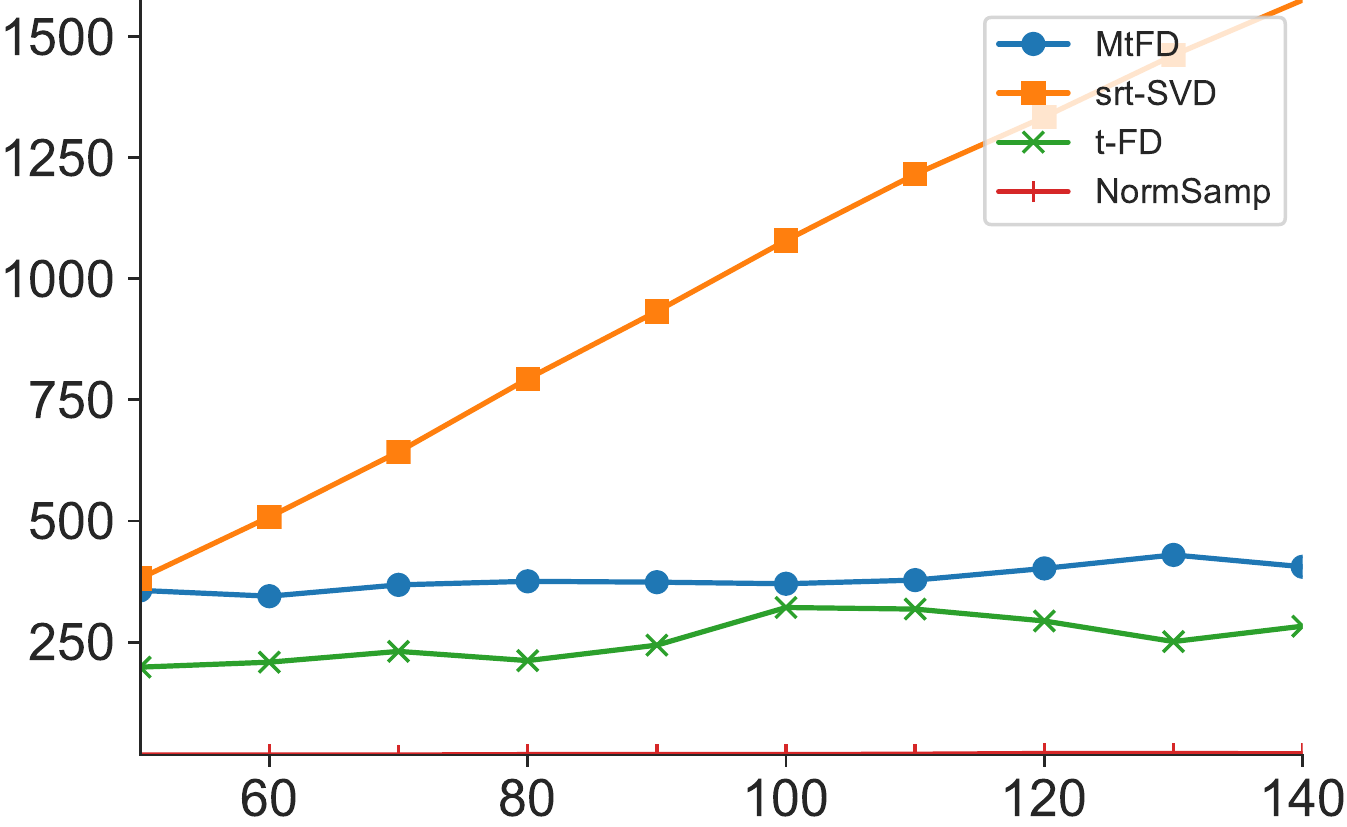}}\\
			\hline
			\rotatebox{90}{\begin{tiny}$\mathsf{The\  value \ of \ c}$\end{tiny}} & 
			\raisebox{-2mm}{\includegraphics[scale = 0.21]{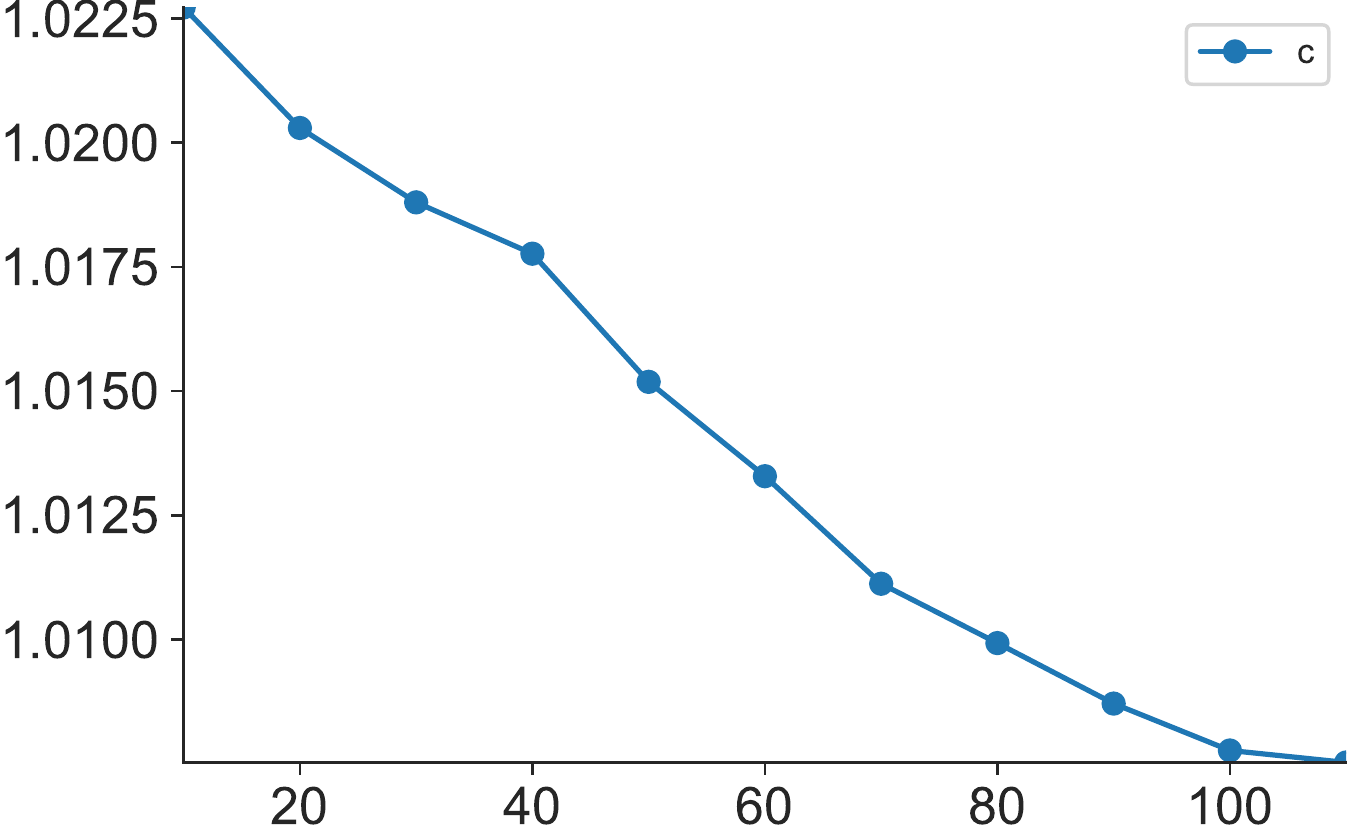}} &  \raisebox{-2mm}{\includegraphics[scale = 0.21]{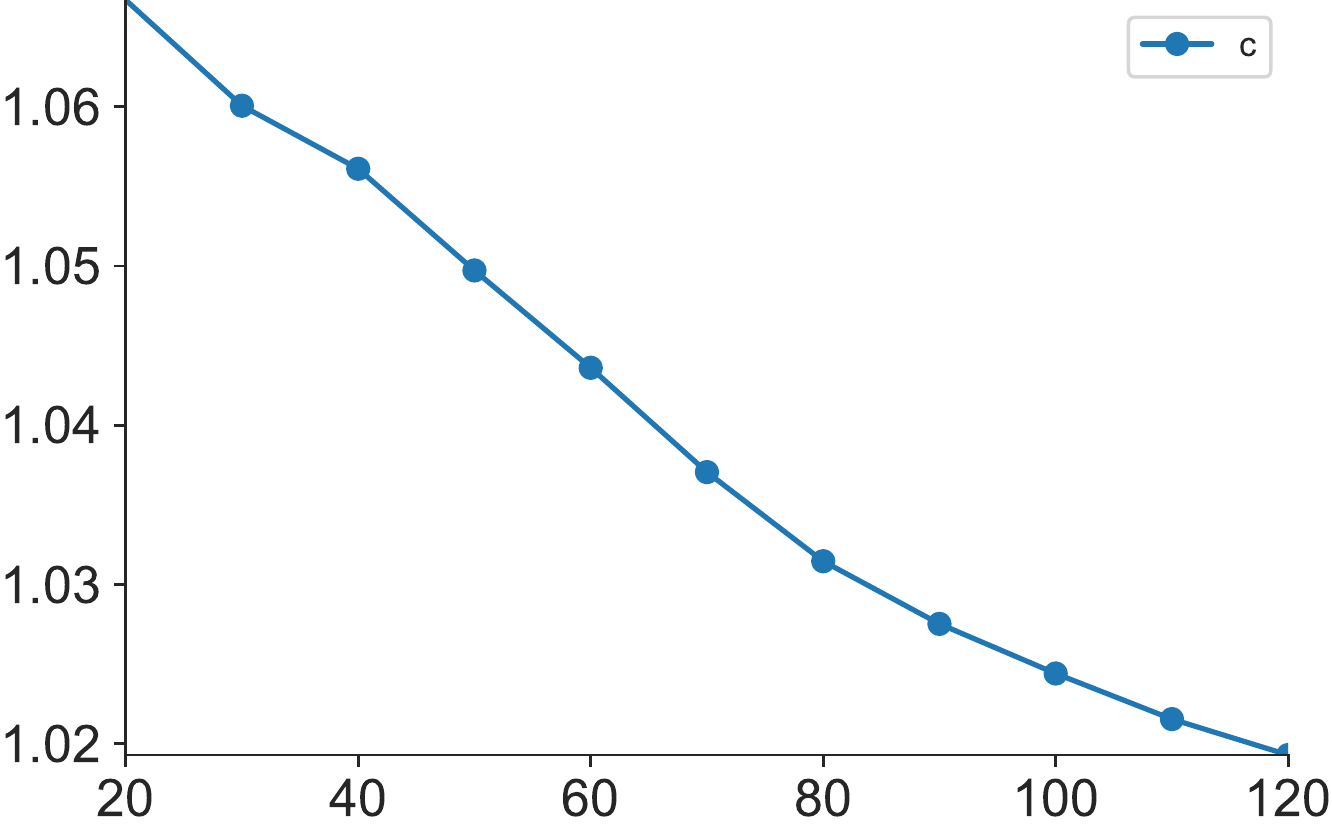}} &  \raisebox{-2mm}{\includegraphics[scale = 0.21]{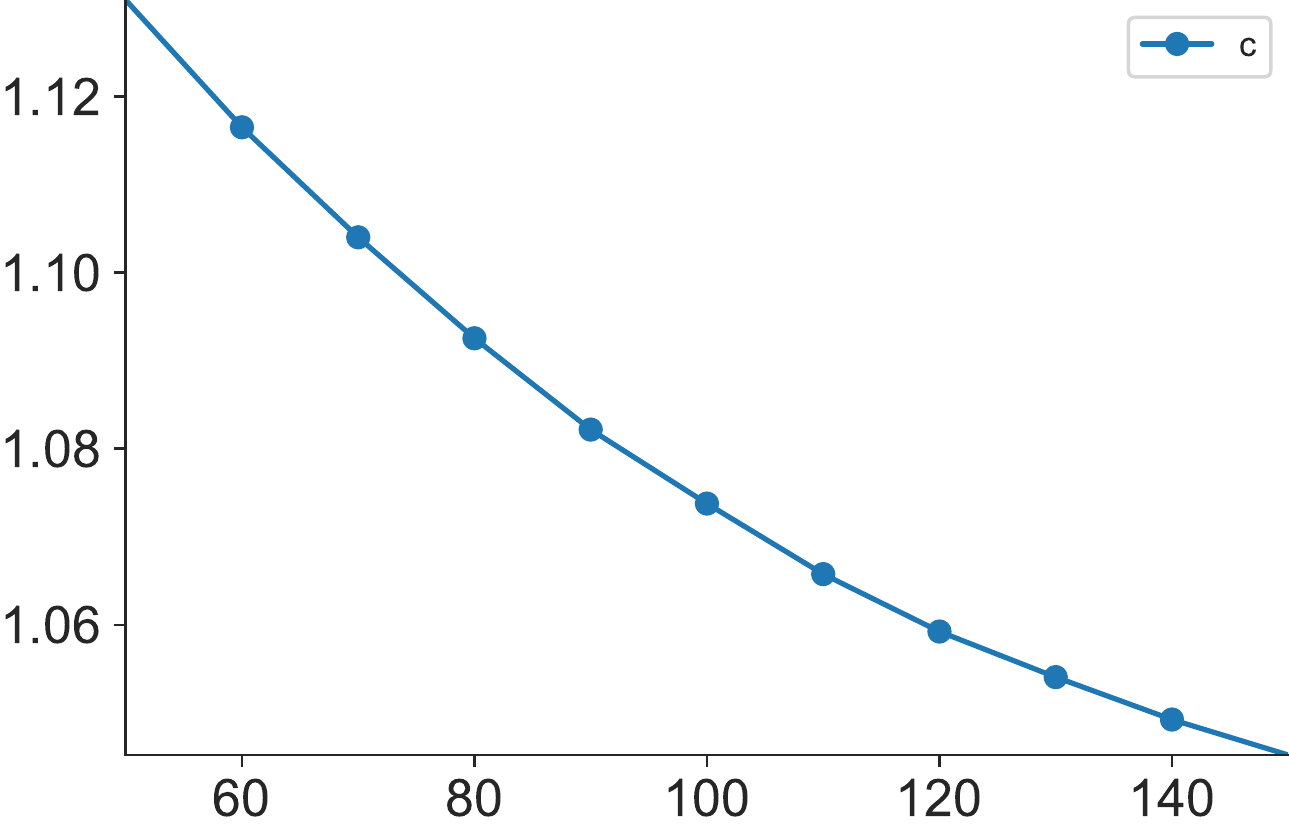}} &  \raisebox{-2mm}{\includegraphics[scale = 0.21]{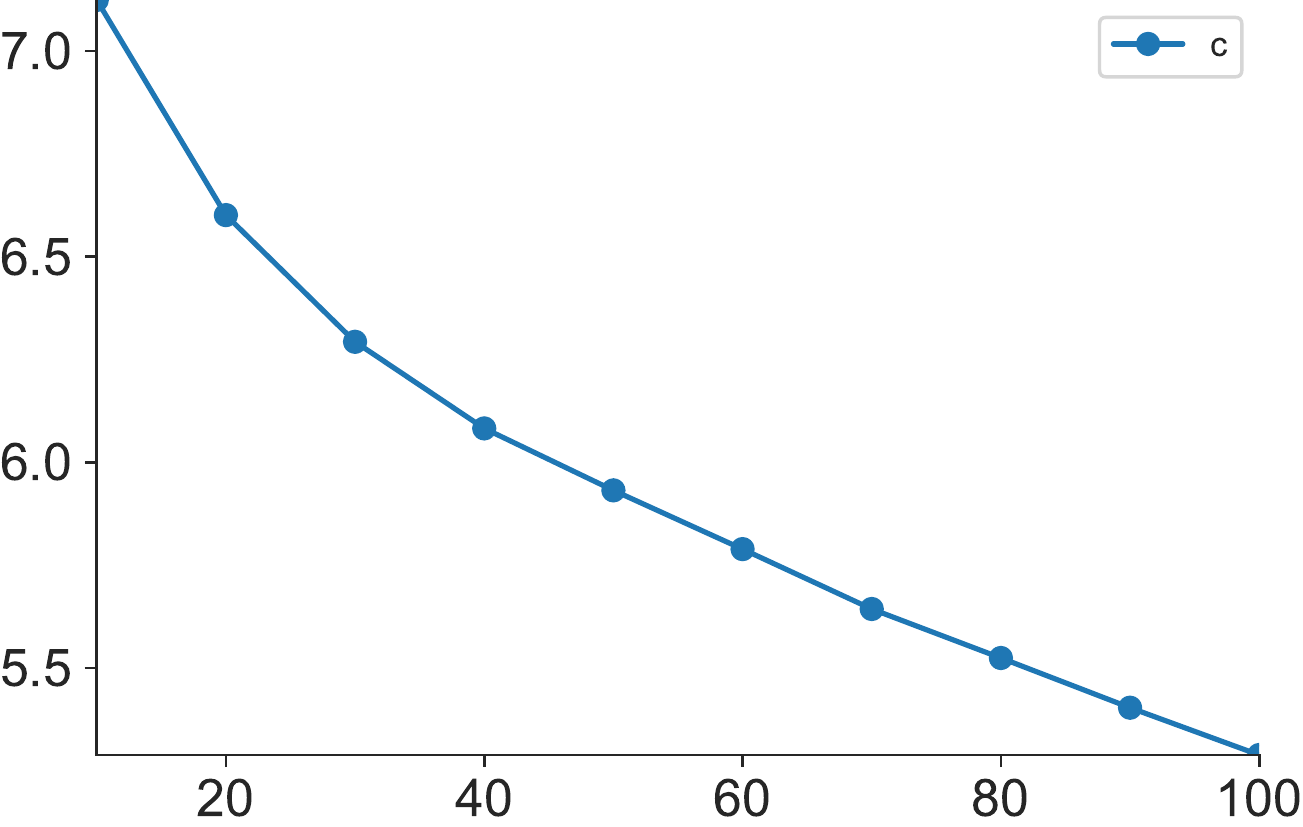}} &  \raisebox{-2mm}{\includegraphics[scale = 0.21]{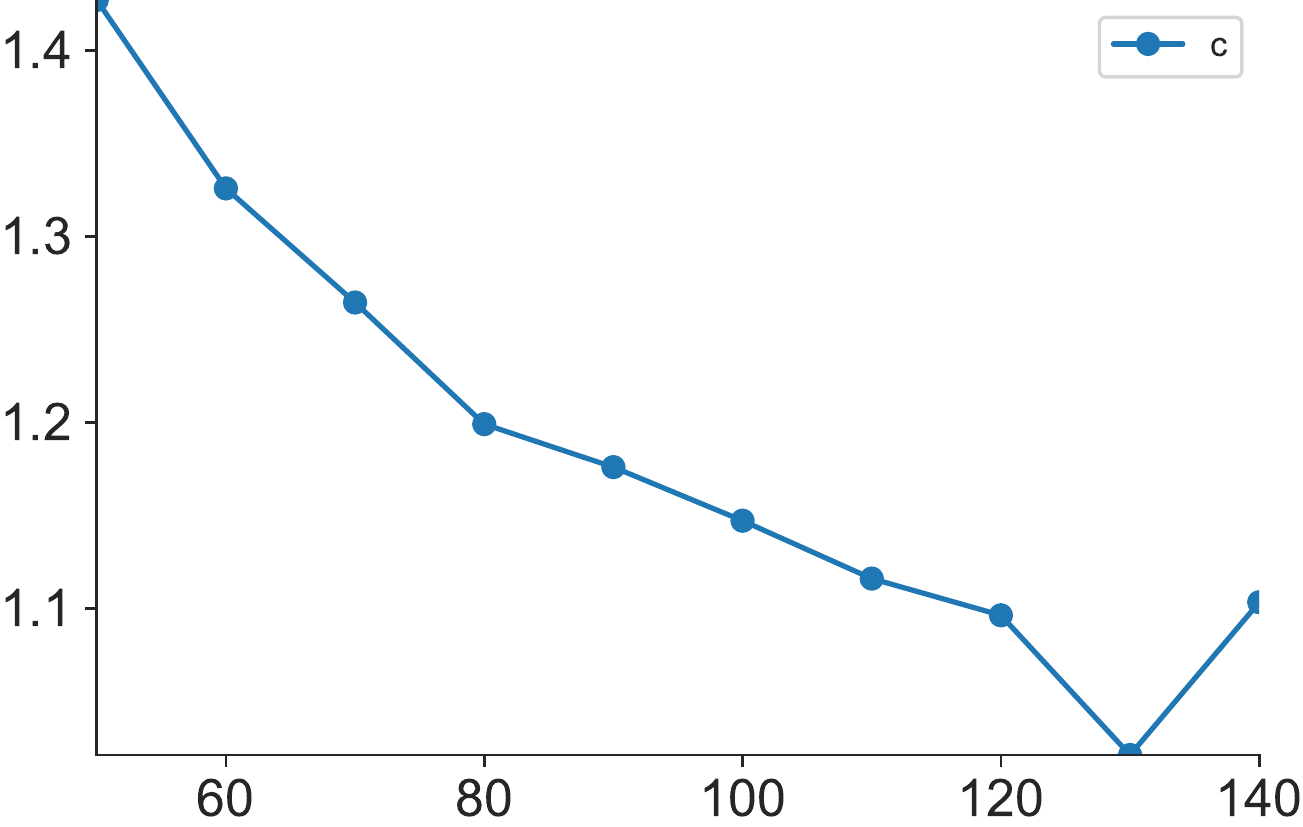}}\\
			\hline
			& $\mathsf{k = 10}$ & $\mathsf{k = 20}$ & $\mathsf{k = 50}$ & $\mathsf{highway}$ & $\mathsf{Uber}$\\
			\hline
		\end{tabular}
	\end{adjustbox}
	\vspace{10pt}
	\caption{Experimental results for third-order synthetic and real datasets} \label{tab:stimuli}
\end{table*}

\begin{table*}[ht!]
	\centering
	\begin{adjustbox}{center}
		\begin{tabular}{|c|c|c|c|c|c|}
			\hline
			\rotatebox{90}{\begin{tiny}$\mathsf{Projection\ error}$\end{tiny}} & 
			\raisebox{-1mm}{\includegraphics[scale = 0.22]{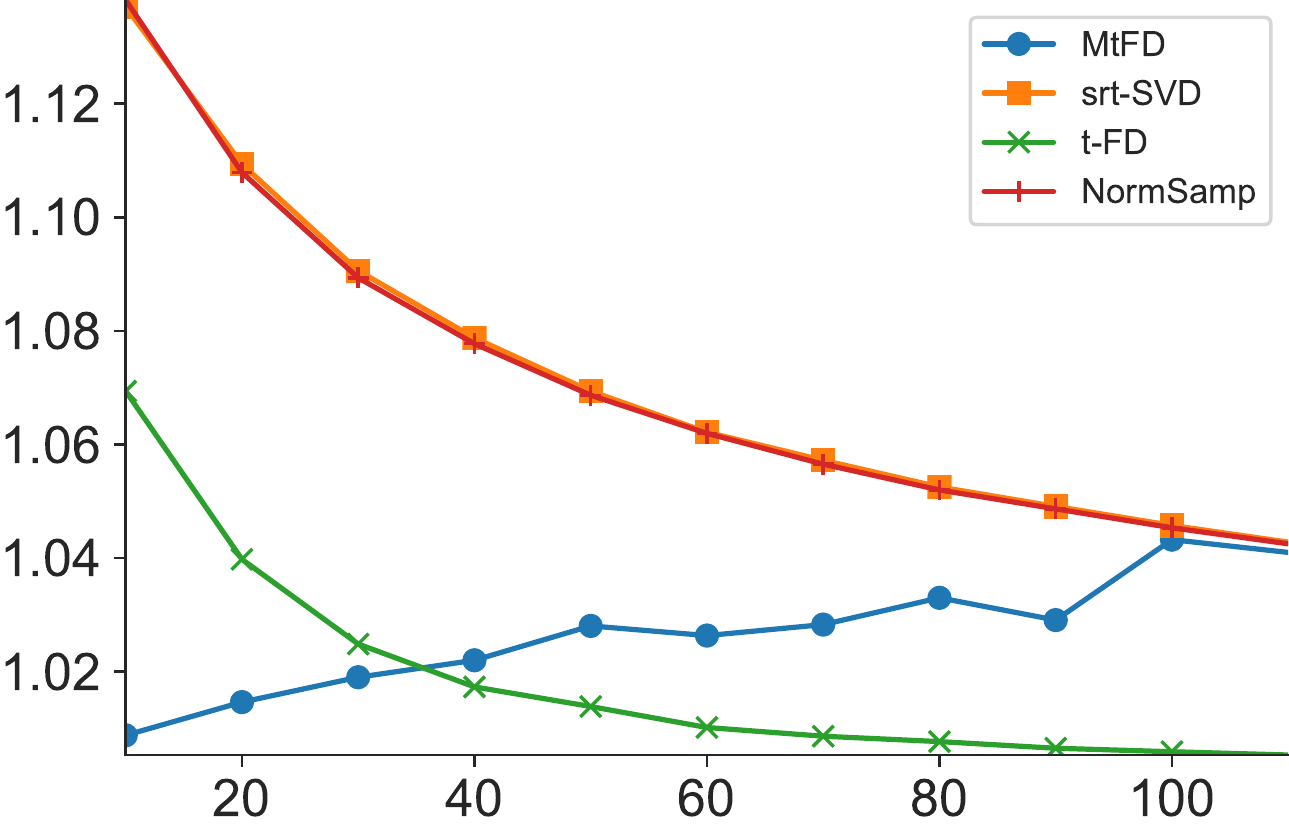}}
			&  \raisebox{-1mm}{\includegraphics[scale = 0.22]{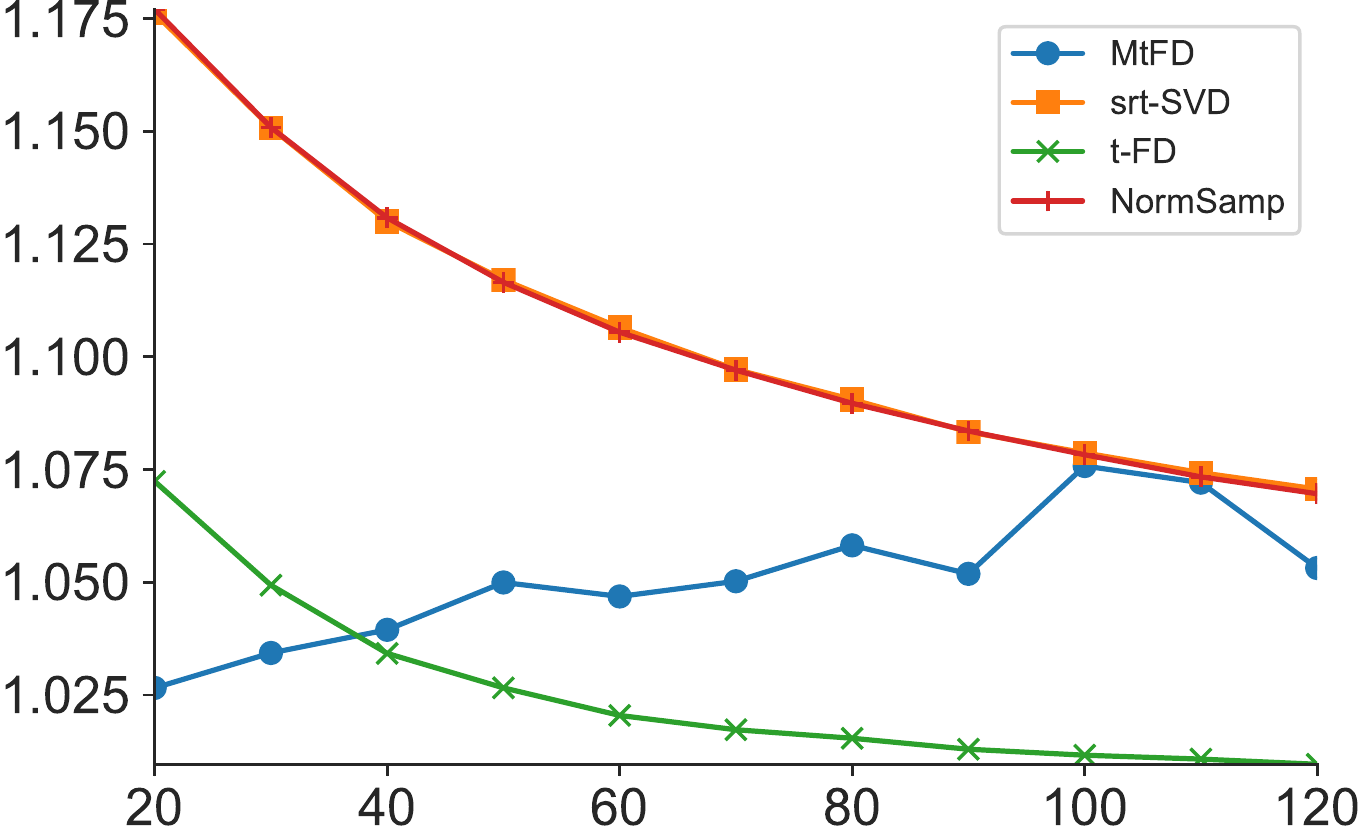}} &  \raisebox{-1mm}{\includegraphics[scale = 0.23]{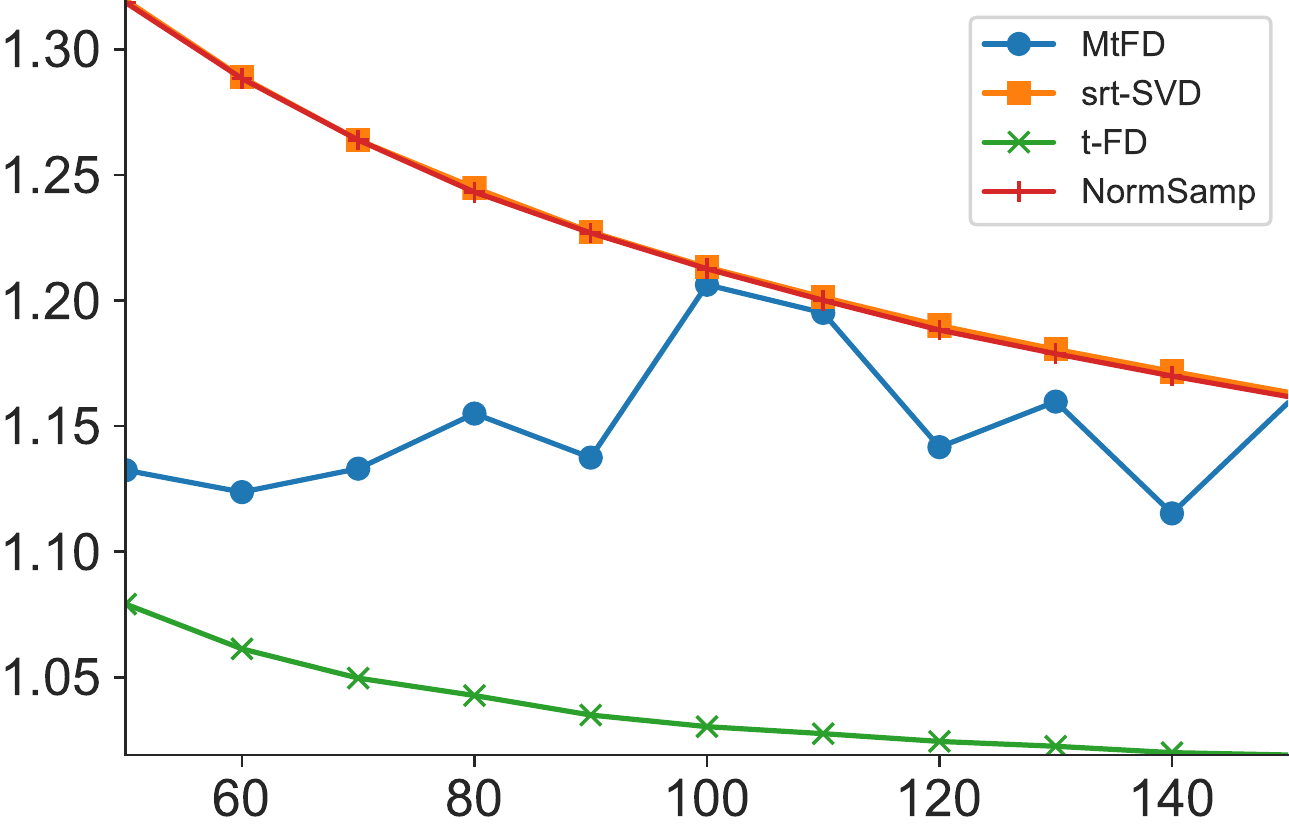}} &  \raisebox{-1mm}{\includegraphics[scale = 0.22]{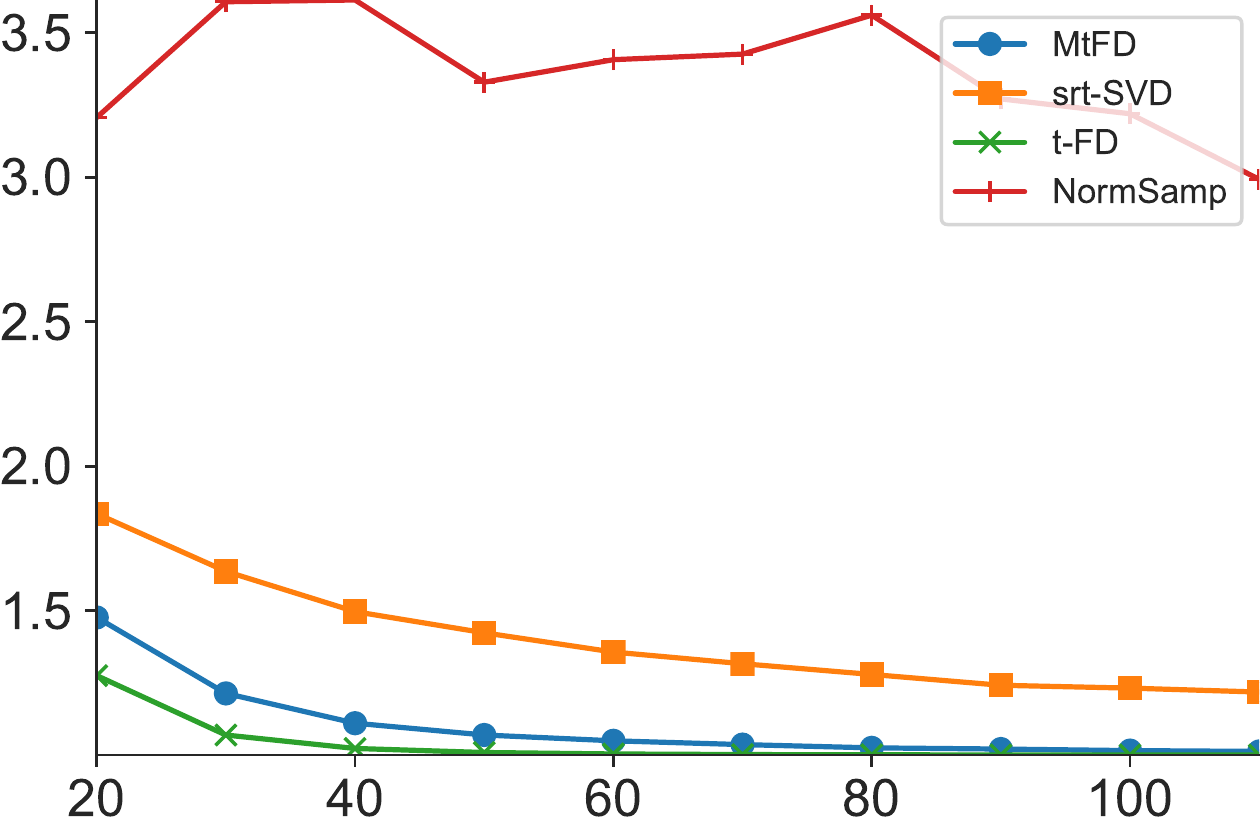}} & \raisebox{-1mm}{\includegraphics[scale = 0.23]{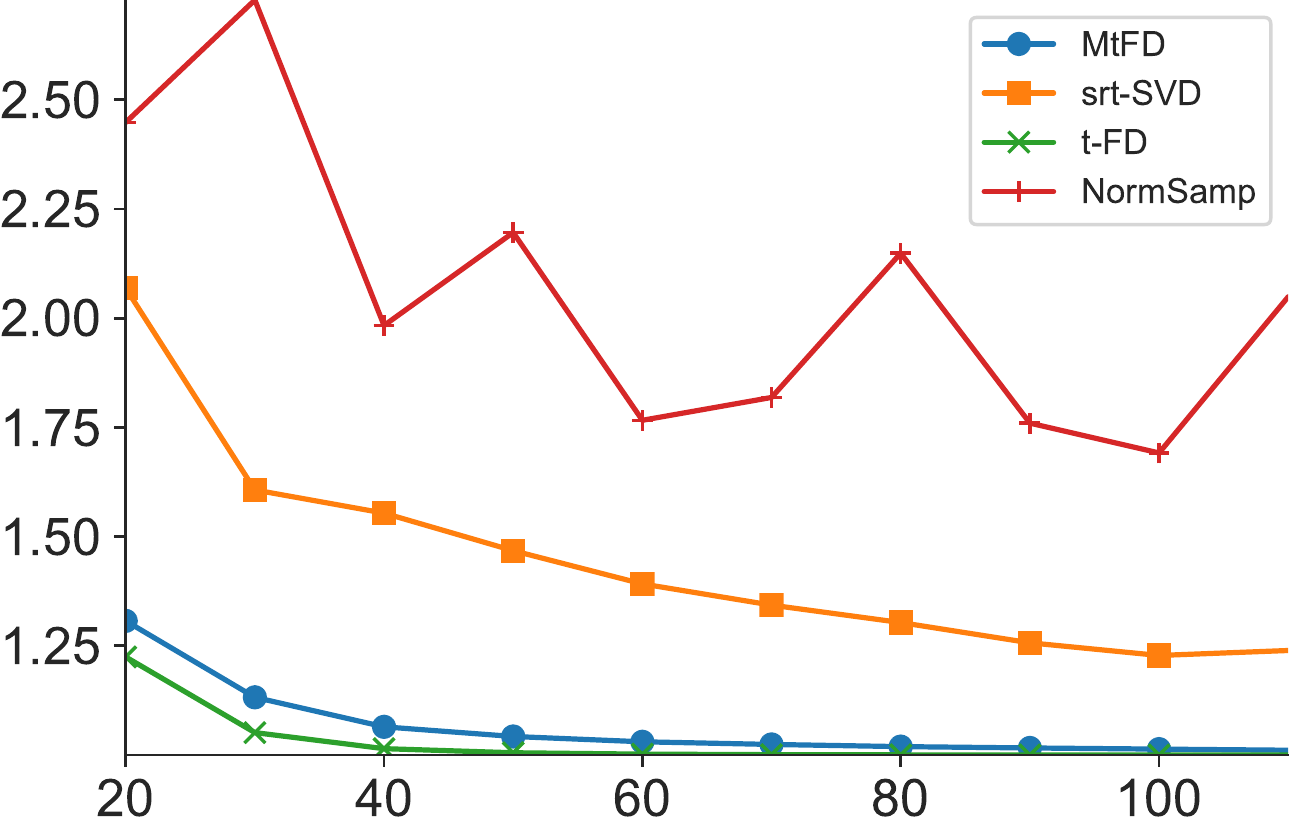}}\\
			\hline
			\rotatebox{90}{\begin{tiny}$\mathsf{Covariance\ error}$\end{tiny}} & 
			\raisebox{-1mm}{\includegraphics[scale = 0.22]{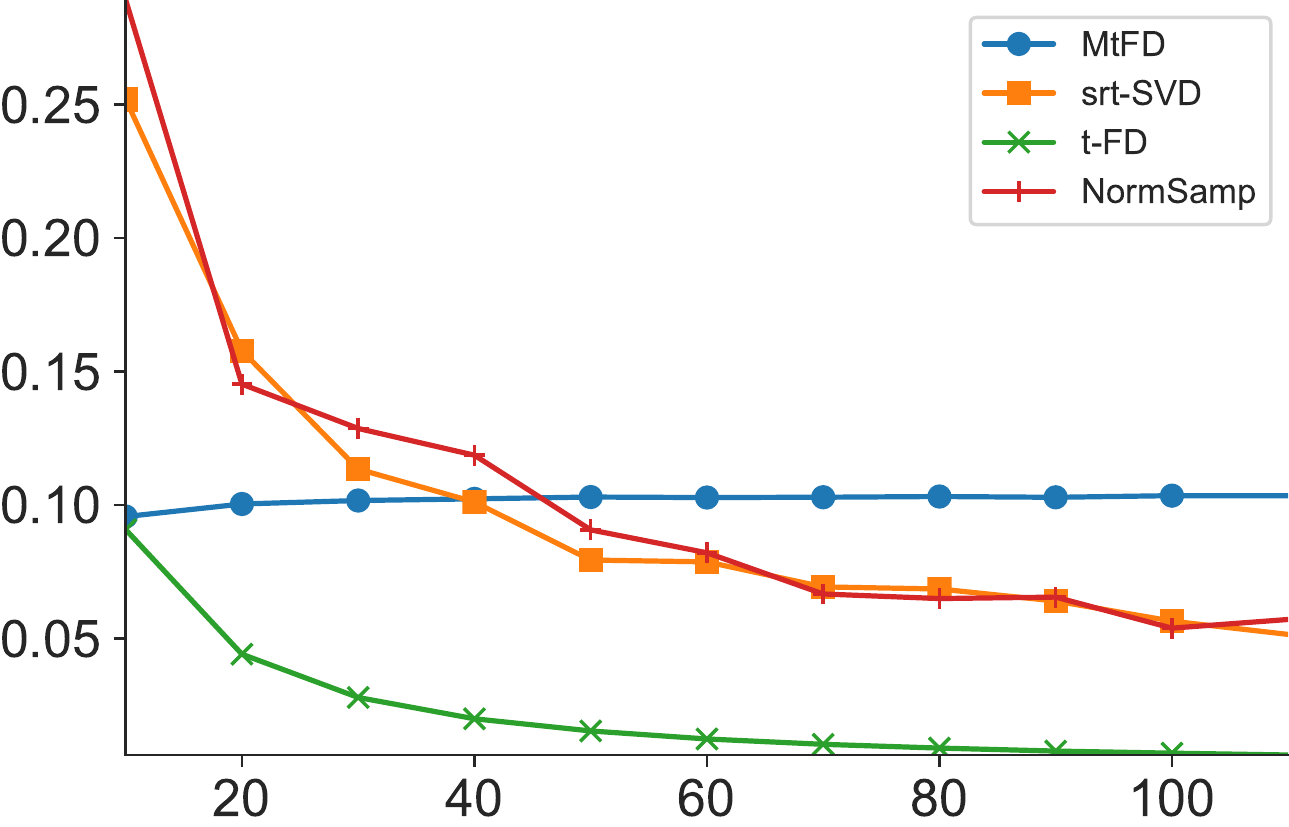}} &  \raisebox{-1mm}{\includegraphics[scale = 0.23]{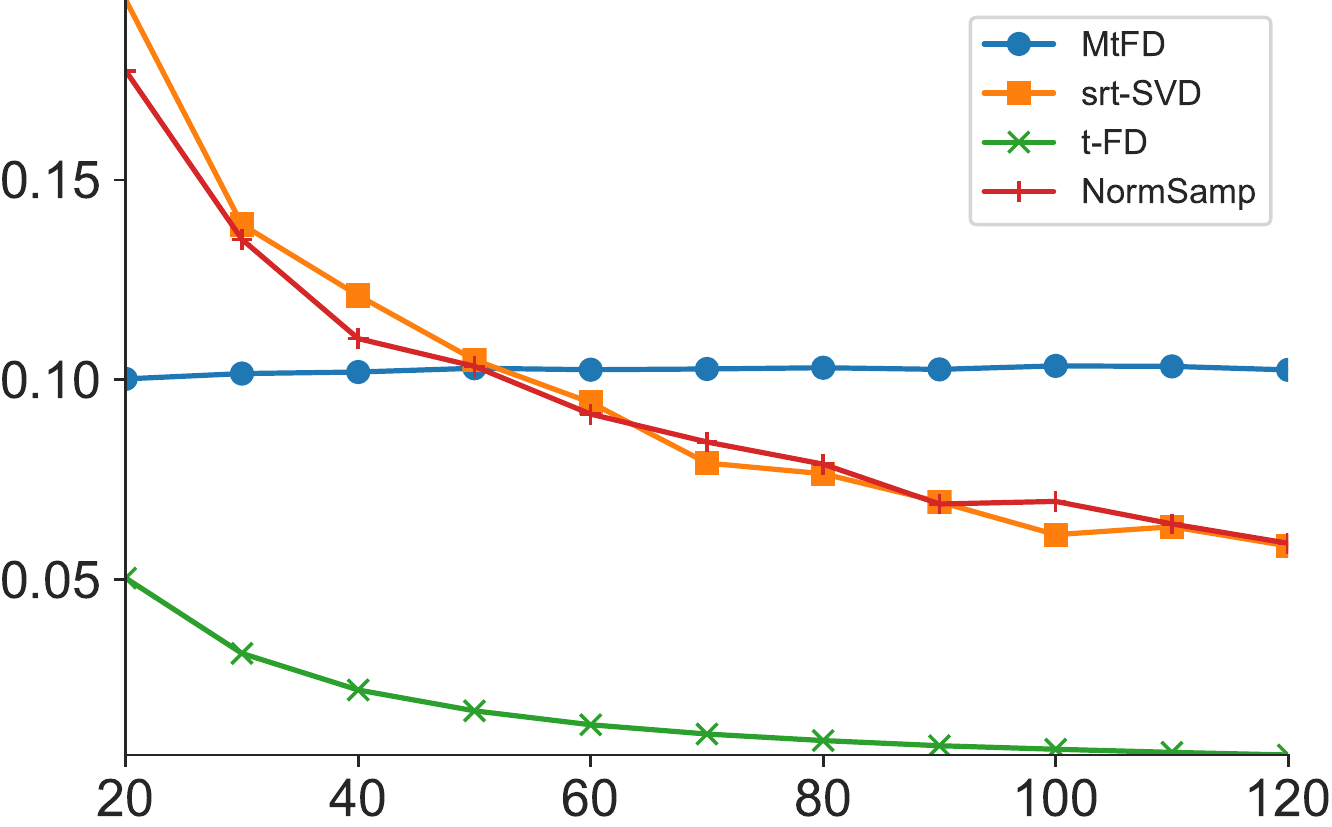}} &  \raisebox{-1mm}{\includegraphics[scale = 0.22]{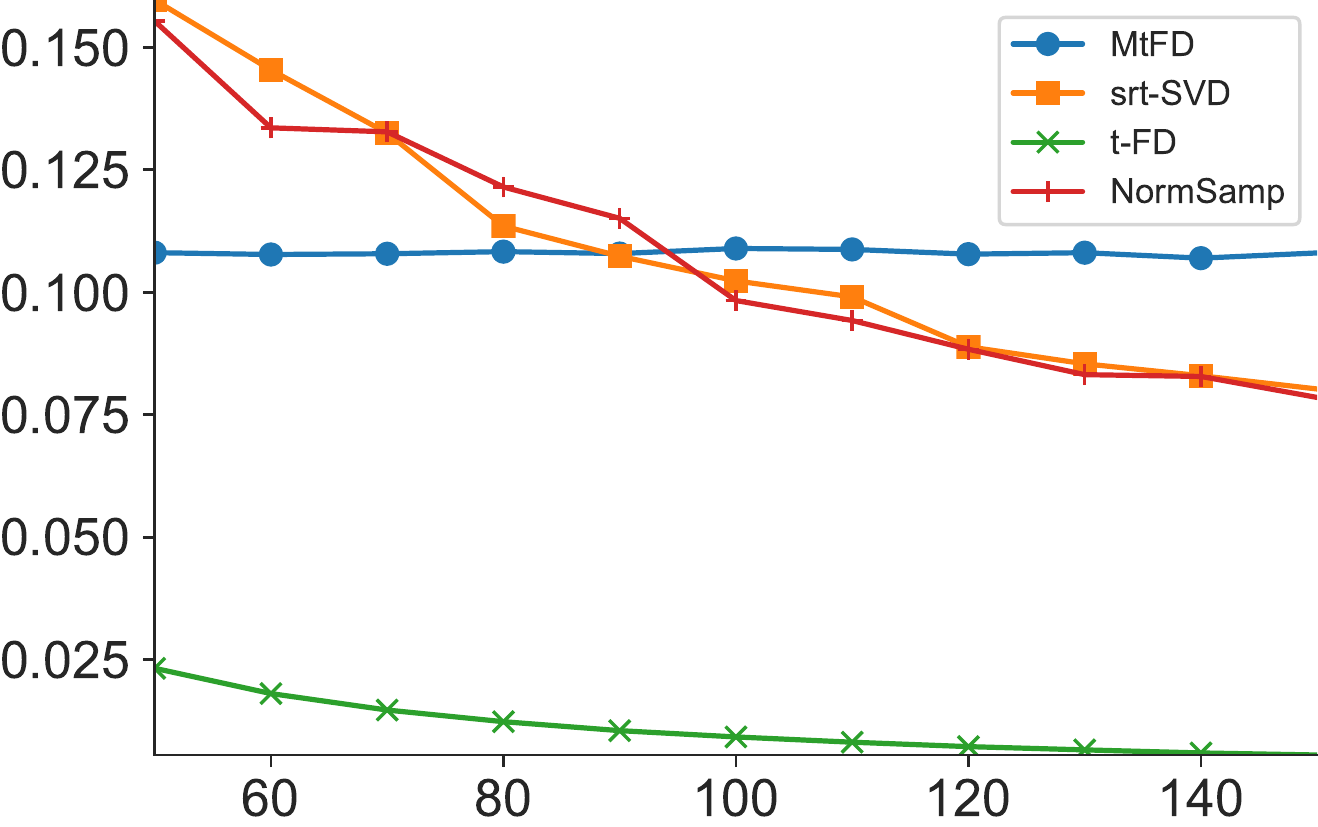}} &  \raisebox{-1mm}{\includegraphics[scale = 0.23]{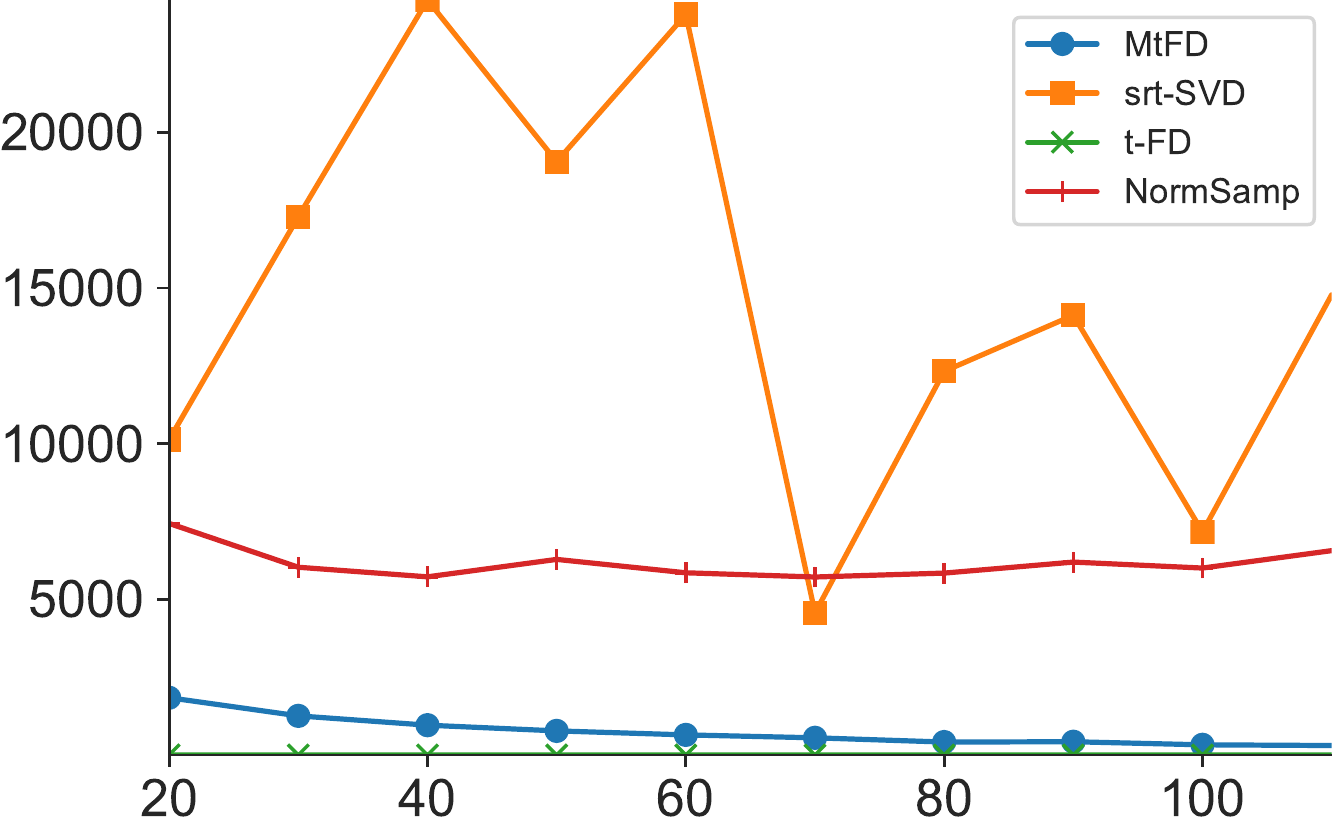}} &  \raisebox{-1mm}{\includegraphics[scale = 0.22]{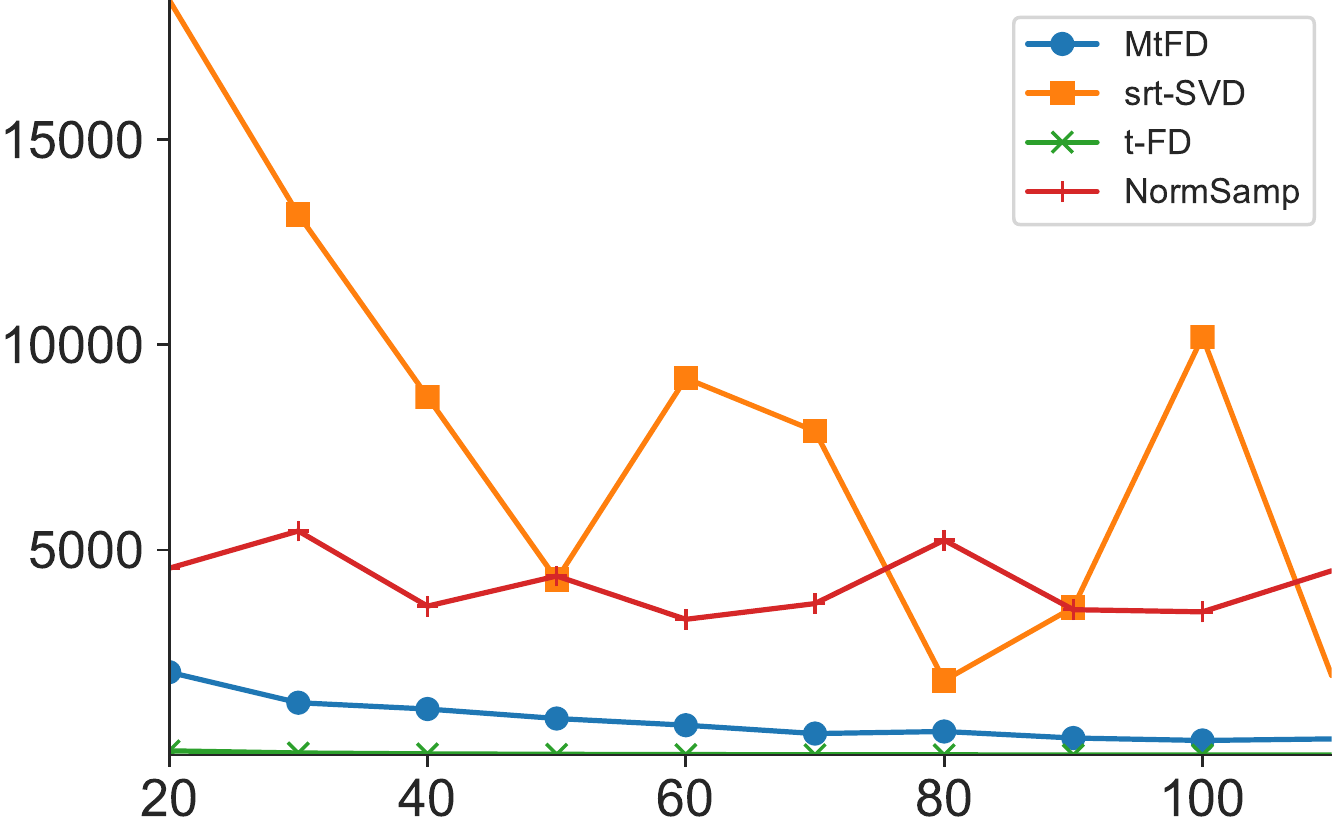}}\\
			\hline
			\rotatebox{90}{\begin{tiny}$\mathsf{Running\ time}$\end{tiny}} & 
			\raisebox{-1mm}{\includegraphics[scale = 0.22]{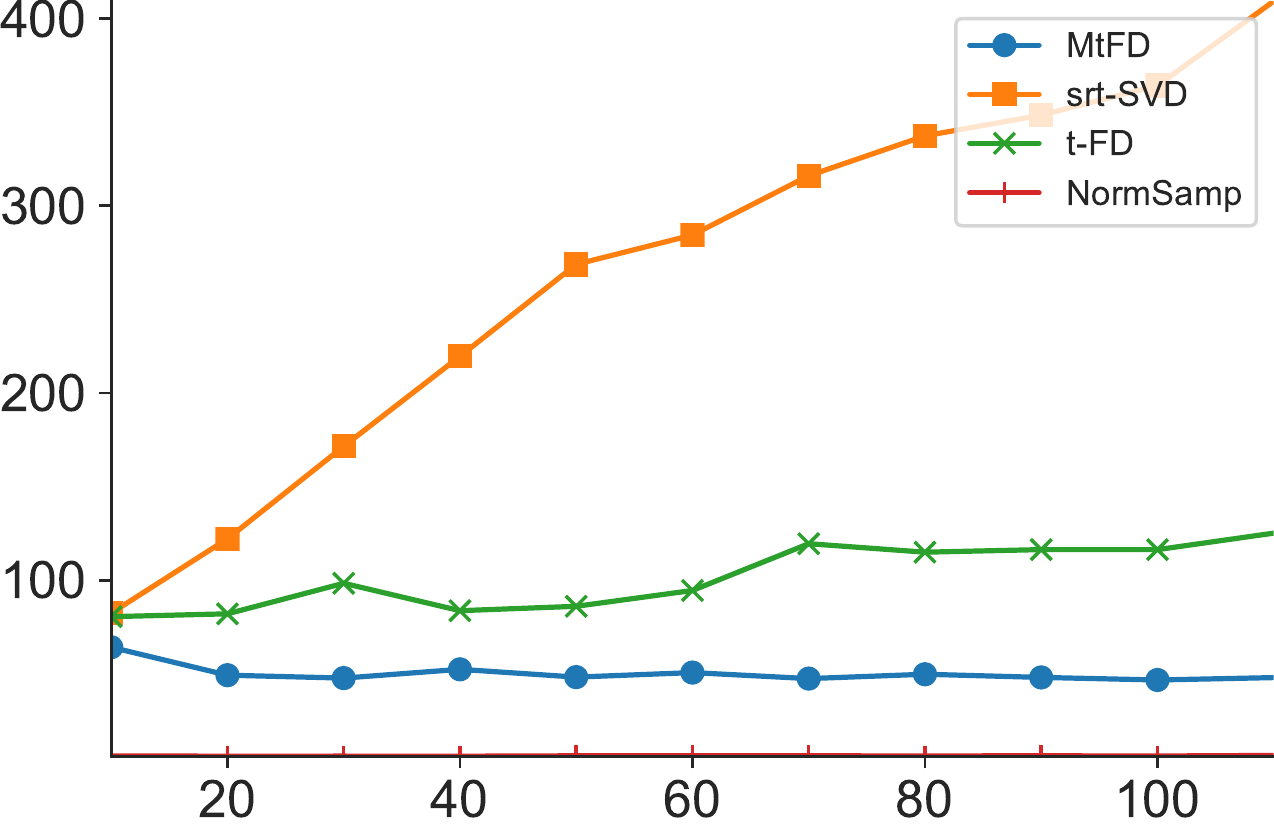}} &  \raisebox{-1mm}{\includegraphics[scale = 0.23]{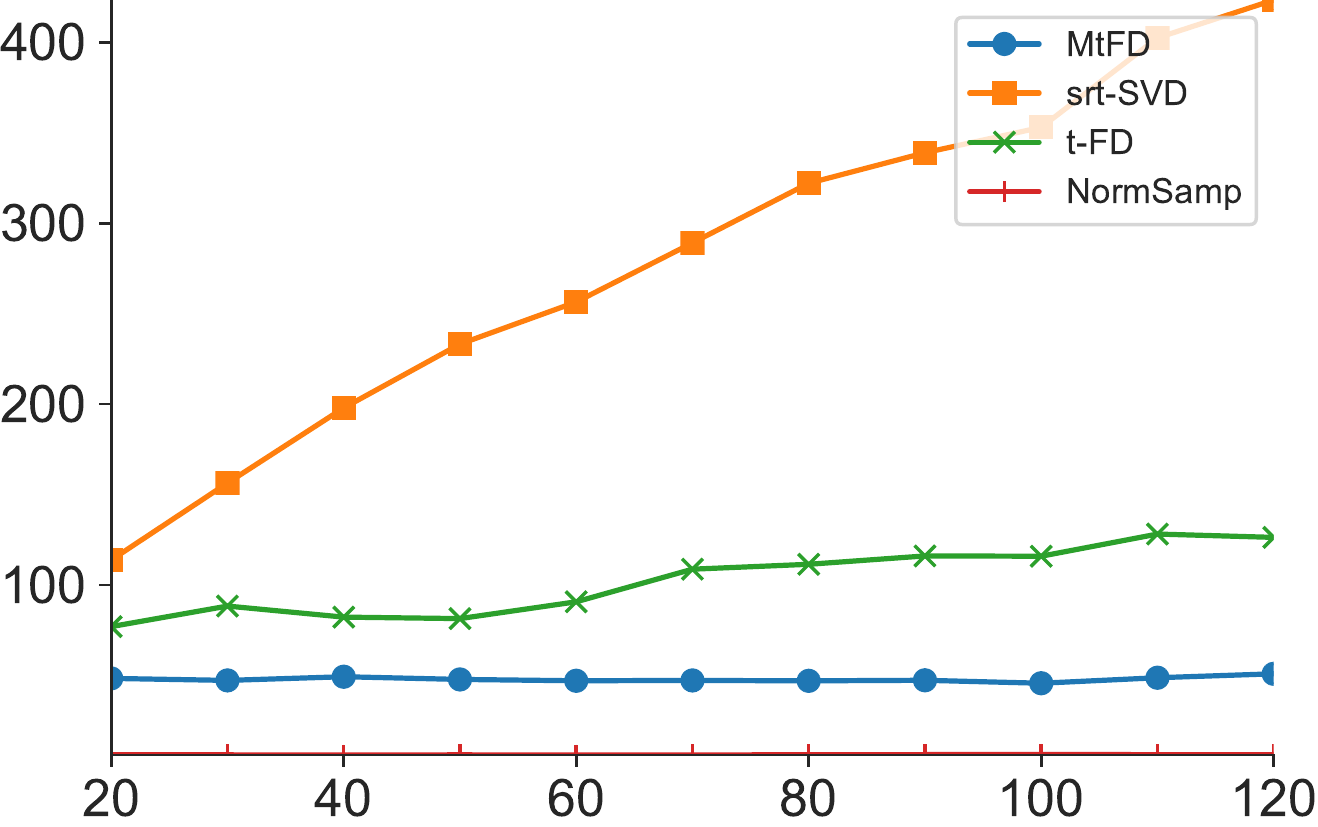}} &  \raisebox{-1mm}{\includegraphics[scale = 0.22]{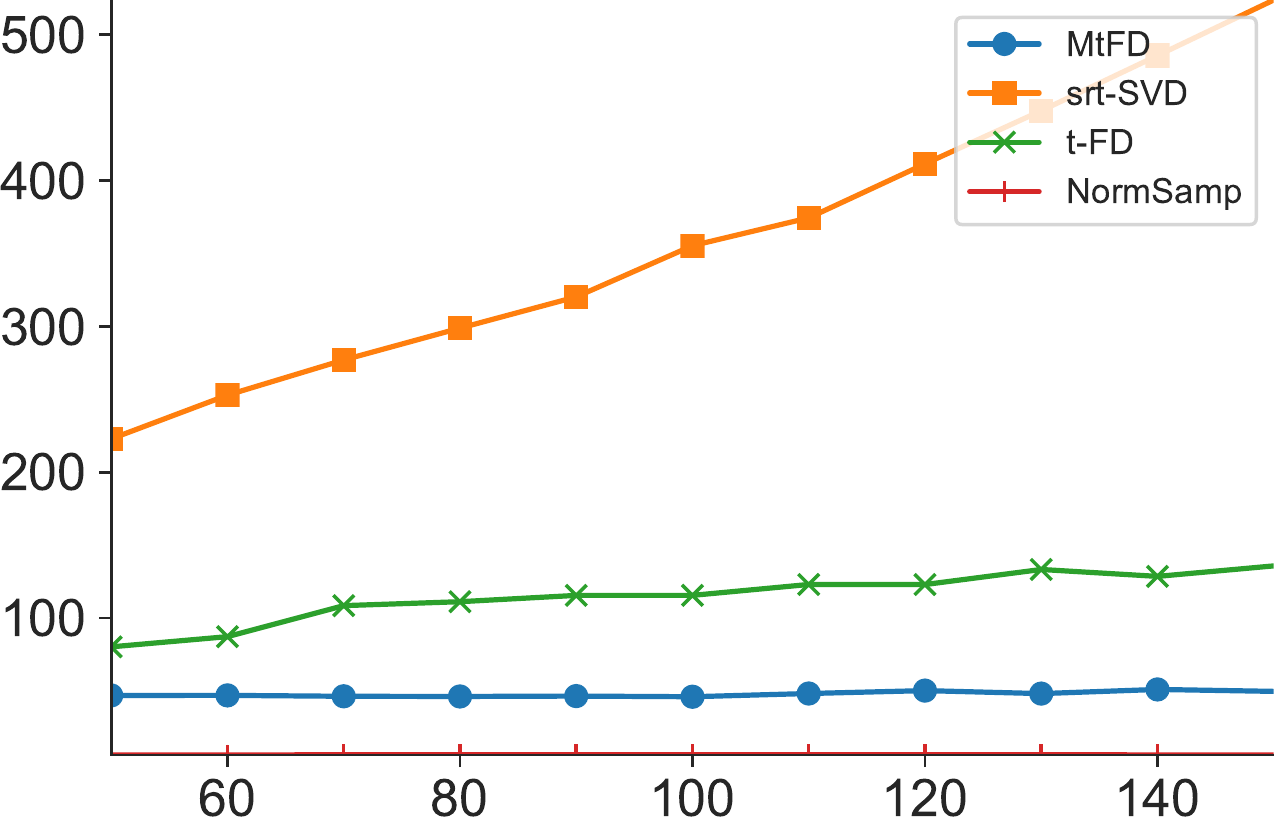}} &  \raisebox{-1mm}{\includegraphics[scale = 0.23]{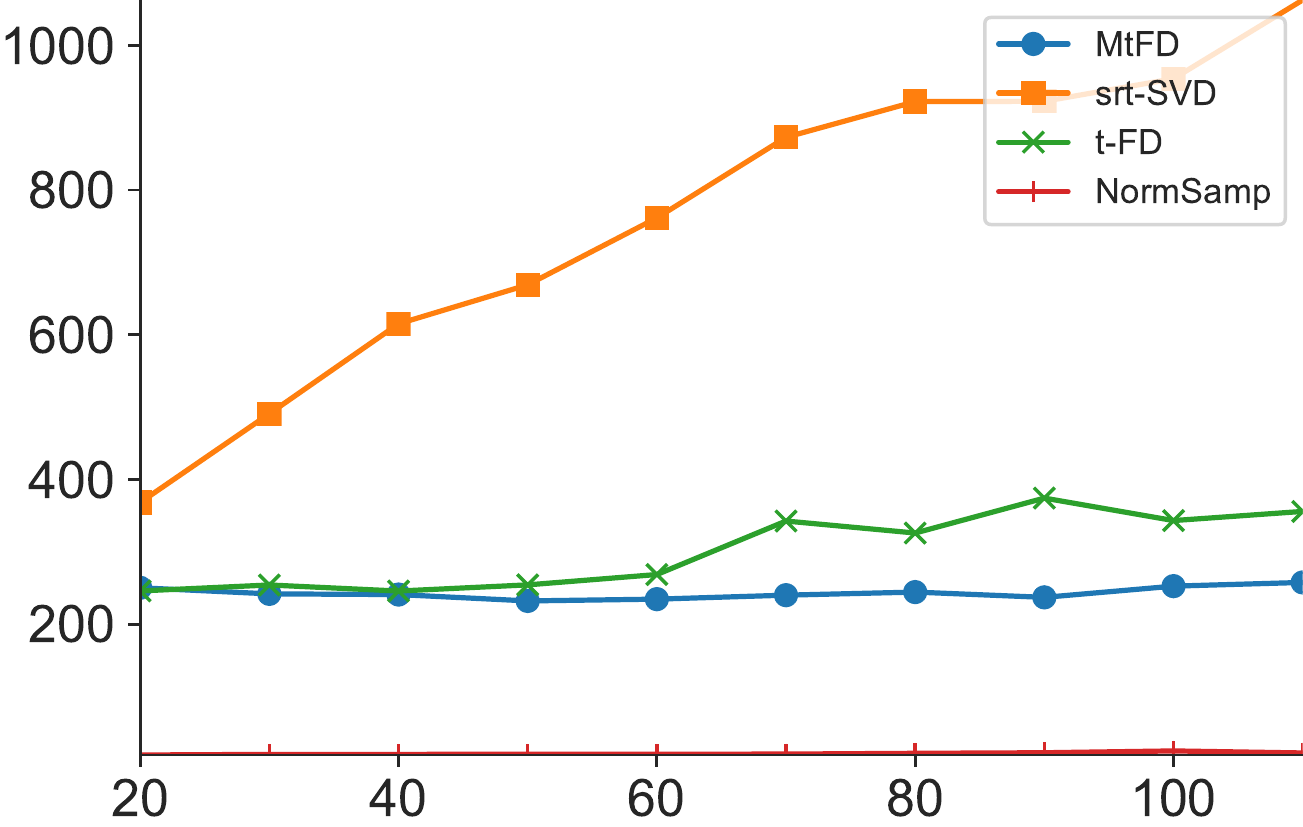}} &  \raisebox{-1mm}{\includegraphics[scale = 0.22]{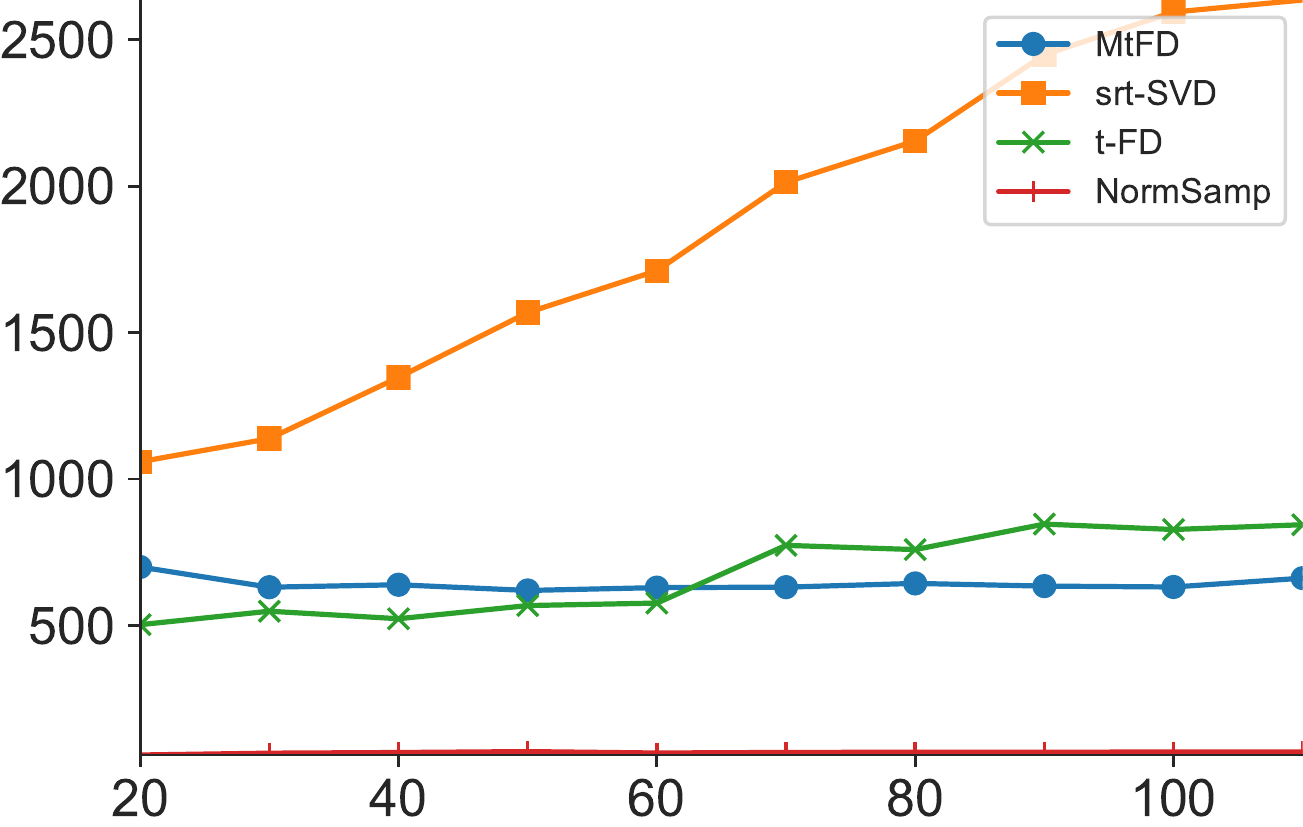}}\\
			\hline
			\rotatebox{90}{\begin{tiny}$\mathsf{The\  value \ of \ c}$\end{tiny}} & 
			\raisebox{-2mm}{\includegraphics[scale = 0.21]{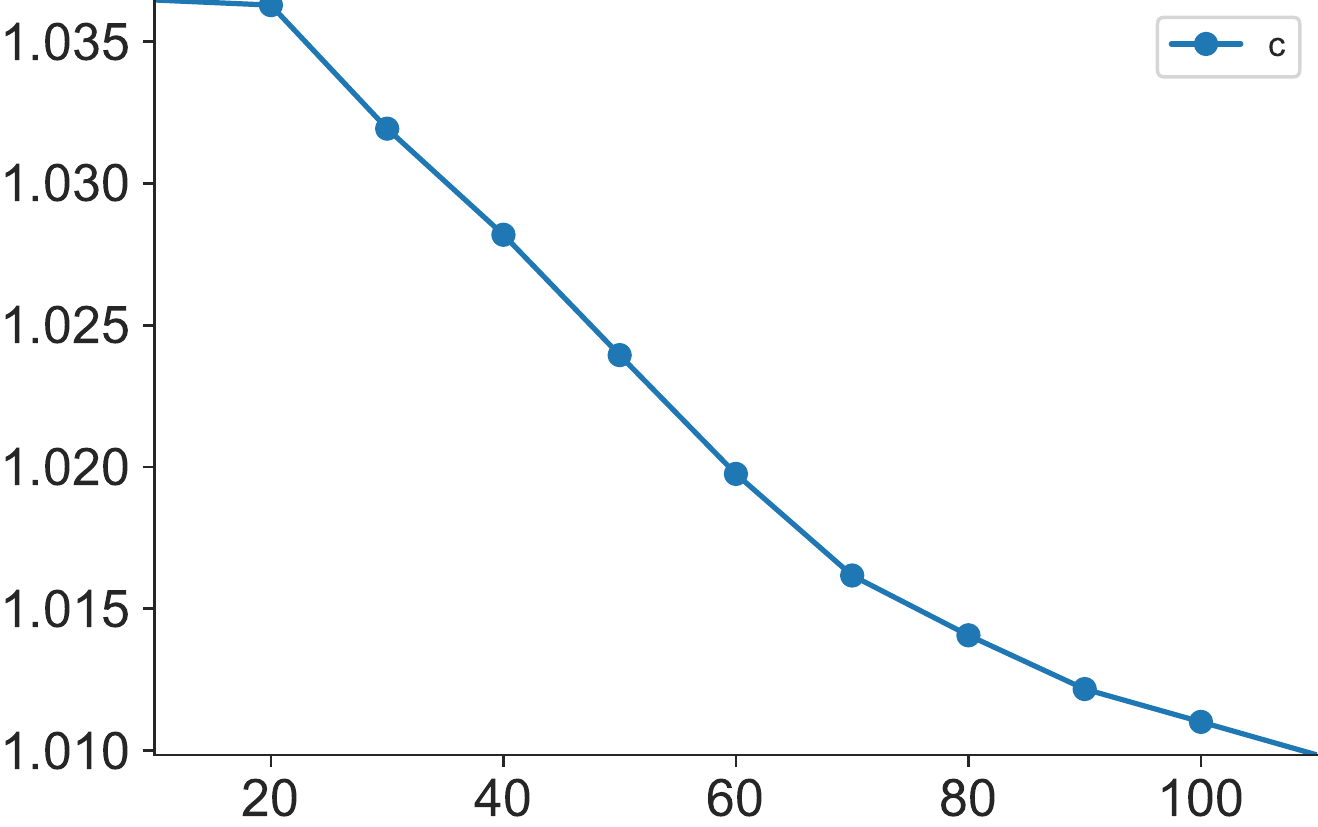}} &  \raisebox{-2mm}{\includegraphics[scale = 0.21]{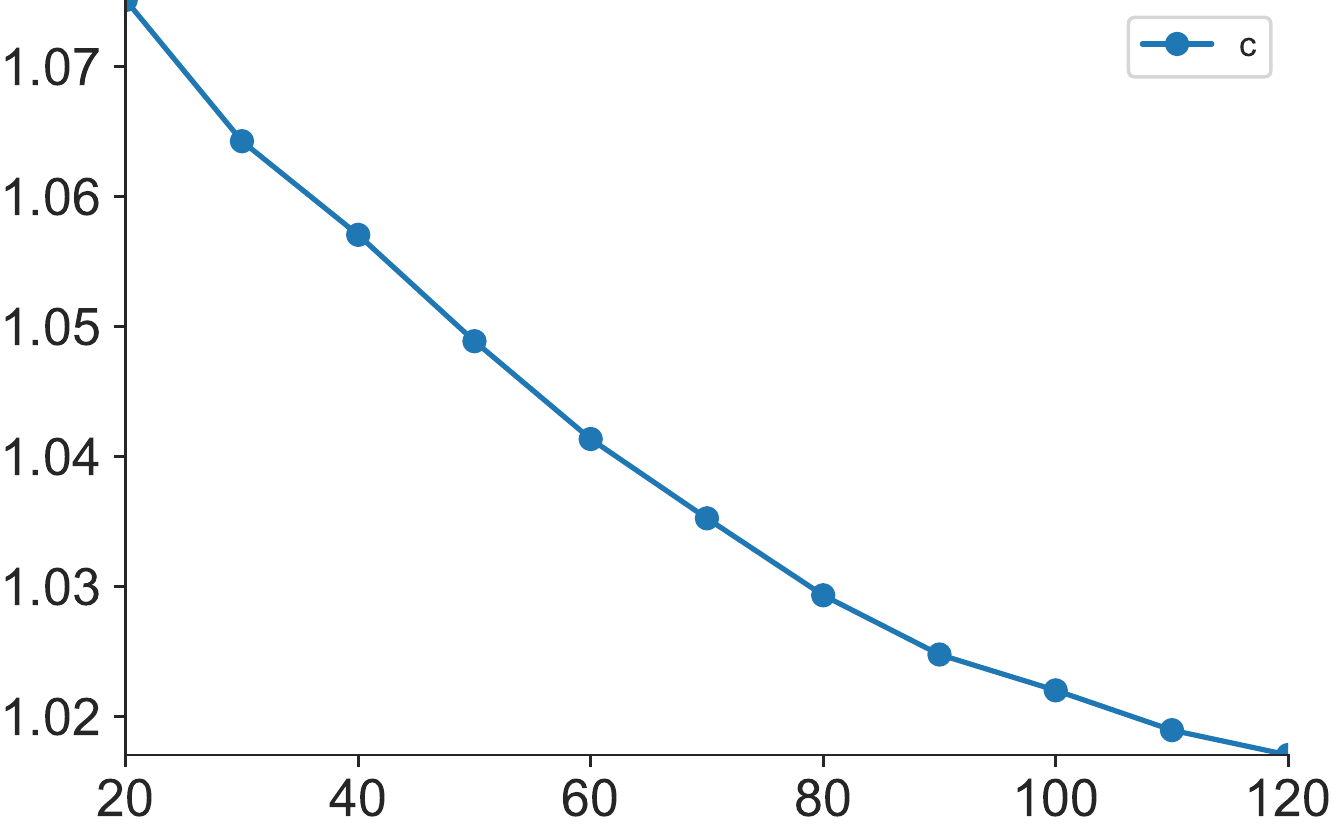}} &  \raisebox{-2mm}{\includegraphics[scale = 0.21]{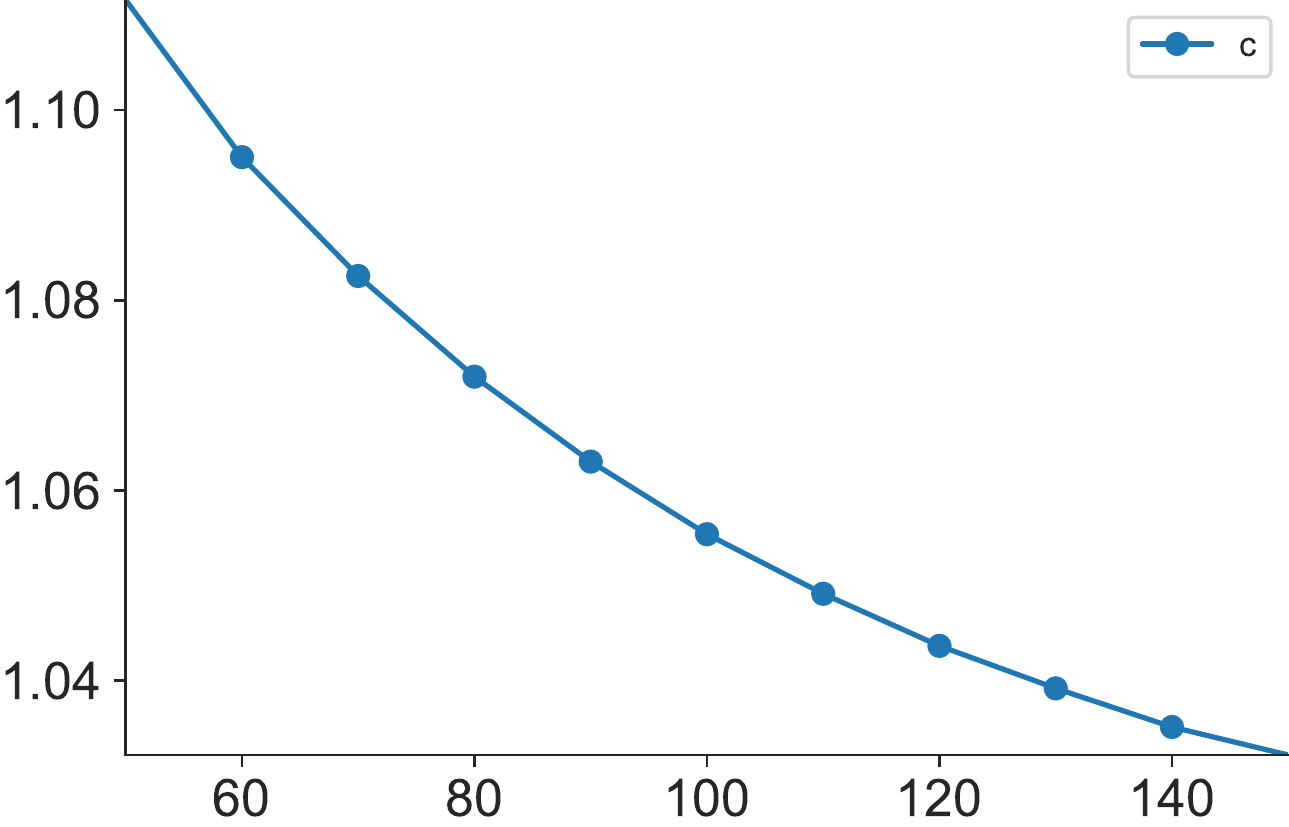}} &  \raisebox{-2mm}{\includegraphics[scale = 0.21]{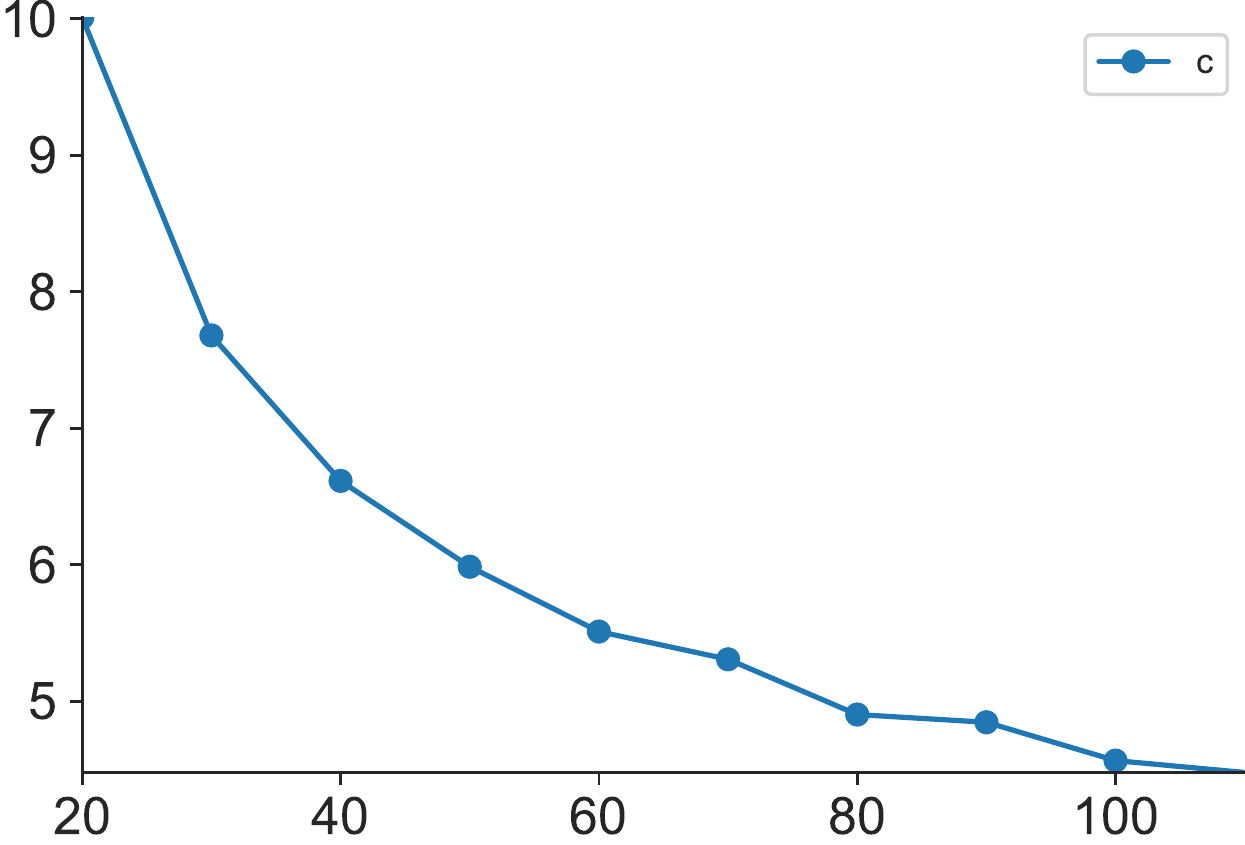}} &  \raisebox{-2mm}{\includegraphics[scale = 0.21]{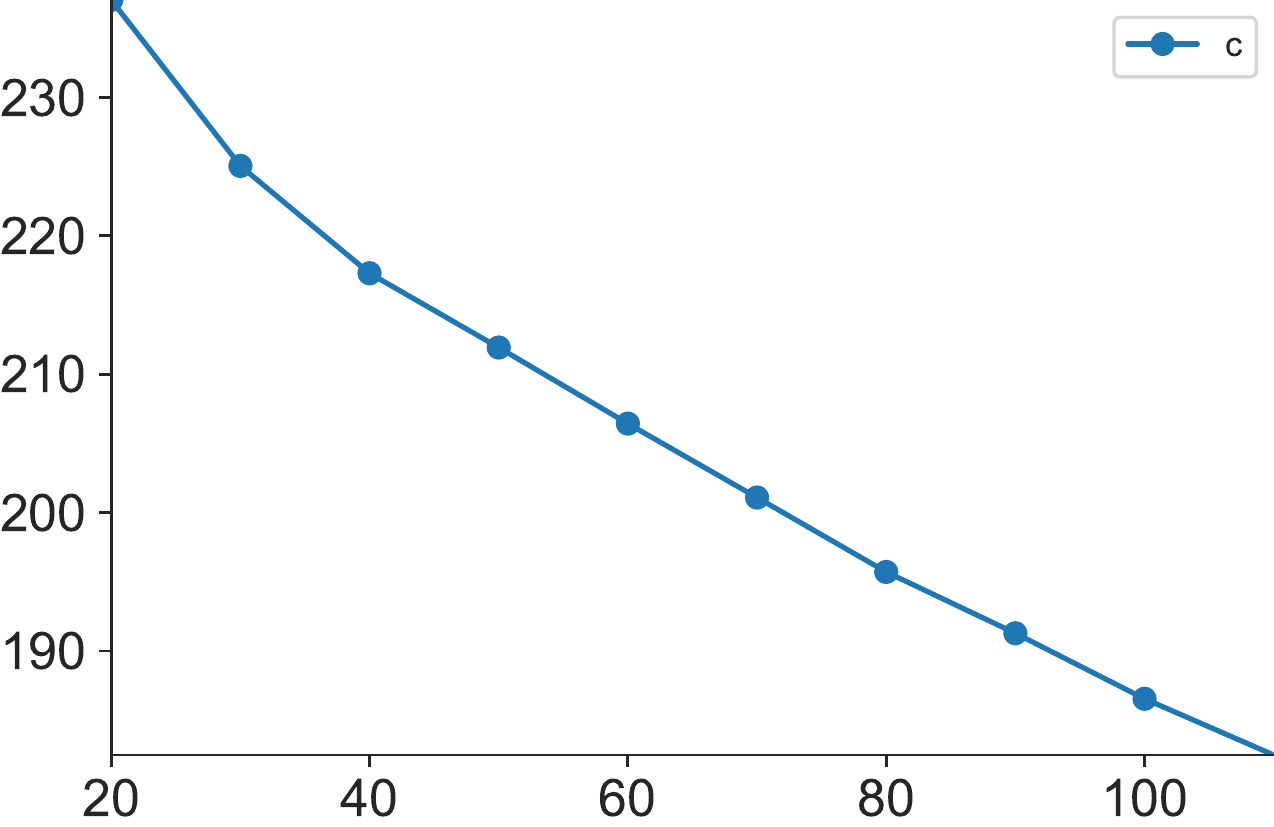}}\\
			\hline
			& $\mathsf{k = 10}$ & $\mathsf{k = 20}$ & $\mathsf{k = 50}$ & $\mathsf{Tabby Cat}$ & $\mathsf{Park Bench}$\\
			\hline
		\end{tabular}
	\end{adjustbox}
	\vspace{10pt}
	\caption{Experimental results for fourth-order synthetic and real datasets} \label{tab:stimuli:fou}
\end{table*}

The performance of t-FD is consistently much better than other  algorithms in terms of both error measures, especially for the covariance error. For the MtFD, even though the sketch tensor could capture a good subspace to achieve low projection error, it fails to approximate well for the covariance. For the third-order tensor, the covariance error only decreases subtly as the sketch size increases. And it  maintains nearly the same for the fourth-order tensor. We attribute this to the intrinsic structure may be destroyed  in the update process. However, our algorithm  shows an obvious decrease when sketch size grows.  Moreover, we notice in the higher rank setting, our method is more competitive.  For the other two randomized algorithms, i.e., srt-SVD and NormSamp, there is small difference between their performance.

For the running time, all these algorithms show a linearly growth in different settings. Clearly, the srt-SVD method is the slowest. For MtFD and t-FD, the t-FD is only slightly slower than MtFD, however, from the performance analysis, we conclude the improvement in the precision is deserved.  Even though NormSamp could be implemented in seconds, the performance is much worse than ours and has no theoretical guarantee as ours.\\
\subsection{Real data examples}
We now test our algorithm using four real-world streaming data. For the highway traffic data \cite{chen2020low}, it records the traffic speed time series over weeks from 11160 sensors and thus can be treated as a dense tensor. Here we choose four weeks data and formulate it as a tensor $\mathcal{B}  \in \mathbb{R}^{11160 \times 288 \times 28}$. Since the sensor data has strong similarity in our observation, a lower rank 10 is used for this comparison. For the Uber data \cite{smith2017frostt}, it can be represented as an extremely sparse tensor $\mathcal{A} \in \mathbb{R}^{183 \times 24 \times 1140 \times 1717}$ with only $0.038\%$ non-zeros. The value at $(i,j,k,l)$ represents the number of pick-ups on day $i$, hours $j$, at latitude $k$ and longitude $l$. We aggregate the time dimension and subsample the location dimension to a tensor $\mathcal{A} \in \mathbb{R}^{4392 \times 500 \times 500}$. Due to the highly sparsity, we set up the rank to $50$.  For the fourth-order tensor datasets, we consider two color video datasets studied in \cite{malik2021sampling}, that is, 
	Tabby Cat that can be represented as a tensor of $1280\times720\times 3 \times286$ and Park Bench that can be represented as a tensor of 1920  $\times$ 1280$\times$ 3 $\times$  364. The rank is set to $20$ for these two datasets. 

It can be easily seen from the last two columns of Tables   \ref{tab:stimuli} and  \ref{tab:stimuli:fou} that, our  algorithm is more accurate and stable, especially for the larger and sparser Uber dataset. We also notice that, even the results are averaged over ten runs, srt-SVD and NormSamp could not  achieve stable results in some cases.



\subsection{Impact of parameter $c$} In our theoretical analysis, the parameter $c$ in Theorems \ref{main1} and \ref{main}  is an uncertainty. It is determined by the structure of the tensor. In our synthetic data as well as the real data, the parameter $c$ is much smaller than $\rho$, in which case our algorithm has a superior performance. Here we construct two extreme  cases to verify the effect of parameter $c$, in which the generated tensor has the form $\mathcal{A} = \mathcal{B} + \alpha \mathcal{U}$. The each frontal slice of $\mathcal{B}$ is the same, sampled from $\mathcal{N}(0,1)$, and $\mathcal{U}$ is a random tensor uniformly distributed on $[0,1]$. The parameter $\alpha$ is set up to control the difference among all the slices. When $\alpha$ becomes smaller, the parameter $c$ would be closer to $\rho$. Thus we consider $\alpha$  varies in $\{0.01, 100\}$ and the test tensor $\mathcal{A}$ with size of  $\mathbb{R}^{3000 \times 300 \times 20}$.

Figs. 2 and 3 show the comparison results between MtFD and t-FD in these two extreme cases. It can be seen that larger $c$ really deteriorates the performance of t-FD, however, we still obtain a comparable performance with MtFD. And when $c$ is very small, our estimation becomes more accurate. Both these findings further demonstrate the superiority of our tensor version of FD over the direct matricization technique for tackling the tensor data.

In the Tables \ref{tab:stimuli} and  \ref{tab:stimuli:fou}, we also draw the parameter $c$ in each setting. For the synthetic datasets, the behaviors of parameter $c$ is consistently decrease as the sketch size becomes larger, and the value of $c$ is much closer to 1. However, for the higher order ParkBench dataset, even though the value of $c$ is  larger than those of other methods, it is still much smaller than $\rho$. Moreover, it decreases quickly with larger sketch sizes. Additionally, our algorithm still outperforms other compared methods, which is consistent with the previous experimental results.

\begin{figure}[t]
	\label{fig:com}
	\centering
	\subfigure{
		\includegraphics[scale=0.25]{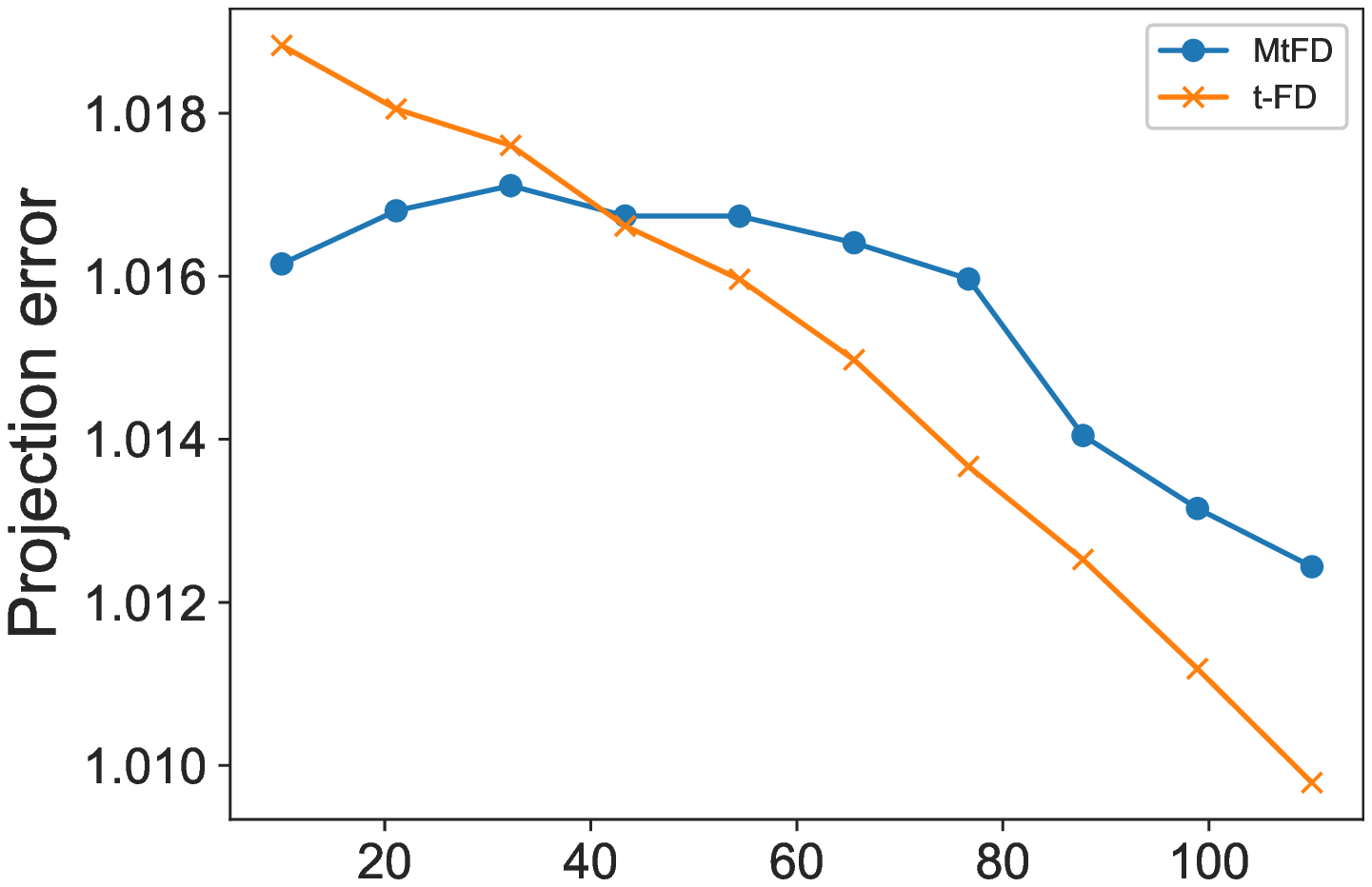}
	}
	\quad
	\subfigure{
		\includegraphics[scale=0.25]{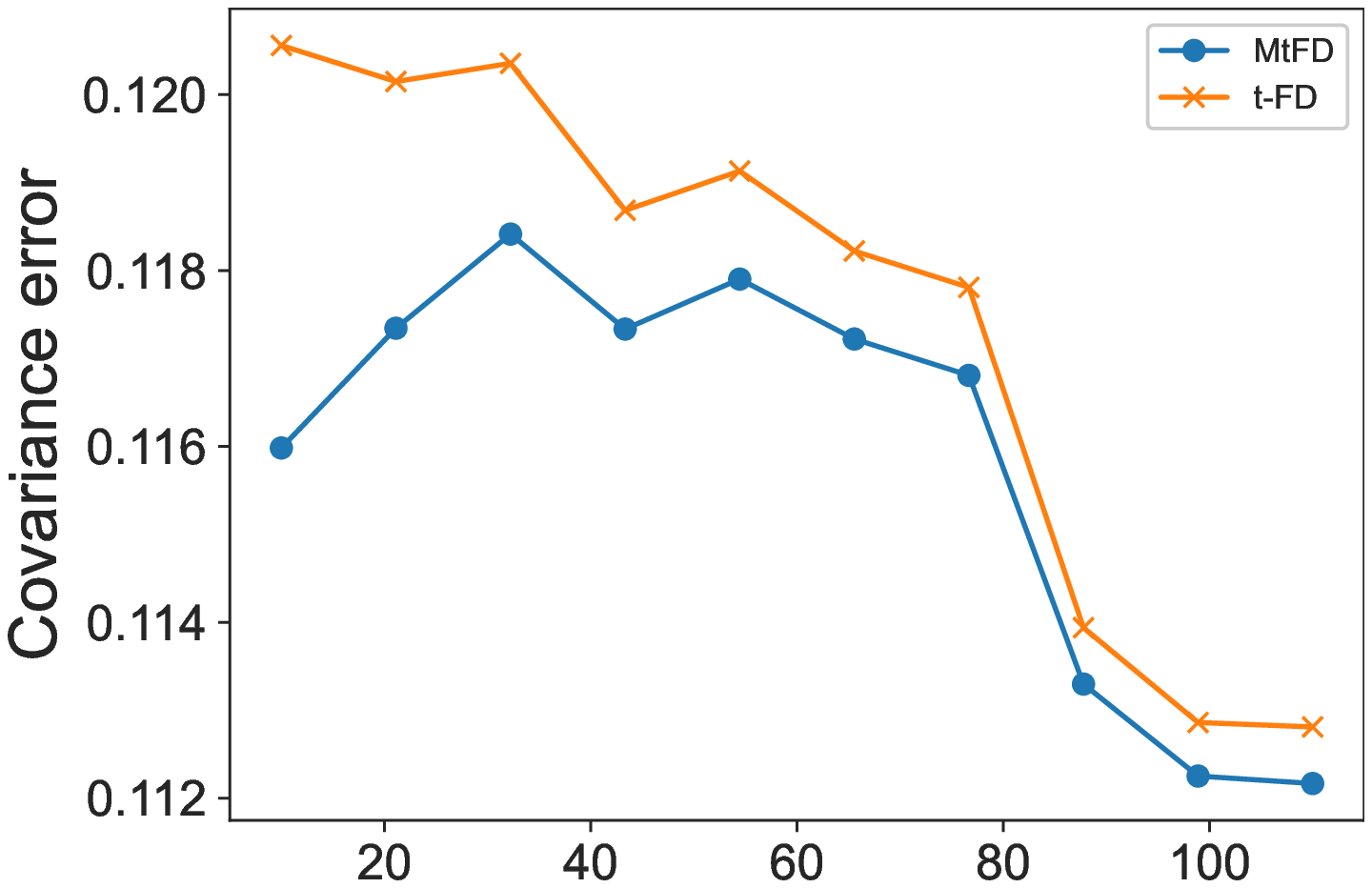}
	} \caption{$\alpha = 0.01, c \approx 19.75$}
	
	\subfigure{
		\includegraphics[scale=0.25]{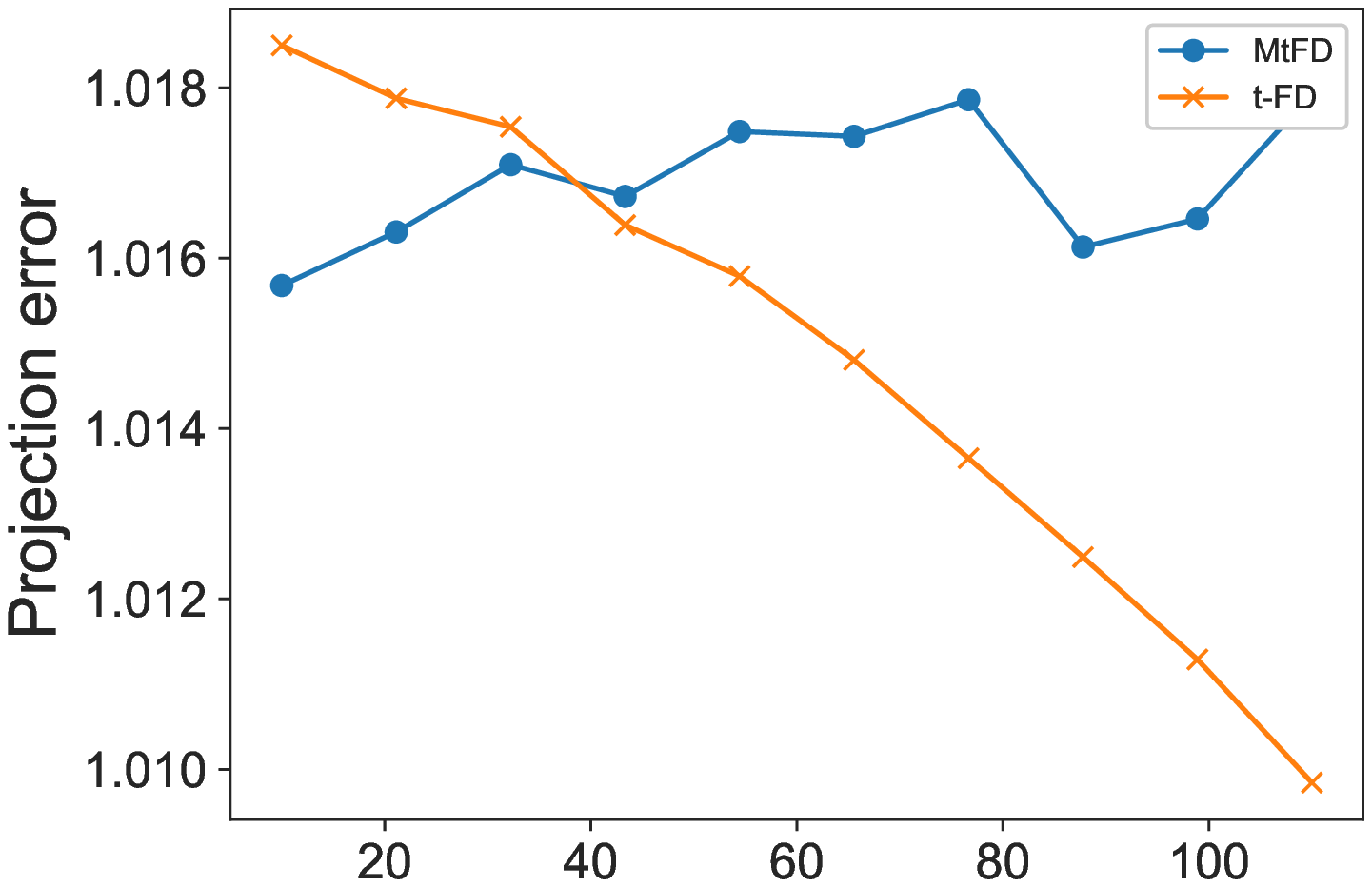}
	}	
	\quad
	\subfigure{
		\includegraphics[scale=0.25]{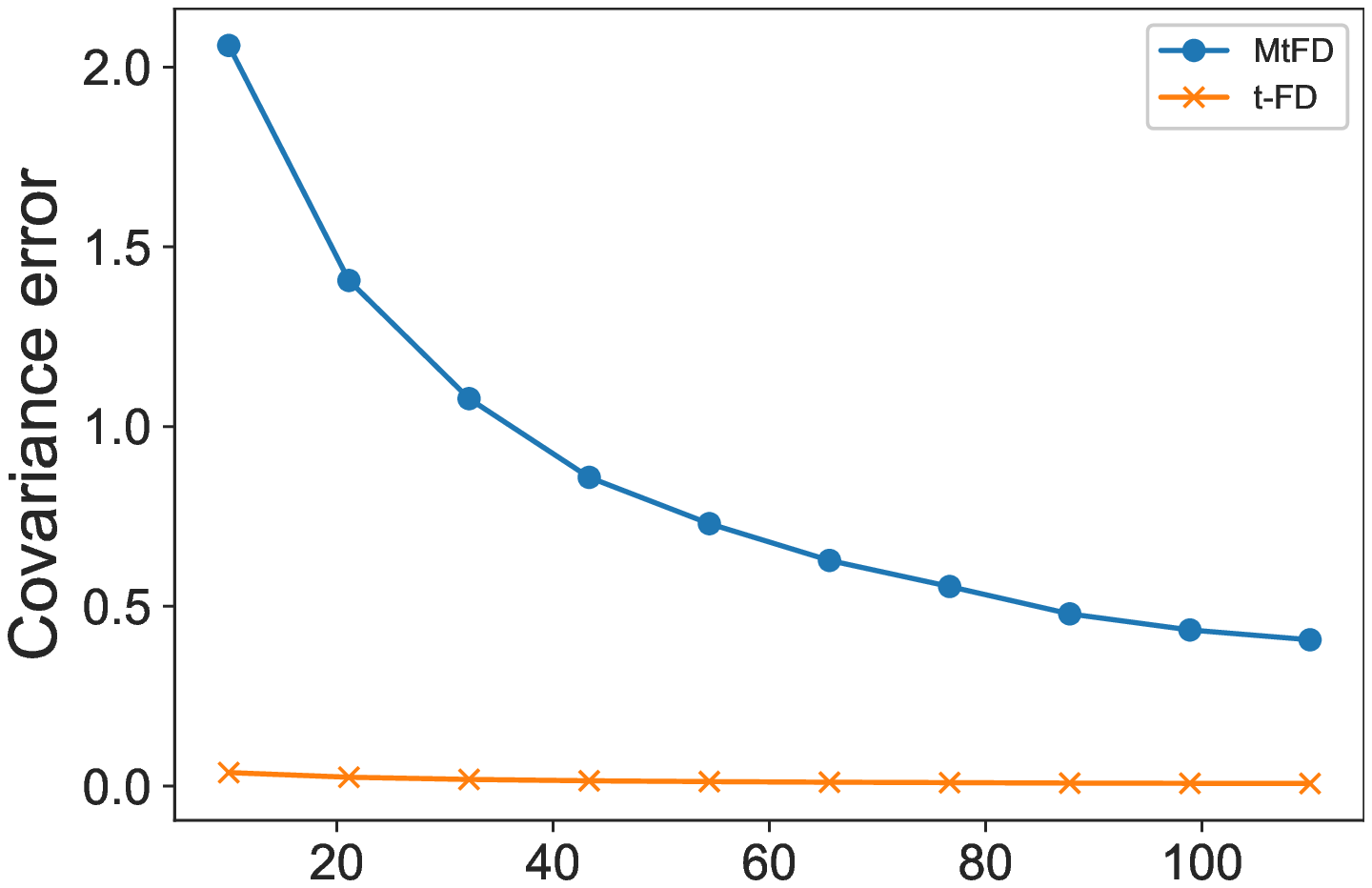}	
	}
	\caption{$\alpha = 100, c \approx 1.01$}
\end{figure}
\vspace{-5mm}
\subsection{Application to Video Scene Classification}
In this subsection, we present how to use our algorithm to classify real video scenes. The video \cite{malik2018low} is documented by a fixed camera, and a person occurs in the camera twice. It consists of 2200 frames, each of size 1080 by 1980. Our aim is to identify the frames in which a person occurs. We sequentially load the whole tensor through the second dimension, by each time loading a slice $\mathcal{A}_j \in \mathbb{R}^{1080 \times 2200}$. We then choose the sketch size $\ell$ varied in $\{10,20,50\}$. As such, we obtain a sketch tensor $\mathcal{B} \in \mathbb{R}^{\ell \times 1080 \times 2200}$. Thus,  by applying the t-SVD, we could obtain the  dominant space $\mathcal{U} \in \mathbb{R}^{\ell \times \ell \times 2200}$, and further get  the mean matrix $\boldsymbol{U} \in \mathbb{R}^{\ell \times 2200}$ along the second dimension. The $i$-th column of $\boldsymbol{U}$ represents the feature vector of the $i$-th frame. To identify the frames, we apply $K$-means clustering algorithm to those feature vectors corresponding to all 2200 frames.

In this real-world application, we could identify most of the frames containing a person by using the proposed t-FD algorithm. Some typical results are demonstrated in Fig. \ref{clf}. In this figure,  when no person appears, the frames are classified into green class; when a person is captured by the camera, the  frames are marked as orange class. Compared with the previous works \cite{malik2018low, sun2019low}, the classification results obtained by t-FD are similar even though we select smaller clusters than such two works. Additionally, among different sketch sizes, the orange frames are all classified correctly for the smaller sketch size $\ell \in \{10, 20\}$, while more frames including a person are classified for the sketch size $\ell = 50$, but also some frames are misclassified. 

%

\begin{figure*}[htp!]
	
	\centering
	\includegraphics[scale=0.56]{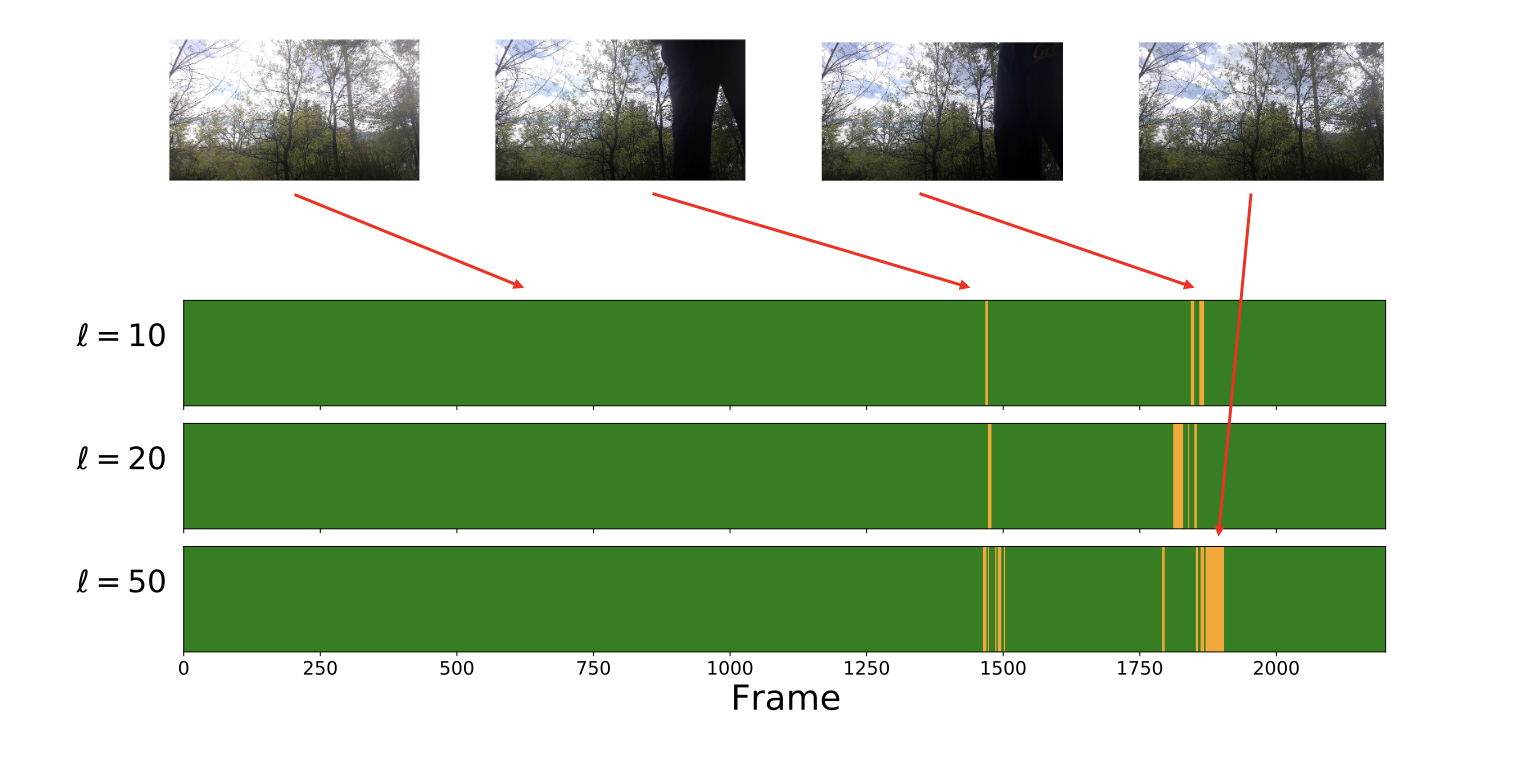}
	\vspace{-10pt}
	\caption{Classification results for different sketch sizes}
	\label{clf}
\end{figure*}

\section{Proofs}
In this section, we shall prove that our proposed algorithm t-FD is within $1+\varepsilon$ of best tubal-rank-$k$ approximation. Meanwhile, we derive the error bounds of MtFD for comparison. To this end, we first need to prove the following auxiliary properties and lemmas. 

\subsection{Some useful lemmas and properties}

Since Lemma \ref{lemma2} plays an important role in our theoretical analysis, we shall first present the proof of the lemma. 
\begin{proof}[Proof of Lemma \ref{lemma2}]
	Set $n=\min \left(n_{1}, n_{2}\right)$, then due to the property that $\left\|\mathcal{A}\right\|_{F}^{2}=\left\|\mathcal{S}\right\|_{F}^{2}=\frac{1}{\rho}\left\|\boldsymbol{\bar{S}}\right\|_{F}^{2}$, we can get that
	$$
	\begin{aligned}
		&\left\|\mathcal{A}-\mathcal{A}_{k}\right\|_{F}^{2} \notag\\
		=&\|\mathcal{S}(k+1: n, k+1: n,:,\ldots,:)\|_{F}^{2} \\
		=&\rho\left\|\boldsymbol{\bar{S}}^{(1)}(k+1: n, k+1: n)\right\|_{F}^{2}+\cdots\notag\\
		&+\rho\left\|\boldsymbol{\bar{S}}^{(\rho)}(k+1: n, k+1: n)\right\|_{F}^{2}.
	\end{aligned}
	$$
	Now let $\mathcal{B} \in \mathbb{A}$, so that $\mathcal{B}=\mathcal{X} * \mathcal{Y}^{T}$. Then
	$$
	\begin{aligned}
		&\|\mathcal{A}-\mathcal{B}\|_{F}^{2}\notag\\ 
		=&\rho\left\|\boldsymbol{\bar{A}}^{(1)}-\boldsymbol{\bar{X}}^{(1)} \boldsymbol{\bar{Y}}^{(1)T}\right\|_{F}^{2}+\cdots+\rho\left\|\boldsymbol{\bar{A}}^{(\rho)}-\boldsymbol{\bar{X}}^{(\rho)} \boldsymbol{\bar{Y}}^{(\rho)T}\right\|_{F}^{2} \\
		\geqslant& \rho\left\|\boldsymbol{\bar{S}}^{(1)}(k+1: n, k+1: n)\right\|_{F}^{2}+\cdots\notag\\
		&+\rho\left\|\boldsymbol{\bar{S}}^{(\rho)}(k+1: n, k+1: n)\right\|_{F}^{2} .
	\end{aligned}
	$$
	This finishes the proof of Lemma \ref{lemma2}.
\end{proof}

The block circulant operation on tensors acts as a bridge for seeking the tensor norm relationship between the original domain and the Fourier domain. At first, we briefly review the relevant properties about it.
\begin{lemma}\cite{lund2020tensor}\label{lund}
	Given tensors $\mathcal{A} \in \mathbb{C}^{n \times n \times p}$ and $\mathcal{B} \in \mathbb{C}^{n \times s \times p} .$ Then\\
	(1) $\mathtt{bcirc}(\mathcal{A} * \mathcal{B})=\mathtt{bcirc}(\mathcal{A}) \mathtt{bcirc}(\mathcal{B})$;\\
	(2) $(\mathcal{A} * \mathcal{B})^{\top}=\mathcal{B}^{\top} * \mathcal{A}^{\top}$;\\
	(3) $\mathtt{bcirc}\left(\mathcal{A}^{\top}\right)=(\mathtt{bcirc}(\mathcal{A}))^{\top}$.
\end{lemma} 
For proving our main theorems, we next need to prove the following three auxiliary properties of Algorithm \ref{tensor-FD}. In the subsequent analysis, the t-SVD of $\mathbf{\mathcal{A}}$ and $\mathbf{\mathcal{B}}$ are expressed as $\mathbf{\mathcal{A}}=\mathbf{\mathcal{Z}}*\mathbf{\mathcal{W}}*\mathbf{\mathcal{Y}}^{T}$ and $\mathbf{\mathcal{B}}=\mathbf{\mathcal{U}}*\mathbf{\mathcal{S}}*\mathbf{\mathcal{V}}^{T}$, respectively. The corresponding rank-$k$ approximation are  $\mathbf{\mathcal{A}}_{k}=\mathbf{\mathcal{Z}}_{k}*\mathbf{\mathcal{W}}_{k}*\mathbf{\mathcal{Y}}^{T}_{k}$ and $\mathbf{\mathcal{B}}_{k}=\mathbf{\mathcal{U}}_{k}*\mathbf{\mathcal{S}}_{k}*\mathbf{\mathcal{V}}_{k}^{T}$. Moreover, let $\vec{\boldsymbol{y}}_{i} \in \mathbb{R}^{n_{2} \times 1 \times n_{3}}$ be the $i$-th lateral slice of $\mathbf{\mathcal{Y}}_{k}$, and $\vec{\boldsymbol{v}}_{i} \in \mathbb{R}^{n_{2} \times 1 \times n_{3}}$ be the $i$-th lateral slice of $\mathbf{\mathcal{V}}_{k}$. Let  $ \Delta = \sum_{j=1}^{n_1}\max\limits_{i} \delta_{j}^{(i)}$ be the sum of maximum information losses in the truncated procedure. 
\begin{property}\label{pro1}
	For any tensor column $\vec{\boldsymbol{x}} \in \mathbb{R}^{n_{2} \times 1 \times n_{3}}$, if $\mathbf{\mathcal{B}}$ is the output result by applying Algorithm \ref{tensor-FD} to the input $\mathbf{\mathcal{A}}$, then $\left\|\mathbf{\mathcal{A}} * \vec{\boldsymbol{x}}\right\|_{2^{*}}	^{2}-\left\|\mathbf{\mathcal{B}} * \vec{\boldsymbol{x}}\right\|_{2^{*}}	^{2} \ge0$.
\end{property} 
\begin{proof}
	Let $\boldsymbol{x}\in \mathbb{R}^{n_{2}n_{3} \times 1 }$ be the vectorized column vector
	of $\vec{\boldsymbol{x}} $. According to the definition of t-product (Def. 1), we obtain 
	$$
	\left\|\mathbf{\mathcal{A}} * \vec{\boldsymbol{x}}\right\|_{2^{*}}	^{2}-\left\|\mathbf{\mathcal{B}} * \vec{\boldsymbol{x}}\right\|_{2^{*}}	^{2}  
	=\left\|\mathtt{bcirc}(\mathbf{\mathcal{A}})\boldsymbol{x}\right\|^{2}-\left\|\mathtt{bcirc}(\mathbf{\mathcal{B}})\boldsymbol{x}\right\|^{2}.
	$$
	
	\noindent Furthermore, from algorithm t-FD, it is clear to observe that $\boldsymbol{C}_{j}^{(i)}$ contains two parts, one of which is the $\boldsymbol{B}_{j-1}^{(i)}$ produced by the last iteration, and the other is the newly inserted row $\boldsymbol{A}_{j}^{(i)}$. Then according to the definition of the block circulant matrix, we can obtain that $$\left\|\mathtt{bcirc}(\mathbf{\mathcal{A}})\boldsymbol{x}\right\|^{2}+\sum_{j=1}^{n_1}\left\|\mathtt{bcirc}(\mathbf{\mathcal{B}}_{j-1})\boldsymbol{x}\right\|^{2}=\sum_{j=1}^{n_1}\left\|\mathtt{bcirc}(\mathbf{\mathcal{C}}_{j})\boldsymbol{x}\right\|^{2}.$$ 
	
	\noindent Therefore, 
	\begin{align}
		&\left\|\mathbf{\mathcal{A}} * \vec{\boldsymbol{x}}\right\|_{2^{*}}	^{2}-\left\|\mathbf{\mathcal{B}} * \vec{\boldsymbol{x}}\right\|_{2^{*}}	^{2}  \notag \\
		=&\left\|\mathtt{bcirc}(\mathbf{\mathcal{A}})\boldsymbol{x}\right\|^{2}-\left\|\mathtt{bcirc}(\mathbf{\mathcal{B}})\boldsymbol{x}\right\|^{2} \notag\\
		=&\left\|\mathtt{bcirc}(\mathbf{\mathcal{A}})\boldsymbol{x}\right\|^{2}+\sum_{j=1}^{n_1}\left(\left\|\mathtt{bcirc}(\mathbf{\mathcal{B}}_{j-1})\boldsymbol{x}\right\|^{2}-\left\|\mathtt{bcirc}(\mathbf{\mathcal{B}}_{j})\boldsymbol{x}\right\|^{2} \right) \notag\\
		=&\sum_{j=1}^{n_1}\left(\left\|\mathtt{bcirc}(\mathbf{\mathcal{C}}_{j})\boldsymbol{x}\right\|^{2}-\left\|\mathtt{bcirc}(\mathbf{\mathcal{B}}_{j})\boldsymbol{x}\right\|^{2} \right) \notag\\
		\ge&0. \notag
	\end{align}
	This finishes the proof of Property \ref{pro1}.
\end{proof}

\begin{property}\label{pro2}
	For any tensor column $\vec{\boldsymbol{x}} \in \mathbb{R}^{n_{2} \times 1 \times n_{3}}$ satisfied $\|\vec{\boldsymbol{x}}\|_{2^{*}} = 1$, if $\mathbf{\mathcal{B}}$ is the output result by applying Algorithm \ref{tensor-FD} to the input $\mathbf{\mathcal{A}}$, then $\left\|\mathbf{\mathcal{A}} * \vec{\boldsymbol{x}}\right\|_{2^{*}}	^{2}-\left\|\mathbf{\mathcal{B}} * \vec{\boldsymbol{x}}\right\|_{2^{*}}	^{2} \le  \Delta $ .
\end{property} 

\begin{proof} 
	For $\vec{\boldsymbol{x}} \in \mathbb{R}^{n_{2} \times 1 \times n_{3}}$, $\boldsymbol{x}\in \mathbb{R}^{n_{2}n_{3} \times 1 }$ denotes the vectorized column vector
	of $\vec{\boldsymbol{x}} $. If we let $\boldsymbol{x}$ be a unit vector, as explained in the proof of Property \ref{pro1} above, there holds
	$$
	\left\|\mathbf{\mathcal{A}} * \vec{\boldsymbol{x}}\right\|_{2^{*}}	^{2}-\left\|\mathbf{\mathcal{B}} * \vec{\boldsymbol{x}}\right\|_{2^{*}}	^{2}=\sum_{j=1}^{n_1}\left(\left\|\mathtt{bcirc}(\mathbf{\mathcal{C}}_{j})\boldsymbol{x}\right\|^{2}-\left\|\mathtt{bcirc}(\mathbf{\mathcal{B}}_{j})\boldsymbol{x}\right\|^{2} \right).
	$$
	Since $\boldsymbol{x}$ is the unit vector, we obtain 
	\begin{align}
		&\left\|\mathtt{bcirc}(\mathbf{\mathcal{C}}_{j})\boldsymbol{x}\right\|^{2}-\left\|\mathtt{bcirc}(\mathbf{\mathcal{B}}_{j})\boldsymbol{x}\right\|^{2} \notag\\
		=&\boldsymbol{x}^{T}\left(\mathtt{bcirc}(\mathbf{\mathcal{C}}_{j})^{T}\mathtt{bcirc}(\mathbf{\mathcal{C}}_{j})-\mathtt{bcirc}(\mathbf{\mathcal{B}}_{j})^{T}\mathtt{bcirc}(\mathbf{\mathcal{B}}_{j})\right)\boldsymbol{x} \notag\\
		\le&\left\|\mathtt{bcirc}(\mathbf{\mathcal{C}}_{j})^{T}\mathtt{bcirc}(\mathbf{\mathcal{C}}_{j})-\mathtt{bcirc}(\mathbf{\mathcal{B}}_{j})^{T}\mathtt{bcirc}(\mathbf{\mathcal{B}}_{j})\right\| .\notag
	\end{align}
	According to Lemma \ref{lund}, we further obtain
	\begin{align}
	&\left\|\mathtt{bcirc}(\mathbf{\mathcal{C}}_{j})^{T}\mathtt{bcirc}(\mathbf{\mathcal{C}}_{j})-\mathtt{bcirc}(\mathbf{\mathcal{B}}_{j})^{T}\mathtt{bcirc}(\mathbf{\mathcal{B}}_{j})\right\| \notag\\
	=&\left\|\mathtt{bcirc}\left(\mathbf{\mathcal{C}}_{j}^{T}*\mathbf{\mathcal{C}}_{j}-\mathbf{\mathcal{B}}_{j}^{T}*\mathbf{\mathcal{B}}_{j}\right)\right\|, \notag
\end{align}
	then due to $\left(\boldsymbol{F}_{n_{3}} \otimes \boldsymbol{I}_{n_{1}}\right) \cdot \mathtt{bcirc}(\mathbf{\mathcal{A}}) \cdot\left(\boldsymbol{F}_{n_{3}}^{-1} \otimes \boldsymbol{I}_{n_{2}}\right)=\boldsymbol{\bar{A}}$ and the property that $\left(\boldsymbol{F}_{n_{3}} \otimes \boldsymbol{I}_{n_{1}}\right) / \sqrt{n_{3}}$ is orthogonal, we have
	$$
	\left\|\mathtt{bcirc}\left(\mathbf{\mathcal{C}}_{j}^{T}*\mathbf{\mathcal{C}}_{j}-\mathbf{\mathcal{B}}_{j}^{T}*\mathbf{\mathcal{B}}_{j}\right)\right\|
	=\left\| \text{DFT}\left(\boldsymbol{C}_{j}^{T}\boldsymbol{C}_{j}- \boldsymbol{B}_{j}^{T}\boldsymbol{B}_{j} \right) \right\|.
	$$
	Furthermore, since DFT is a linear transform, the following equation holds
	$$
	\left\| \text{DFT}\left(\boldsymbol{C}_{j}^{T}\boldsymbol{C}_{j}- \boldsymbol{B}_{j}^{T}\boldsymbol{B}_{j} \right) \right\|
	=\left\| \boldsymbol{\bar{C}}_{j}^{T}\boldsymbol{\bar{C}}_{j}- \boldsymbol{\bar{B}}_{j}^{T}\boldsymbol{\bar{B}}_{j}  \right\|.
	$$
	Therefore,	
	\begin{align}
	\left\|\mathbf{\mathcal{A}} * \vec{\boldsymbol{x}}\right\|_{2^{*}}	^{2}-\left\|\mathbf{\mathcal{B}} * \vec{\boldsymbol{x}}\right\|_{2^{*}}	^{2} 
	\le&\sum_{j=1}^{n_1}\left\| \boldsymbol{\bar{C}}_{j}^{T}\boldsymbol{\bar{C}}_{j}- \boldsymbol{\bar{B}}_{j}^{T}\boldsymbol{\bar{B}}_{j}  \right\|  \notag\\
	=&\sum_{j=1}^{n_1}\max\limits_{i} \delta_{j}^{(i)}=\Delta. \notag
	\end{align}
	This completes the proof.
\end{proof}

\begin{lemma}
	\label{lemma:re}
	For the tensor column $\vec{\boldsymbol{x}} \in \mathbb{R}^{n_{2} \times 1 \times n_{3}}$, $\boldsymbol{x}\in \mathbb{R}^{n_{2}n_{3} \times 1 }$ denotes the vectorized column vector
	of $\vec{\boldsymbol{x}} $. Let $\boldsymbol{x}$ be the eigenvector of $\operatorname{bcirc}(\mathbf{\mathcal{A}})^{T}\operatorname{bcirc}(\mathbf{\mathcal{A}})-\operatorname{bcirc}(\mathbf{\mathcal{B}})^{T}\operatorname{bcirc}(\mathbf{\mathcal{B}})$ corresponding to its largest eigenvalue, then
	$$
	\left\|\mathbf{\mathcal{A}}^{T}*\mathbf{\mathcal{A}}-\mathbf{\mathcal{B}}^{T}*\mathbf{\mathcal{B}}\right\|  
	=\left\|\mathbf{\mathcal{A}} * \vec{\boldsymbol{x}}\right\|_{2^{*}} ^{2}-\left\|\mathbf{\mathcal{B}} * \vec{\boldsymbol{x}}\right\|_{2^{*}} ^{2} .
	$$
\end{lemma}

\begin{proof}
	In what follows, we take $\boldsymbol{x}$ as the eigenvector of $\mathtt{bcirc}(\mathbf{\mathcal{A}})^{T}\mathtt{bcirc}(\mathbf{\mathcal{A}})-\mathtt{bcirc}(\mathbf{\mathcal{B}})^{T}\mathtt{bcirc}(\mathbf{\mathcal{B}})$ corresponding to its largest eigenvalue. According to the definition of tensor spectral norm (Def. \ref{def_tsn}), we have
	\begin{align}
		&\left\|\mathbf{\mathcal{A}}^{T}*\mathbf{\mathcal{A}}-\mathbf{\mathcal{B}}^{T}*\mathbf{\mathcal{B}}\right\|  \notag\\
		=&\left\| \boldsymbol{\bar{A}}^{T}\boldsymbol{\bar{A}}- \boldsymbol{\bar{B}}^{T}\boldsymbol{\bar{B}}\right\| \notag \\
		=&\left\| \mathtt{bcirc}(\mathbf{\mathcal{A}})^{T}\mathtt{bcirc}(\mathbf{\mathcal{A}})-\mathtt{bcirc}(\mathbf{\mathcal{B}})^{T}\mathtt{bcirc}(\mathbf{\mathcal{B}}) \right\|  \notag \\
		=&\left\|\mathtt{bcirc}(\mathbf{\mathcal{A}})\boldsymbol{x}\right\|^{2}-\left\|\mathtt{bcirc}(\mathbf{\mathcal{B}})\boldsymbol{x}\right\|^{2} \notag\\
		=&\left\|\mathbf{\mathcal{A}} * \vec{\boldsymbol{x}}\right\|_{2^{*}}	^{2}-\left\|\mathbf{\mathcal{B}} * \vec{\boldsymbol{x}}\right\|_{2^{*}}	^{2} \notag.
	\end{align} 
	This concludes the proof of Lemma \ref{lemma:re}.
\end{proof}

The first two properties bound the projected distance for a tensor  column $\vec{\boldsymbol{x}} $ from $\mathcal{A}$ and $\mathcal{B}$, which indicate the tensor sketch $\mathcal{B}$ really captures the principle subspace of $\mathcal{A}$.  And Lemma \ref{lemma:re} demonstrates the importance to build an upper bound for $\Delta$.

\begin{property}\label{pro3}
	If $\mathbf{\mathcal{B}}$ is the output result by applying Algorithm \ref{tensor-FD} to the input $\mathbf{\mathcal{A}}$ with prescribed sketch size $\ell$, then for any $\ell> ck$, we have $\Delta \le \frac{1}{\frac{\ell}{c}-k}\left\|\mathbf{\mathcal{A}}-\mathbf{\mathcal{A}}_{k}\right\|_{F}^{2}$, where $c=\frac{n_{3}\sum_{j=1}^{n_1}\max\limits_{i} \delta_{j}^{(i)} }{\sum_{j=1}^{n_1}\sum_{i=1}^{n_{3}} \delta_{j}^{(i)}}$.
\end{property}

\begin{proof}
	Noting that $\mathbf{\mathcal{B}}$ is initialized to a all-zero tensor, we have
	\begin{align}
		\left\|\mathbf{\mathcal{B}}\right\|_{F}^{2}  
		=&\sum_{j=1}^{n_1}\left(\left\|\mathbf{\mathcal{B}}_{j}\right\|_{F}^{2}-\left\|\mathbf{\mathcal{B}}_{j-1}\right\|_{F}^{2} \right) \notag\\
		=&\sum_{j=1}^{n_1}\left[\left(\left\|\mathbf{\mathcal{C}}_{j}\right\|_{F}^{2}-\left\|\mathbf{\mathcal{B}}_{j-1}\right\|_{F}^{2} \right) -\left(\left\|\mathbf{\mathcal{C}}_{j}\right\|_{F}^{2}-\left\|\mathbf{\mathcal{B}}_{j}\right\|_{F}^{2} \right) \right]. \notag
	\end{align}
	Since $\boldsymbol{C}_{j}^{(i)}$ is composed of $\boldsymbol{B}_{j-1}^{(i)}$ and $\boldsymbol{A}_{j}^{(i)}$, the following relationship holds
	$$
	\sum_{j=1}^{n_1}\left(\left\|\mathbf{\mathcal{C}}_{j}\right\|_{F}^{2}-\left\|\mathbf{\mathcal{B}}_{j-1}\right\|_{F}^{2} \right)
	=\left\|\mathbf{\mathcal{A}}\right\|_{F}^{2}.
	$$
	And according to the definition and property of the DFT tensor $\boldsymbol{\bar{C}}_{j}$, we obtain $$\left\|\mathbf{\mathcal{C}}_{j}\right\|_{F}^{2}=\frac{1}{n_{3}}\left\|\mathtt{bcirc}(\mathbf{\mathcal{C}}_{j})\right\|_{F}^{2}=\frac{1}{n_{3}}\left\|\boldsymbol{\bar{C}}_{j}\right\|_{F}^{2},$$
	by utilizing the relation between the Frobenius norm and trace of a matrix, we further obtain
	$$
	\left\|\boldsymbol{\bar{C}}_{j}\right\|_{F}^{2}=\text{tr}\left(\boldsymbol{\bar{C}}_{j}^{T}\boldsymbol{\bar{C}}_{j}\right).
	$$
	Therefore,
	\begin{align}
		\left\|\mathbf{\mathcal{B}}\right\|_{F}^{2}  
		=&\sum_{j=1}^{n_1}\left[\left(\left\|\mathbf{\mathcal{C}}_{j}\right\|_{F}^{2}-\left\|\mathbf{\mathcal{B}}_{j-1}\right\|_{F}^{2} \right) -\left(\left\|\mathbf{\mathcal{C}}_{j}\right\|_{F}^{2}-\left\|\mathbf{\mathcal{B}}_{j}\right\|_{F}^{2} \right) \right] \notag\\
		=&\left\|\mathbf{\mathcal{A}}\right\|_{F}^{2}-\frac{1}{n_{3}}\sum_{j=1}^{n_1}\text{tr}\left(\boldsymbol{\bar{C}}_{j}^{T}\boldsymbol{\bar{C}}_{j}- \boldsymbol{\bar{B}}_{j}^{T}\boldsymbol{\bar{B}}_{j}\right) \notag \\
		\le&\left\|\mathbf{\mathcal{A}}\right\|_{F}^{2}-\frac{\ell}{n_{3}}\sum_{j=1}^{n_1}\sum_{i=1}^{n_{3}} \delta_{j}^{(i)} \notag \\
		=&\left\|\mathbf{\mathcal{A}}\right\|_{F}^{2}-\frac{\ell}{c} \Delta, \notag 
	\end{align}
	where $c=\frac{n_{3}\sum_{j=1}^{n_1}\max\limits_{i} \delta_{j}^{(i)} }{\sum_{j=1}^{n_1}\sum_{i=1}^{n_{3}} \delta_{j}^{(i)}}$. Furthermore, based on the property that $\|\mathbf{\mathcal{A}}\|_{F}^{2}=\|\mathbf{\mathcal{A}}*\mathbf{\mathcal{Y}}\|_{F}^{2}=\sum_{i=1}^{r}\left\|\mathbf{\mathcal{A}} * \vec{\boldsymbol{y}}_{i}\right\|_{2^{*}}	^{2}$, where $r$ is the tensor tubal rank of $\mathbf{\mathcal{A}}$, we have
	\begin{align}
		\frac{\ell}{c} \Delta  & \le\|\mathbf{\mathcal{A}}\|_{F}^{2}-\|\mathbf{\mathcal{B}}\|_{F}^{2} \notag\\
		&=\sum_{i=1}^{k}\left\|\mathbf{\mathcal{A}} * \vec{\boldsymbol{y}}_{i}\right\|_{2^{*}}	^{2}+\sum_{i=k+1}^{r}\left\|\mathbf{\mathcal{A}} * \vec{\boldsymbol{y}}_{i}\right\|_{2^{*}}	^{2}-\|\mathbf{\mathcal{B}}\|_{F}^{2} \notag\\
		&=\sum_{i=1}^{k}\left\|\mathbf{\mathcal{A}} * \vec{\boldsymbol{y}}_{i}\right\|_{2^{*}}	^{2}+\left\|\mathbf{\mathcal{A}}-\mathbf{\mathcal{A}}_{k}\right\|_{F}^{2}-\|\mathbf{\mathcal{B}}\|_{F}^{2} \notag\\
		& \leq\left\|\mathbf{\mathcal{A}}-\mathbf{\mathcal{A}}_{k}\right\|_{F}^{2}+\sum_{i=1}^{k}\left(\left\|\mathbf{\mathcal{A}} * \vec{\boldsymbol{y}}_{i}\right\|_{2^{*}}	^{2}-\left\|\mathbf{\mathcal{B}} * \vec{\boldsymbol{y}}_{i}\right\|_{2^{*}}	^{2}\right) \notag\\
		& \leq\left\|\mathbf{\mathcal{A}}-\mathbf{\mathcal{A}}_{k}\right\|_{F}^{2}+k \Delta.  \notag
	\end{align}
	Then, we can conclude that $\Delta \le \frac{1}{\frac{\ell}{c}-k}\left\|\mathbf{\mathcal{A}}-\mathbf{\mathcal{A}}_{k}\right\|_{F}^{2}$.
\end{proof}

\subsection{Error bounds of the proposed t-FD}

\begin{proof}[Proof of Theorem \ref{main1}]
	By Lemma \ref{lemma:re}, when the vectorized column vector
	of $\vec{\boldsymbol{x}} $ takes the eigenvector of $\mathtt{bcirc}(\mathbf{\mathcal{A}})^{T}\mathtt{bcirc}(\mathbf{\mathcal{A}})-\mathtt{bcirc}(\mathbf{\mathcal{B}})^{T}\mathtt{bcirc}(\mathbf{\mathcal{B}})$ corresponding to its largest eigenvalue, the tensor covariance error is equivalent to the following formulation:
	\begin{align}
		\left\|\mathbf{\mathcal{A}} * \vec{\boldsymbol{x}}\right\|_{2^{*}} ^{2}-\left\|\mathbf{\mathcal{B}} * \vec{\boldsymbol{x}}\right\|_{2^{*}} ^{2}. \label{equivalent}
	\end{align}
	
	So, to complete the proof, we only need to analyze the error bound of (\ref{equivalent}). Note that the aforementioned tensor column $\vec{\boldsymbol{x}}$ satisfies the unit norm constraint, thus combining Properties 2 and 3 can get the desired bound.
\end{proof}

\begin{proof}[Proof of Theorem \ref{main}]
	By using the Pythagorean theorem, $\left\|\mathbf{\mathcal{A}}-\mathbf{\mathcal{A}} * \mathbf{\mathcal{V}}_{k} * \mathbf{\mathcal{V}}_{k}^{T}\right\|_{F}^{2}
	=\left\|\mathbf{\mathcal{A}}\right\|_{F}^{2}-\left\|\mathbf{\mathcal{A}} * \mathbf{\mathcal{V}}_{k}\right\|_{F}^{2}$. Since $\vec{\boldsymbol{v}}_{i} \in \mathbb{R}^{n_{2} \times 1 \times n_{3}}$ is the $i$-th lateral slice of $\mathbf{\mathcal{V}}_{k}$, we rewrite $\left\|\mathbf{\mathcal{A}} * \mathbf{\mathcal{V}}_{k} \right\|_{F}^{2}$ as $\sum_{i=1}^{k}\left\|\mathbf{\mathcal{A}} * \vec{\boldsymbol{v}}_{i}\right\|_{2^{*}}	^{2}$. So, we obtain
	$$
	\left\|\mathbf{\mathcal{A}}-\mathbf{\mathcal{A}} * \mathbf{\mathcal{V}}_{k} * \mathbf{\mathcal{V}}_{k}^{T}\right\|_{F}^{2} 
	=\left\|\mathbf{\mathcal{A}}\right\|_{F}^{2}-\sum_{i=1}^{k}\left\|\mathbf{\mathcal{A}} * \vec{\boldsymbol{v}}_{i}\right\|_{2^{*}}	^{2} .
	$$
	According to Property \ref{pro1}, it is easy to see that
	$$
	\left\|\mathbf{\mathcal{A}} * \vec{\boldsymbol{v}}_{i}\right\|_{2^{*}}	^{2}\ge\left\|\mathbf{\mathcal{B}} * \vec{\boldsymbol{v}}_{i}\right\|_{2^{*}}	^{2}.
	$$
	Therefore, we can further obtain that
	$$
	\left\|\mathbf{\mathcal{A}}-\mathbf{\mathcal{A}} * \mathbf{\mathcal{V}}_{k} * \mathbf{\mathcal{V}}_{k}^{T}\right\|_{F}^{2}\le \left\|\mathbf{\mathcal{A}}\right\|_{F}^{2}-\sum_{i=1}^{k}\left\|\mathbf{\mathcal{B}} * \vec{\boldsymbol{v}}_{i}\right\|_{2^{*}}	^{2} .
	$$
	Noting that $\mathbf{\mathcal{B}}_{k}=\mathbf{\mathcal{U}}_{k}*\mathbf{\mathcal{S}}_{k}*\mathbf{\mathcal{V}}_{k}^{T}$, and $\vec{\boldsymbol{v}}_{i} \in \mathbb{R}^{n_{2} \times 1 \times n_{3}}$ is the $i$-th lateral slice of $\mathbf{\mathcal{V}}_{k}$. Thus, $\sum_{i=1}^{k}\left\|\mathbf{\mathcal{B}} * \vec{\boldsymbol{v}}_{i}\right\|_{2^{*}}	^{2}\ge \sum_{i=1}^{k}\left\|\mathbf{\mathcal{B}} * \vec{\boldsymbol{y}}_{i}\right\|_{2^{*}}	^{2}.$ Therefore,
	$$
	\left\|\mathbf{\mathcal{A}}-\mathbf{\mathcal{A}} * \mathbf{\mathcal{V}}_{k} * \mathbf{\mathcal{V}}_{k}^{T}\right\|_{F}^{2}
	\le\left\|\mathbf{\mathcal{A}}\right\|_{F}^{2}-\sum_{i=1}^{k}\left\|\mathbf{\mathcal{B}} * \vec{\boldsymbol{y}}_{i}\right\|_{2^{*}}	^{2} . 
	$$
	Then, it follows from the conclusion of Property \ref{pro2} that	
	\begin{align}
		\left\|\mathbf{\mathcal{A}}-\mathbf{\mathcal{A}} * \mathbf{\mathcal{V}}_{k} * \mathbf{\mathcal{V}}_{k}^{T}\right\|_{F}^{2} 
		\le&\|\mathbf{\mathcal{A}}\|_{F}^{2}-\sum_{i=1}^{k}\left(\left\|\mathbf{\mathcal{A}} * \vec{\boldsymbol{y}}_{i}\right\|_{2^{*}}	^{2}-\Delta\right)  \notag \\
		=&\|\mathbf{\mathcal{A}}\|_{F}^{2}-\left\|\mathbf{\mathcal{A}}_{k}\right\|_{F}^{2}+k \Delta \notag\\
		\le&\frac{\ell}{\ell-ck}\left\|\mathbf{\mathcal{A}}-\mathbf{\mathcal{A}}_{k}\right\|_{F}^{2},    \label{fsimilar}
	\end{align}	
	where the last inequality can be derived directly from Property \ref{pro3}. This completes the proof of Theorem \ref{main}. And if we set $\ell=c\lceil k+k / \varepsilon\rceil$, then we can get the standard bound that $     \left\|\mathbf{\mathcal{A}}-\mathbf{\mathcal{A}} * \mathbf{\mathcal{V}}_{k} * \mathbf{\mathcal{V}}_{k}^{T}\right\|_{F}^{2} \le (1+\varepsilon)\left\|\mathbf{\mathcal{A}}-\mathbf{\mathcal{A}}_{k}\right\|_{F}^{2}.$
\end{proof}

	\begin{proof}[Proof of Theorem \ref{tce-p-order}]
		The proof is very similar to that of Theorem \ref{main1}. Firstly, according to the relationship of the spectral norm between $\mathbf{\mathcal{A}}$, $\boldsymbol{\tilde{A}}$ and $\boldsymbol{\bar{A}}$, the tensor covariance error can be reformulated as:
		\begin{align}
			\left\|\mathbf{\mathcal{A}}^{T}*\mathbf{\mathcal{A}}-\mathbf{\mathcal{B}}^{T}*\mathbf{\mathcal{B}}\right\|  
			=\left\|\mathbf{\mathcal{A}} * \vec{\boldsymbol{x}}\right\|_{2^{*}} ^{2}-\left\|\mathbf{\mathcal{B}} * \vec{\boldsymbol{x}}\right\|_{2^{*}} ^{2} .\label{combine1}
		\end{align}
		Then the second step is to upper bound $\left\|\mathbf{\mathcal{A}} * \vec{\boldsymbol{x}}\right\|_{2^{*}} ^{2}-\left\|\mathbf{\mathcal{B}} * \vec{\boldsymbol{x}}\right\|_{2^{*}} ^{2}$ by the following formulation:
		\begin{align}
			\left\|\mathbf{\mathcal{A}} * \vec{\boldsymbol{x}}\right\|_{2^{*}}	^{2}-\left\|\mathbf{\mathcal{B}} * \vec{\boldsymbol{x}}\right\|_{2^{*}}	^{2} \le  \Delta \le \frac{1}{\frac{\ell}{c}-k}\left\|\mathbf{\mathcal{A}}-\mathbf{\mathcal{A}}_{k}\right\|_{F}^{2},\label{combine2}
		\end{align}
		where $ \Delta = \sum_{j=1}^{n_1}\max\limits_{i} \delta_{j}^{(i)}$ is the sum of maximum information losses in the truncated procedure, and $c=\frac{\rho\sum_{j=1}^{n_1}\max\limits_{i} \delta_{j}^{(i)} }{\sum_{j=1}^{n_1}\sum_{i=1}^{\rho} \delta_{j}^{(i)}}$. The derivation  of (\ref{combine2}) depends on the the relationship between the Frobenius norm of original and Fourier domains, that is, $\left\|\mathbf{\mathcal{C}}_{j}\right\|_{F}^{2}=\frac{1}{\rho}\left\|\boldsymbol{\bar{C}}_{j}\right\|_{F}^{2}$. Different from the third-order case, the number of blocks of $\boldsymbol{\bar{C}}_{j}$ is changed to $\rho$, which directly results in the change of $c$. Finally, combining (\ref{combine1}) with (\ref{combine2})  yields the desired bound. 
	\end{proof}
	\begin{proof}[Proof of Theorem \ref{tpe-p-order}]
		The proof follows the similar  ideas of  the proof of Theorem 2. By using the Pythagorean theorem,
		we get $\left\|\mathcal{A}-\mathcal{A} * \mathcal{V}_{k} * \mathcal{V}_{k}^{T}\right\|_{F}^{2}=\|\mathcal{A}\|_{F}^{2}-\left\|\mathcal{A} * \mathcal{V}_{k}\right\|_{F}^{2}$. Since
		$\vec{\boldsymbol{v}}_{i} \in \mathbb{R}^{n_{2} \times 1 \times n_{3} \times \cdots \times n_{p}}$ is the $i$-th lateral slice of $\mathcal{V}_{k}$, we
		rewrite $\left\|\mathcal{A} * \mathcal{V}_{k}\right\|_{F}^{2}$ as $\sum_{i=1}^{k}\left\|\mathcal{A} * \vec{\boldsymbol{v}}_{\boldsymbol{i}}\right\|_{2^{*}}^{2} .$ 
		Then similar to with the third-order case, it is easy to get 
		$
		\left\|\mathcal{A} * \vec{\boldsymbol{v}}_{i}\right\|_{2^{*}}^{2} \geq\left\|\mathcal{B} * \vec{\boldsymbol{v}}_{i}\right\|_{2^{*}}^{2}
		$
		and
		$\sum_{i=1}^{k}\left\|\mathcal{B} * \vec{\boldsymbol{v}}_{i}\right\|_{2^{*}}^{2} \geq
		\sum_{i=1}^{k}\left\|\mathcal{B} * \vec{\boldsymbol{y}}_{i}\right\|_{2^{*}}^{2} $. Thus, 
		$$
		\left\|\mathcal{A}-\mathcal{A} * \mathcal{V}_{k} * \mathcal{V}_{k}^{T}\right\|_{F}^{2} \leq\|\mathcal{A}\|_{F}^{2}-\sum_{i=1}^{k}\left\|\mathcal{B} * \vec{\boldsymbol{y}}_{i}\right\|_{2^{*}}^{2}.
		$$
		Finally, combining with (\ref{combine2}), we obtain
		$$
		\begin{aligned}
			\left\|\mathcal{A}-\mathcal{A} * \mathcal{V}_{k} * \mathcal{V}_{k}^{T}\right\|_{F}^{2} & \leq\|\mathcal{A}\|_{F}^{2}-\sum_{i=1}^{k}\left(\left\|\mathcal{A} * \vec{\boldsymbol{y}}_{i}\right\|_{2^{*}}^{2}-\Delta\right) \\
			&=\|\mathcal{A}\|_{F}^{2}-\left\|\mathcal{A}_{k}\right\|_{F}^{2}+k \Delta \\
			& \leq \frac{\ell}{\ell-c k}\left\|\mathcal{A}-\mathcal{A}_{k}\right\|_{F}^{2},
		\end{aligned}
		$$
		where $c=\frac{\rho\sum_{j=1}^{n_1}\max\limits_{i} \delta_{j}^{(i)} }{\sum_{j=1}^{n_1}\sum_{i=1}^{\rho} \delta_{j}^{(i)}}$. This finishes the proof. 
\end{proof}

\subsection{Error bounds of the compared algorithm MtFD}

\begin{proof}[Proof of Theorem \ref{thm-norm2-matrization}]
	For the $n_1\rho \times n_2\rho$ block matrix $\boldsymbol{\tilde{A}}$, its explicit form could be very complicated for higher order case. So for simplicity, we take the third-order case as an example to describe the proof process. And more generally,  the proof technique presented here is also applicable to order-$p(p>3)$ case. According to the proof framework of Theorem \ref{main1}, it is easy to see that
	$$
	\left\|\mathbf{\mathcal{A}}^{T}*\mathbf{\mathcal{A}}-\mathbf{\mathcal{B}}^{T}*\mathbf{\mathcal{B}}\right\| 
	=\left\|\mathbf{\mathcal{A}} * \vec{\boldsymbol{x}}\right\|_{2^{*}}	^{2}-\left\|\mathbf{\mathcal{B}} * \vec{\boldsymbol{x}}\right\|_{2^{*}}	^{2}, 
	$$
	thus, our core is to bound $\left\|\mathbf{\mathcal{A}} * \vec{\boldsymbol{x}}\right\|_{2^{*}}	^{2}-\left\|\mathbf{\mathcal{B}} * \vec{\boldsymbol{x}}\right\|_{2^{*}}$. For $\vec{\boldsymbol{x}} \in \mathbb{R}^{n_{2} \times 1 \times n_{3}}$, let $\boldsymbol{x}\in \mathbb{R}^{n_{2}n_{3} \times 1 }$ denote the vectorized column vector
	of $\vec{\boldsymbol{x}} $ with unit norm. By the definition of t-product (Def. 1), the following equivalent relation is established:
	$$
	\left\|\mathbf{\mathcal{A}} * \vec{\boldsymbol{x}}\right\|_{2^{*}}	^{2}-\left\|\mathbf{\mathcal{B}} * \vec{\boldsymbol{x}}\right\|_{2^{*}}	^{2}  
	=\left\|\mathtt{bcirc}(\mathbf{\mathcal{A}})\boldsymbol{x}\right\|^{2}-\left\|\mathtt{bcirc}(\mathbf{\mathcal{B}})\boldsymbol{x}\right\|^{2}.
	$$
	Divide $\boldsymbol{x}$ into $n_{3}$ parts, where each part $\boldsymbol{x^{i}}\in \mathbb{R}^{n_{2} \times 1 }$. 
	The 2-norm of a vector is defined as the square root of the inner product of the vector with itself. Therefore, after rearranging the blocks in the block circulant matrix and positions of the corresponding $\boldsymbol{x^{i}}$ is, the vector 2-norm remains unchanged. So,
	\begin{align}
		&\left\|\mathtt{bcirc}(\mathbf{\mathcal{A}})\boldsymbol{x}\right\|^{2} \notag\\
		=&\left\|\left[\begin{array}{cccc}
			\boldsymbol{A}^{(1)} & \boldsymbol{A}^{\left(n_{3}\right)} & \cdots & \boldsymbol{A}^{(2)} \\
			\boldsymbol{A}^{(2)} & \boldsymbol{A}^{(1)} & \cdots & \boldsymbol{A}^{(3)} \\
			\vdots & \vdots & \ddots & \vdots \\
			\boldsymbol{A}^{\left(n_{3}\right)} & \boldsymbol{A}^{\left(n_{3}-1\right)} & \cdots & \boldsymbol{A}^{(1)}
		\end{array}\right]  \left(\begin{array}{c}
			\boldsymbol{x^{1}} \\
			\boldsymbol{x^{2}} \\
			\vdots \\
			\boldsymbol{x^{n_{3}}} 
		\end{array}\right)\right\|^{2} \notag \\
		=&\left\|\left[\begin{array}{cccc}
			\boldsymbol{A}^{(1)} & \boldsymbol{A}^{\left(n_{3}\right)} & \cdots & \boldsymbol{A}^{(2)} 
		\end{array}\right]  \left(\begin{array}{c}
			\boldsymbol{x^{1}} \\
			\boldsymbol{x^{2}} \\
			\vdots \\
			\boldsymbol{x^{n_{3}}} 
		\end{array}\right)\right\|^{2} + \cdots \notag \\
		=&\left\|\left[\begin{array}{cccc}
			\boldsymbol{A}^{(1)} & \boldsymbol{A}^{\left(2\right)} & \cdots & \boldsymbol{A}^{(n_{3})} 
		\end{array}\right]  \left(\begin{array}{c}
			\boldsymbol{x^{1}} \\
			\boldsymbol{x^{n_{3}}} \\
			\vdots \\
			\boldsymbol{x^{2}} 
		\end{array}\right)\right\|^{2} + \cdots\notag \\
		=&\sum_{i=1}^{n_{3}}\left\|\boldsymbol{A}_{(1)}\boldsymbol{x}_{i}\right\|^{2}, \notag
	\end{align}
	where each $\boldsymbol{x}_{i}\in \mathbb{R}^{n_{2}n_{3} \times 1 }$ is a unit vector.
	Therefore,
	\begin{align}
		\left\|\mathbf{\mathcal{A}}^{T}*\mathbf{\mathcal{A}}-\mathbf{\mathcal{B}}^{T}*\mathbf{\mathcal{B}}\right\| 
		=\sum_{i=1}^{n_{3}}\left(\left\|\boldsymbol{A}_{(1)}\boldsymbol{x}_{i}\right\|^{2}-\left\|\boldsymbol{B}_{(1)}\boldsymbol{x}_{i}\right\|^{2}\right). \label{tensor-matrix}
	\end{align}
	Then, it follows from the properties of FD proved in Theorem 1.1 of \cite{ghashami2016frequent} that
	\begin{align}
		&\sum_{i=1}^{n_{3}}\left(\left\|\boldsymbol{A}_{(1)}\boldsymbol{x}_{i}\right\|^{2}-\left\|\boldsymbol{B}_{(1)}\boldsymbol{x}_{i}\right\|^{2}\right) \notag\\
		\le&\frac{n_{3}}{\ell-k}\left\|\boldsymbol{A}_{(1)}-\boldsymbol{A}_{(1)_{k}}\right\|_{F}^{2} \notag \\
		\le& \frac{n_{3}}{\ell-k}\left\|\boldsymbol{A}_{(1)}\right\|_{F}^{2} \notag\\
		=&\frac{n_{3}}{\ell-k}\left\|\mathbf{\mathcal{A}}\right\|_{F}^{2},\notag
	\end{align}
	which together with the equation (\ref{tensor-matrix}) conclude the desired result. 
\end{proof}

\begin{proof}[Proof of Theorem \ref{rem: rel}]
	Also, we here present the proof for the third-order tensors, and the general order-$p$ case is also applicable. The only difference is that the spectral norm of block circulant operation $\mathtt{bcirc}(\mathbf{\mathcal{A}})$, which acts as the bridge between tensor spectral norm and the $\ell_{2^{*}}$ norm of tensor column, is replaced with the spectral norm of the general $\boldsymbol{\tilde{A}}$.	Noting that for any $\vec{\boldsymbol{x}} \in \mathbb{R}^{n_{2} \times 1 \times n_{3}}$, if the vectorized column vector $\boldsymbol{x}\in \mathbb{R}^{n_{2}n_{3} \times 1 }$ is a unit vector, then we have
	\begin{align}
		&\left|\left\|\mathbf{\mathcal{A}} * \vec{\boldsymbol{x}}\right\|_{2^{*}}	^{2}-\left\|\mathbf{\mathcal{B}} * \vec{\boldsymbol{x}}\right\|_{2^{*}}	^{2}\right| \notag\\
		=&\left|\left\|\mathtt{bcirc}(\mathbf{\mathcal{A}})\boldsymbol{x}\right\|^{2}-\left\|\mathtt{bcirc}(\mathbf{\mathcal{B}})\boldsymbol{x}\right\|^{2}\right| \notag\\
		\le&\left\| \mathtt{bcirc}(\mathbf{\mathcal{A}})^{T}\mathtt{bcirc}(\mathbf{\mathcal{A}})-\mathtt{bcirc}(\mathbf{\mathcal{B}})^{T}\mathtt{bcirc}(\mathbf{\mathcal{B}}) \right\|_{2}. \notag
	\end{align}
	Then, according to the definition of tensor spectral norm (Def. \ref{def_tsn}), we have
	\begin{align}	
		&\left\| \mathtt{bcirc}(\mathbf{\mathcal{A}})^{T}\mathtt{bcirc}(\mathbf{\mathcal{A}})-\mathtt{bcirc}(\mathbf{\mathcal{B}})^{T}\mathtt{bcirc}(\mathbf{\mathcal{B}}) \right\|_{2}  \notag \\
		=&\left\| \boldsymbol{\bar{A}}^{T}\boldsymbol{\bar{A}}- \boldsymbol{\bar{B}}^{T}\boldsymbol{\bar{B}}\right\|_{2}  \notag \\
		=&\left\|\mathbf{\mathcal{A}}^{T}*\mathbf{\mathcal{A}}-\mathbf{\mathcal{B}}^{T}*\mathbf{\mathcal{B}}\right\| . \notag
	\end{align}
	That is to say,
	\begin{align}
		\left|\left\|\mathbf{\mathcal{A}} * \vec{\boldsymbol{x}}\right\|_{2^{*}}	^{2}-\left\|\mathbf{\mathcal{B}} * \vec{\boldsymbol{x}}\right\|_{2^{*}}	^{2}\right| \le
		\left\|\mathbf{\mathcal{A}}^{T}*\mathbf{\mathcal{A}}-\mathbf{\mathcal{B}}^{T}*\mathbf{\mathcal{B}}\right\|. \label{spectralnormand2*}
	\end{align}
	As analyzed in Theorem \ref{main}, $\left\|\mathbf{\mathcal{A}}-\mathbf{\mathcal{A}} * \mathbf{\mathcal{V}}_{k} * \mathbf{\mathcal{V}}_{k}^{T}\right\|_{F}^{2}$ is equal to $\left\|\mathbf{\mathcal{A}}\right\|_{F}^{2}-\sum_{i=1}^{k}\left\|\mathbf{\mathcal{A}} * \vec{\boldsymbol{v}}_{i}\right\|_{2^{*}}	^{2}$. Therefore, we obtain
	\begin{align}
		&\left\|\mathbf{\mathcal{A}}-\mathbf{\mathcal{A}} * \mathbf{\mathcal{V}}_{k} * \mathbf{\mathcal{V}}_{k}^{T}\right\|_{F}^{2} \notag \\
		=&\left\|\mathbf{\mathcal{A}}\right\|_{F}^{2}-\sum_{i=1}^{k}\left\|\mathbf{\mathcal{A}} * \vec{\boldsymbol{v}}_{i}\right\|_{2^{*}}	^{2}  \notag \\
		\le &\left\|\mathbf{\mathcal{A}}\right\|_{F}^{2}-\sum_{i=1}^{k}\left\|\mathbf{\mathcal{B}} * \vec{\boldsymbol{v}}_{i}\right\|_{2^{*}}	^{2} +k\left\|\mathbf{\mathcal{A}}^{T}*\mathbf{\mathcal{A}}-\mathbf{\mathcal{B}}^{T}*\mathbf{\mathcal{B}}\right\| \text{(By Eq.(\ref{spectralnormand2*}))} \notag \\
		\le&\left\|\mathbf{\mathcal{A}}\right\|_{F}^{2}-\sum_{i=1}^{k}\left\|\mathbf{\mathcal{B}} * \vec{\boldsymbol{y}}_{i}\right\|_{2^{*}}	^{2} +k\left\|\mathbf{\mathcal{A}}^{T}*\mathbf{\mathcal{A}}-\mathbf{\mathcal{B}}^{T}*\mathbf{\mathcal{B}}\right\|  \notag \\
		\le&\left\|\mathbf{\mathcal{A}}\right\|_{F}^{2}-\sum_{i=1}^{k}\left\|\mathbf{\mathcal{A}} * \vec{\boldsymbol{y}}_{i}\right\|_{2^{*}}	^{2} +2k\left\|\mathbf{\mathcal{A}}^{T}*\mathbf{\mathcal{A}}-\mathbf{\mathcal{B}}^{T}*\mathbf{\mathcal{B}}\right\|  \text{(By Eq.(\ref{spectralnormand2*}))}\notag\\
		=&\left\|\mathbf{\mathcal{A}}-\mathbf{\mathcal{A}}_{k}\right\|_{F}^{2}+2k\left\|\mathbf{\mathcal{A}}^{T}*\mathbf{\mathcal{A}}-\mathbf{\mathcal{B}}^{T}*\mathbf{\mathcal{B}}\right\|. \notag
	\end{align}
	This completes the proof.
\end{proof}

\section{Conclusion}
In this paper, we propose a simple and effective sketching algorithm for obtaining a low-tubal-rank tensor approximation in the streaming setting. The main idea is to extend matrix FD algorithm to the higher order tensor case using the t-SVD framework. The  theoretical analysis shows that our new algorithm could provide a near optimal low-tubal-rank tensor approximation in terms of both covariance and projection errors. Extensive experiments on both synthetic and real data also verify the efficiency and effectiveness of the proposed algorithm. In the future, we are planning to incorporate this new algorithm into some popular tensor recovery models, namely tensor completion and tensor robust PCA, in the streaming setting.

\section*{Acknowledgment}

This work was supported in part by the National Key Research and Development Program of China under
Grant 2018YFB1402600, in part by the National Natural Science Foundation of China under Grant 11971374 and Grant 11501440.

\bibliographystyle{IEEEtran}
\bibliography{tensorfd}
\ifCLASSOPTIONcaptionsoff
  \newpage
\fi

\end{document}